\documentclass{article}

\PassOptionsToPackage{numbers, compress}{natbib}

\def\includehome{include}
\def\fighome{figures}

\usepackage{subcaption}
\usepackage{wrapfig}
\usepackage{subfloat}
\usepackage{tikz}
\usetikzlibrary{fit}
\usetikzlibrary{patterns}
\usetikzlibrary{calc,shapes}
\usetikzlibrary{decorations.pathmorphing} 
\usetikzlibrary{fit}					
\usetikzlibrary{backgrounds}	

\usepackage{xspace}
\usepackage{soul}

\newcommand{\mdp}{\mathcal{M}}
\newcommand{\states}{\mathcal{S}}
\newcommand{\actions}{\mathcal{A}}
\newcommand{\attball}{\mathcal{B}_\epsilon}
\newcommand{\worstbell}{\underline{\mathcal{T}}}
\newcommand{\worstbellpi}{\worstbell^\pi}
\newcommand{\worstqpi}{\underline{Q}^\pi}
\newcommand{\worstvpi}{\underline{V}^\pi}
\newcommand{\advaction}{\mathcal{A}_{\mathrm{adv}}}

\newcommand{\aadvaction}{\hat{\mathcal{A}}_{\mathrm{adv}}}

\newcommand{\worstcritic}{\underline{Q}^\pi_\phi}

\newcommand{\lossest}{\mathcal{L}_{\mathrm{est}}}
\newcommand{\lossworst}{\mathcal{L}_{\mathrm{wst}}}
\newcommand{\lossreg}{\mathcal{L}_{\mathrm{reg}}}
\newcommand{\lossrl}{\mathcal{L}_{\mathrm{RL}}}
\newcommand{\loss}{\mathcal{L}}

\newcommand{\weightworst}{\kappa_{\mathrm{wst}}}
\newcommand{\weightreg}{\kappa_{\mathrm{reg}}}

\newcommand{\oursfull}{{Worst-case-aware Robust RL}\xspace}
\newcommand{\ours}{{WocaR-RL}\xspace}
\newcommand{\ourppo}{{WocaR-PPO}\xspace}
\newcommand{\ourdqn}{{WocaR-DQN}\xspace}
\newcommand{\ourac}{{WocaR-A2C}\xspace}
\newcommand{\sappo}{{SA-PPO}\xspace}
\newcommand{\worstvname}{{worst-attack value}\xspace}
\newcommand{\worstqname}{{worst-attack action value}\xspace}
\newcommand{\worstcriticname}{{worst-attack critic}\xspace}
\newcommand{\bellmanname}{{worst-attack Bellman operator}\xspace}
\newcommand{\bellmannamecap}{{Worst-attack Bellman Operator}\xspace}
\newcommand{\estname}{{worst-attack value estimation}\xspace}
\newcommand{\worstname}{{worst-case-aware policy optimization}\xspace}
\newcommand{\regname}{{value-enhanced state regularization}\xspace}

\usepackage[final]{\includehome/neurips_2022}

\usepackage[utf8]{inputenc} 
\usepackage[T1]{fontenc}    
\usepackage{url}            
\usepackage{booktabs}       
\usepackage{amsfonts}       
\usepackage{nicefrac}       
\usepackage{microtype}      
\usepackage{xcolor}         
\def\mytitle{Efficient Adversarial Training without Attacking: Worst-Case-Aware Robust Reinforcement Learning}

\usepackage{microtype}
\usepackage{graphicx}
\usepackage{tikz}
\usetikzlibrary{fit}
\usetikzlibrary{calc,shapes}
\usetikzlibrary{decorations.pathmorphing} 
\usetikzlibrary{fit}     
\usetikzlibrary{backgrounds} 
\usetikzlibrary{pgfplots.groupplots}

\usepackage[utf8]{inputenc}
\usepackage{pgfplots}
\DeclareUnicodeCharacter{2212}{−}
\usepgfplotslibrary{groupplots,dateplot}
\usetikzlibrary{patterns,shapes.arrows}
\pgfplotsset{compat=newest}
\usepackage{booktabs} 
\usepackage{xcolor}
\usepackage{colortbl}
\usepackage{multirow}
\usepackage{array}
\usepackage{makecell}
\usepackage{caption, subcaption}
\usepackage{hyperref}
\hypersetup{colorlinks=true,linkcolor=blue,citecolor=green,urlcolor=blue}
\usepackage{algorithm}
\usepackage{algorithmic}

\usepackage{amsmath}
\usepackage{amssymb}
\usepackage{mathtools}
\usepackage{amsthm}
\usepackage[font=small,labelfont=bf]{caption, subcaption}

\usepackage[capitalize,noabbrev]{cleveref}

\theoremstyle{plain}
\newtheorem{theorem}{Theorem}[section]

\theoremstyle{definition}
\newtheorem{definition}[theorem]{Definition}

\theoremstyle{remark}

\usepackage[textsize=tiny]{todonotes}
\usepackage{pgfplots}

\title{\mytitle}

\author{%
  Yongyuan Liang\textsuperscript{\rm $\dag$}\thanks{Equal contribution.}
  \qquad
  Yanchao Sun\textsuperscript{\rm $\ddag$}\footnotemark[1]
  \qquad
  Ruijie Zheng\textsuperscript{\rm $\ddag$} 
  \qquad
  Furong Huang\textsuperscript{\rm $\ddag$}  \\
  \textsuperscript{\rm $\dag$} Shanghai AI Lab, \quad  \textsuperscript{\rm $\ddag$} University of Maryland, College Park  \\
  \texttt{
  \textsuperscript{\rm $\dag$}\texttt{cheryllLiang@outlook.com}
  \textsuperscript{\rm $\ddag$}\{ycs,rzheng12,furongh\}@umd.edu
  } 
}

\begin{document}

\maketitle

\begin{abstract}
Recent studies reveal that a well-trained deep reinforcement learning (RL) policy can be particularly vulnerable to adversarial perturbations on input observations. Therefore, it is crucial to train RL agents that are robust against any attacks with a bounded budget. Existing robust training methods in deep RL either treat correlated steps separately, ignoring the robustness of long-term rewards, or train the agents and RL-based attacker together, doubling the computational burden and sample complexity of the training process. In this work, we propose a strong and efficient robust training framework for RL, named \oursfull (\ours), that directly estimates and optimizes the worst-case reward of a policy under bounded $\ell_p$ attacks without requiring extra samples for learning an attacker. Experiments on multiple environments show that \ours achieves state-of-the-art performance under various strong attacks, and obtains significantly higher training efficiency than prior state-of-the-art robust training methods. The code of this work is available at \url{https://github.com/umd-huang-lab/WocaR-RL}.
\end{abstract}

\section{Introduction}
\label{sec:intro}


Deep reinforcement learning (DRL) has achieved impressive results by using deep neural networks (DNN) to learn complex policies in large-scale tasks. However, well-trained DNNs may drastically fail under adversarial perturbations of the input~\citep{akhtar2018threat,chakraborty2018adversarial}. Therefore, before deploying DRL policies to real-life applications, it is crucial to improve the robustness of deep policies against adversarial attacks, especially worst-case attacks that maximally depraves the performance of trained agents~\citep{sun2021strongest}.

\begin{wrapfigure}{r}{0.23\textwidth}
\vspace{-2em}
  \centering
  \includegraphics[width=0.2\textwidth]{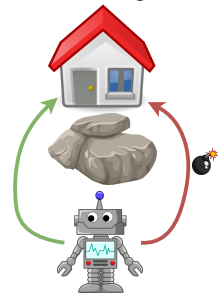}
  \vspace{-0.5em}
  \caption{\small{Policies have different vulnerabilities.}}
  \label{fig:example}
\vspace{-1.3em}
\end{wrapfigure}

A line of regularization-based robust methods~\cite{zhang2020robust,oikarinen2020robust,shen2020deep} focuses on improving the robustness of the DNN itself and regularizes the policy network to output similar actions under bounded state perturbations. 
However, different from supervised learning problems, the vulnerability of a deep policy comes not only from the DNN approximator, but also from the dynamics of the RL environment~\citep{zhang2021robust}. 
These regularization-based methods neglect the intrinsic vulnerability of policies under the environment dynamics, and thus may still fail
under strong attacks~\citep{sun2021strongest}.
For example, in the go-home task shown in Figure~\ref{fig:example}, both the green policy and the red policy arrive home without rock collision, when there is no attack. 
However, although regularization-based methods may ensure a minor action change under a state perturbation, the red policy may still be susceptible to a low reward under attacks, as a very small divergence can lead it to the bomb. 
On the contrary, the green policy is more robust to adversarial attacks since it stays away from the bomb. 
Therefore, besides promoting the robustness of DNN approximators (such as the policy network), it is also important to learn a policy with stronger intrinsic robustness. 



There is another line of work considering the long-term robustness of a deep policy under strong adversarial attacks. 
In particular, it is theoretically proved~\citep{zhang2020robust,sun2021strongest} that the strongest (worst-case) attacker against a policy can be learned as an RL problem, and training the agent under such a learned attacker can result in a robust policy. Zhang et al.~\cite{zhang2021robust} propose the \emph{Alternating Training with Learned Adversaries (ATLA)} framework, which alternately trains an RL agent and an RL attacker. Sun et al.~\citep{sun2021strongest} further propose PA-ATLA, which alternately trains an agent and the proposed more efficient PA-AD RL attacker, obtaining state-of-the-art robustness in many MuJoCo environments. 
However, training an RL attacker requires extra samples from the environment, and the attacker's RL problem may even be more difficult and sample expensive to solve than the agent's original RL problem~\citep{zhang2021robust,sun2021strongest}, especially in large-scale environments such as Atari games with pixel observations. Therefore, although ATLA and PA-ATLA are able to achieve high long-term reward under attacks, they double the computational burden and sample complexity to train the robust agent.

The above analysis of existing literature suggests two main challenges in improving the adversarial robustness of DRL agents: 
(1) correctly characterizing the long-term reward vulnerability of an RL policy, and 
(2) efficiently training a robust agent without requiring much more effort than vanilla training.
To tackle these challenges, in this paper, we propose a generic and efficient robust training framework named \textit{\oursfull (\ours)} that estimates and improves the long-term robustness of an RL agent.

\ours has 3 key mechanisms.
\underline{First}, \ours introduces a novel \emph{\bellmanname} which uses existing off-policy samples to estimate the lower bound of the policy value under the worst-case attack. Compared to prior works~\citep{zhang2021robust,sun2021strongest} which attempt to learn the worst-case attack by RL methods, \ours does not require any extra interaction with the environment.
\underline{Second}, using the estimated worst-case policy value, \ours optimizes the policy to select actions that not only achieve high natural future reward, but also achieve high worst-case reward when there are adversarial attacks. Therefore, \ours learns a policy with less intrinsic vulnerability.
\underline{Third}, \ours regularizes the policy network with a carefully designed state importance weight. As a result, the DNN approximator tolerates state perturbations, especially for more important states where decisions are crucial for future reward.
The above 3 mechanisms can also be interpreted from a geometric perspective of adversarial policy learning, as detailed in Appendix~\ref{app:understand}.

Our \textbf{contributions} can be summarized as below. 
\textbf{(1)} We provide an approach to estimate the worst-case value of any policy under any bounded $\ell_p$ adversarial attacks. This helps evaluate the robustness of a policy without learning an attacker which requires extra samples and exploration.
\textbf{(2)} We propose a novel and principled robust training framework for RL, named \textit{\oursfull (\ours)}, which characterizes and improves the worst-case robustness of an agent. \ours can be used to robustify existing DRL algorithms (e.g. PPO~\citep{schulman2017proximal}, DQN~\citep{mnih2013playing}).
\textbf{(3)} We show by experiments that \ours achieve \textbf{improved robustness} against various adversarial attacks as well as \textbf{higher efficiency}, compared with state-of-the-art (SOTA) robust RL methods in many MuJoCo and Atari games. 
For example, compared to the SOTA algorithm PA-ATLA-PPO~\citep{sun2021strongest} in the Walker environment, we obtain 20\% more worst-case reward (under the strongest attack algorithm), with only about 50\% training samples and 50\% running time.
Moreover, \ours learns \textbf{more interpretable ``robust behaviors''} than PA-ATLA-PPO in Walker as shown in Figure~\ref{fig:walker}.

\begin{figure}[!hb]
\vspace{-0.5em}
    \centering
    \includegraphics[width=0.85\columnwidth]{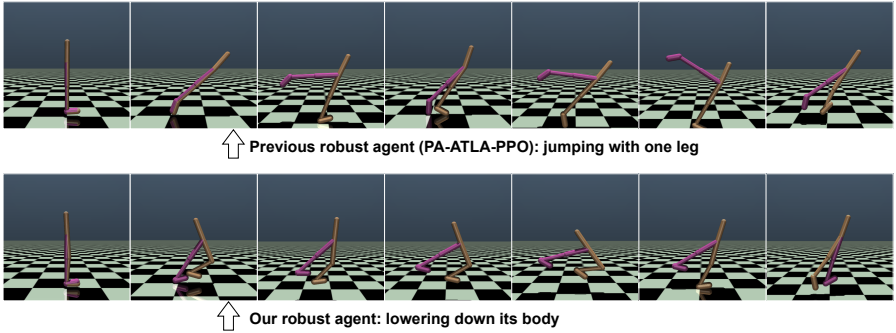}
    \vspace{-0.8em}
    \caption{\small{The robust Walker agents trained with \textbf{(top)} the state-of-the-art method PA-ATLA-PPO~\citep{sun2021strongest} and \textbf{(bottom)} our \ours.
    Although PA-ATLA-PPO agent also achieves high reward under attacks, it learns to jump with one leg, which is counter-intuitive and may indicate some level of overfitting to a specific attacker. 
    In contrast, our \ours agent learns to lower down its body, which is more intuitive and interpretable. 
    The full agent trajectories in Walker and other environments are provided in supplementary materials as GIF figures.
    }}
    \label{fig:walker}
\vspace{-1em}
\end{figure}

\vspace{-0.5em}
\section{Related Work}
\label{sec:related}
\vspace{-0.5em}



\textbf{Defending against Adversarial Perturbations on State Observations.}\quad
\textbf{(1)} 
\textit{Regularization-based methods}~\citep{zhang2020robust, shen2020deep,oikarinen2020robust} enforce the policy to have similar outputs under similar inputs, 
which achieves certifiable performance for DQN in some Atari games. But in continuous control tasks, these methods may not reliably improve the worst-case performance. A recent work by Korkmaz~\citep{korkmaz2021investigating} points out that these adversarially trained models may still be sensible to new perturbations.
\textbf{(2)} \textit{Attack-driven methods} train DRL agents with adversarial examples. Some early works~\citep{kos2017delving, behzadan2017whatever, mandlekar2017adversarially, pattanaik2017robust} apply weak or strong gradient-based attacks on state observations to train RL agents against adversarial perturbations. 
Zhang et al.~\citep{zhang2021robust} propose Alternating Training with Learned Adversaries (ATLA), which alternately trains an RL agent and an RL adversary and significantly improves the policy robustness in continuous control games.
Sun et al.~\citep{sun2021strongest} further extend this framework to PA-ATLA with their proposed more advanced RL attacker PA-AD. Although ATLA and PA-ATLA achieve strong empirical robustness, they require training an extra RL adversary that can be computationally and sample expensive. 
\textbf{(3)} There is another line of work studying \textit{certifiable robustness} of RL policies. 
Several works~\citep{lutjens2020certified, oikarinen2020robust, fischer2019online} computed lower bounds of the action value network $Q^\pi$ to certify robustness of action selection at every step. However, these bounds do not consider the distribution shifts caused by attacks, so some actions that appear safe for now can lead to extremely vulnerable future states and low long-term reward under future attacks. Moreover, these methods cannot apply to continuous action spaces. 
Kumar et al. and Wu et al.\citep{kumar2021policy, wu2021crop} both extend randomized smoothing~\citep{cohen2019certified} to derive robustness certificates for trained policies. But these works mostly focus on theoretical analysis, and effective robust training approaches rather than robust training.

\textbf{Adversarial Defenses against Other Adversarial Attacks.}\quad
Besides observation perturbations, attacks can happen in many other scenarios. For example, the agent's executed actions can be perturbed~\citep{xiao2019characterizing,tan2020robustifying, tessler2019action,lee2020query}. Moreover, in a multi-agent game, an agent's behavior can create adversarial perturbations to a victim agent~\citep{gleave2019adversarial}.
Pinto et al.~\citep{pinto2017robust} model the competition between the agent and the attacker as a zero-sum two-player game, and train the agent under a learned attacker to tolerate both environment shifts and adversarial disturbances.
We point out that although we mainly consider state adversaries, our \ours can be extended to action attacks as formulated in Appendix~\ref{app:extension}. 
Note that we focus on robustness against test-time attacks, different from poisoning attacks which alter the RL training process~\citep{behzadan2017vulnerability,huang2019deceptive,sun2020vulnerability,zhang2020adaptive,rakhsha2020policy}.

\textbf{Safe RL and Risk-sensitive RL.}\quad 
There are several lines of work that study RL under safety/risk constraints~\citep{Heger1994ConsiderationOR,gaskett2003reinforcement,garcia2015comprehensive,bechtle2020curious,thomas2021safe} or under intrinsic uncertainty of environment dynamics~\citep{lim2013reinforcement,Mankowitz2020Robust}. However, these works do not deal with adversarial attacks, which can be adaptive to the learned policy. More comparison between these methods and our proposed method is discussed in Section~\ref{sec:alg}.

\vspace{-0.5em}
\section{Preliminaries and Background}
\label{sec:prelim}
\setlength\abovedisplayskip{2pt}
\setlength\belowdisplayskip{2pt}
\vspace{-0.5em}
\textbf{Reinforcement Learning (RL).}\quad
An RL environment is modeled by a Markov Decision Process (MDP), denoted by a tuple $\mdp=\langle\states, \actions, P, R, \gamma\rangle$, 
where $\states$ is a state space, $\actions$ is an action space, $P: \states \times \actions \rightarrow \Delta(\states)$ is a stochastic dynamics model\footnote{$\Delta(\mathcal{X})$ denotes the space of probability distributions over $\mathcal{X}$.}, $R: \states \times \actions \rightarrow \mathbb{R}$ is a reward function and $\gamma \in[0,1)$ is a discount factor. 
An agent takes actions based on a policy $\pi: \states \rightarrow \Delta(\actions)$. For any policy, its \emph{natural performance} can be measured by the value function
$V^\pi(s) := \mathbb{E}_{P,\pi}[\sum_{t=0}^{\infty} \gamma^{t} R\left(s_t, a_t\right) \mid s_0=s]$,
and the action value function
$Q^\pi(s,a) := \mathbb{E}_{P,\pi}[\sum_{t=0}^{\infty} \gamma^{t} R\left(s_t, a_t\right) \mid s_0=s, a_0=a]$.
We call $V^\pi$ the \textit{natural value} and $Q^\pi$ the \textit{natural action value} in contrast to the values under attacks, as will be introduced in Section~\ref{sec:alg}.

\textbf{Deep Reinforcement Learning (DRL).}\quad
In large-scale problems, a policy can be parameterized by a neural network. For example, value-based RL methods (e.g. DQN~\citep{mnih2013playing}) usually fit a Q network and take the greedy policy $\pi(s)=\mathrm{argmax}_a Q(s,a)$. In actor-critic methods (e.g. PPO~\citep{schulman2017proximal}), the learner directly learns a policy network and a critic network. 
In practice, an agent usually follows a stochastic policy during training that enables exploration, and executes a trained policy deterministically in test-time, e.g. the greedy policy learned with DQN. 
Throughout this paper, we use $\pi_\theta$ to denote the training-time stochastic policy parameterized by $\theta$, while $\pi$ denotes the trained deterministic policy that maps a state to an action.

\textbf{Test-time Adversarial Attacks.}\quad 
After training, the agent is deployed into the environment and executes a pre-trained fixed policy $\pi$.
An attacker/adversary, during the deployment of the agent, may perturb the state observation of the agent/victim at every time step with a certain attack budget $\epsilon$. 
Note that the attacker only perturbs the inputs to the policy, and the underlying state in the environment does not change. This is a realistic setting because real-world observations can come from noisy sensors or be manipulated by malicious attacks. For example, an auto-driving car receives sensory observations; an attacker may add imperceptible noise to the camera, or perturb the GPS signal, although the underlying environment (the road) remains unchanged.
In this paper, we consider the $\ell_p$ \textit{thread model} which is widely used in adversarial learning literature: at step $t$, the attacker alters the observation $s_t$ into $\tilde{s}_t \in \attball(s_t)$, where $\attball(s_t)$ is a $\ell_p$ norm ball centered at $s_t$ with radius $\epsilon$. 
The above setting ($\ell_p$ constrained observation attack) is the same with many prior works~\citep{huang2017adversarial, pattanaik2017robust, zhang2020robust, zhang2021robust, sun2021strongest}.

\vspace{-0.5em}
\section{\oursfull}
\label{sec:alg}
\vspace{-0.5em}
In this section, we present \textit{\oursfull(\ours)}, a generic framework that can be fused with any DRL approach to improve the adversarial robustness of an agent.
We will introduce the three key mechanisms in \ours: \estname, \worstname, and \regname, respectively. 
Then, we will illustrate how to incorporate these mechanisms into existing DRL algorithms to improve their robustness.


\textbf{Mechanism 1: Worst-attack Value Estimation}


Traditional RL aims to learn a policy with the maximal value $V^\pi$. 
However, in a real-world problem where observations can be noisy or even adversarially perturbed, it is not enough to only consider the natural value $V^\pi$ and $Q^\pi$. As motivated in Figure~\ref{fig:example}, two policies with similar natural rewards can get totally different rewards under attacks.
To comprehensively evaluate how good a policy is in an adversarial scenario and to improve its robustness, we should be aware of the lowest possible long-term reward of the policy when its observation is adversarially perturbed with a certain attack budget $\epsilon$ at every step (with an $\ell_p$ attack model introduced in Section~\ref{sec:prelim}). 



The worst-case value of a policy is, by definition, the cumulative reward obtained under the optimal attacker. 
As justified by prior works~\citep{zhang2020robust,sun2021strongest}, for any given victim policy $\pi$ and attack budget $\epsilon>0$, there exists an optimal attacker, and finding the optimal attacker is equivalent to learning the optimal policy in another MDP. 
We denote the optimal (deterministic) attacker's policy as $h^*$. 
However, learning such an optimal attacker by RL algorithms requires extra interaction samples from the environment, due to the unknown dynamics. Moreover, learning the attacker by RL can be hard and expensive, especially when the state observation space is high-dimensional.

Instead of explicitly learning the optimal attacker with a large amount of samples, we propose to directly estimate the worst-case cumulative reward of the policy by characterizing the vulnerability of the given policy. 
We first define the \textit{\worstqname} of policy $\pi$ as $\worstqpi(s,a) := \mathbb{E}_{P} [\sum\nolimits_{t=0}^{\infty} \gamma^{t} R\left(s_t, \pi(h^*(s_t))\right) \mid s_0=s, a_0=a ].$
The \textit{\worstvname} $\worstvpi$ can be defined using $h^*$ in the same way, as shown in Definition~\ref{def:worstv} in Appendix~\ref{app:theory}.
Then, we introduce a novel operator $\worstbellpi$, namely the \textit{\bellmanname}, defined as below.

\begin{definition}[\bellmannamecap]
\label{def:bellman}
For MDP $\mathcal{M}$, given a fixed policy $\pi$ and attack radius $\epsilon$, define the \bellmanname
$\worstbellpi$ as
\begin{equation}
\label{eq:bellman}
\left(\worstbellpi Q\right)(s, a):=\mathbb{E}_{s^{\prime} \sim P(s, a)} [R(s, a)+\gamma \min_{a^{\prime} \in \mathcal{A}_{\mathrm{adv}}(s^\prime, \pi)} Q\left(s^{\prime}, a^{\prime}\right) ],
\end{equation}
where $\forall s\in\states$, $\mathcal{A}_{\mathrm{adv}}(s, \pi)$ is defined as
\begin{equation}
\label{eq:adv_action}
    \mathcal{A}_{\mathrm{adv}}(s, \pi) := \{a\in\mathcal{A}: \exists \tilde{s}\in\mathcal{B}_\epsilon(s) \text{ s.t. } \pi(\tilde{s}) = a \}.
\end{equation}
\end{definition}
\vspace{-0.5em}
Here $\mathcal{A}_{\mathrm{adv}}(s^\prime, \pi)$ denotes the set of actions an adversary can mislead the victim $\pi$ into selecting by perturbing the state $s^\prime$ into a neighboring state $\tilde{s}\in\mathcal{B}_\epsilon(s^\prime)$. This hypothetical perturbation to the \textit{future} state $s^\prime$ is the key for characterizing the worst-case long-term reward under attack. The following theorem associates the \bellmanname and the \worstqname.
\begin{theorem}[Worst-attack Bellman Operator and Worst-attack Action Value]
\label{thm:main}
For any given policy $\pi$, $\underline{\mathcal{T}}^\pi$ is a contraction whose fixed point is $\worstqpi$, the \worstqname of $\pi$ under any $\ell_p$ observation attacks with radius $\epsilon$.
\end{theorem}
\vspace{-0.5em}

Theorem~\ref{thm:main} proved in Appendix~\ref{app:theory} suggests that the lowest possible cumulative reward of a policy under bounded observation attacks can be computed by \bellmanname. The corresponding \worstvname $\worstvpi$ can be obtained by $\worstvpi(s)=\min_{a\in\mathcal{A}_{\mathrm{adv}}(s, \pi)} \worstqpi(s,a)$.

\textbf{How to Compute $\advaction$.}\quad 
To obtain $\mathcal{A}_{\mathrm{adv}}(s,\pi)$, we need to identify the actions that can be the outputs of the policy $\pi$ when the input state $s$ is perturbed within $\attball(s)$. This can be solved by commonly-used convex relaxation of neural networks~\citep{gowal2019scalable,zhang2018finding,wong2018provable,zhang2020towards,gowal2018effectiveness}, where layer-wise lower and upper bounds of the neural
network are derived.
That is, we calculate $\overline{\pi}$ and $\underline{\pi}$ such that $\overline{\pi}(s)\geq \pi(\hat{s}) \geq \underline{\pi}(s), \forall \hat{s}\in\attball(s)$. 
With such a relaxation, we can obtain a superset of $\advaction$, namely $\aadvaction$. Then, the fixed point of Equation~\eqref{eq:bellman} with $\advaction$ being replaced by $\aadvaction$ becomes a lower bound of the \worstqname.
For a continuous action space, $\aadvaction(s,\pi)$ contains actions bounded by $\overline{\pi}(s)$ and $\underline{\pi}(s)$. 
For a discrete action space, we can first compute the maximal and minimal probabilities of taking each action, and derive the set of actions that are likely to be selected.
The computation of $\aadvaction$ is not expensive, as
there are many efficient convex relaxation methods~\citep{mirman2018differentiable,zhang2020towards} which compute $\overline{\pi}$ and $\underline{\pi}$ with only constant-factor more computations than directly computing $\pi(s)$. 
Experiment in Section~\ref{sec:exp} verifies the efficiency of our approach, where we use a well-developed toolbox $\mathrm{auto\_LiRPA}$~\citep{xu2020automatic} to calculate the convex relaxation.
More implementation details and explanations are provided in Appendix~\ref{app:ibp}.
\textbf{Estimating Worst-attack Value.}\quad
Note that the \bellmanname $\worstbellpi$ is similar to the optimal Bellman operator $\mathcal{T}^*$, although it uses $\min_{a\in\advaction}$ instead of $\max_{a\in\actions}$. Therefore, once we identify $\advaction$ as introduced above, it is straightforward to compute the \worstqname using Bellman backups. 
To model the \worstqname, we train a network named \textit{\worstcriticname}, denoted by {\small{$\worstcritic$}}, where $\phi$ is the parameterization. 
Concretely, for any mini-batch $\{s_t,a_t,r_t,s_{t+1}\}_{t=1}^N$, {\small{$\worstcritic$}} is optimized by minimizing the following estimation loss:
\setlength\abovedisplayskip{2pt}
\setlength\belowdisplayskip{2pt}
\begin{align}
    \lossest(\worstcritic) \! :=\! \frac{1}{N} \sum_{t=1}^N(\underline{y}_t \!-\worstcritic(s_t,a_t))^{2}, 
    \text{where }\underline{y}_t \! = r_t + \gamma \min_{\hat{a}\in\advaction(s_{t+1},\pi)} \worstcritic(s_{t+1},\hat{a}). \label{loss:est} 
\end{align}
For a discrete action space, $\advaction$ is a discrete set and solving {\small{$\underline{y}_t$}} is straightforward.
For a continuous action space, we use gradient descent to approximately find the minimizer $\hat{a}$. 
Since $\advaction$ is in general small, this minimization is usually easy to solve. In MuJoCo, we find that 50-step gradient descent already converges to a good solution with little computational cost, as detailed in Appendix~\ref{app:exp:eff}.

\textbf{Differences with Worst-case Value Estimation in Related Work.} 
Our proposed \bellmanname is different from the worst-case Bellman operator in the literature of risk-sensitive RL~\citep{Heger1994ConsiderationOR,gaskett2003reinforcement,tamar2013scaling,garcia2015comprehensive,bechtle2020curious,thomas2021safe}, whose goal is to avoid unsafe trajectories under the intrinsic uncertainties of the MDP. 
These inherent uncertainties of the environment are independent of the learned policy. 
In contrast, our focus is to defend against adversarial perturbations created by malicious attackers that can be \emph{adaptive} to the policy. 
The GWC reward proposed by \cite{oikarinen2020robust} also estimates the worst-case reward of a policy under state perturbations. But their evaluation is based on a greedy strategy and requires interactions with the environment, which is different from our estimation.

\textbf{Mechanism 2: Worst-case-aware Policy Optimization}





So far we have introduced how to evaluate the \worstvname of a policy by learning a \worstcriticname. 
Inspired by the actor-critic framework, where the actor policy network $\pi_\theta$ is optimized towards a direction that the critic value increases the most, we can regard \worstcriticname as a special critic that directs the actor to increase the \worstvname. That is, we encourage the agent to select an action with a higher \worstqname, by minimizing the worst-attack policy loss below:
\setlength\abovedisplayskip{2pt}
\setlength\belowdisplayskip{2pt}
\begin{equation}
\label{loss:worst}
    \lossworst(\pi_\theta;\worstcritic) := - \frac{1}{N} \sum_{t=1}^N \sum_{a\in\actions} \pi_\theta(a|s_t) \worstcritic(s_t, a),
\end{equation}
where $\worstcritic$ is the \worstcriticname learned via $\lossest$ introduced in Equation~\eqref{loss:est}.
Note that $\lossworst$ is a general form, while the detailed implementation of the worst-attack policy optimization can vary depending on the architecture of $\pi_\theta$ in the base RL algorithm (e.g. PPO has a policy network, while DQN acts using the greedy policy induced by a Q network). In Appendix~\ref{app:ppo} and Appendix~\ref{app:dqn}, we illustrate how to implement $\lossworst$ for PPO and DQN as two examples.

The proposed \worstname has several \textbf{merits} compared to prior ATLA~\citep{zhang2021robust} and PA-ATLA~\citep{sun2021strongest} methods which alternately train the agent and an RL attacker. 
\textbf{(1)} Learning the optimal attacker $h^*$ requires collecting extra samples using the current policy (on-policy estimation). In contrast, {\small{$\worstcritic$}} can be learned using off-policy samples, e.g., historical samples in the replay buffer, and thus is more suitable for training where the policy changes over time. ({\small{$\worstcritic$}} depends on the current policy via the computation of {\small{$\advaction$}}.)
\textbf{(2)} We properly exploit the policy function that is being trained by computing the set of possibly selected actions {\small{$\aadvaction$}} for any state. In contrast, ATLA~\citep{zhang2021robust} learns an attacker by treating the current policy as a black box, ignoring the intrinsic properties of the policy. PA-ATLA~\citep{sun2021strongest}, although assumes white-box access to the victim policy, also needs to explore and learn from extra on-policy interactions. 
\textbf{(3)} The attacker trained with DRL methods, namely $\hat{h}^*$, is not guaranteed to converge to an optimal solution, such that the performance of $\pi$ estimated under $\hat{h}^*$ can be overly optimistic.
Our estimation, as mentioned in Mechanism 1, computes a lower bound of $\worstqpi$ and thus can better indicate the robustness of a policy.

\textbf{Mechanism 3: Value-enhanced State Regularization}

As discussed in Section~\ref{sec:intro}, the vulnerability of a deep policy comes from both the policy's intrinsic vulnerability with the RL dynamics and the DNN approximator. The first two mechanisms of \ours mainly focus on the policy's intrinsic vulnerability, i.e., let the policy select actions that are less vulnerable to possible attacks in all future steps. However, if a bounded state perturbation can cause the network to output a very different action, then the $\advaction$ set will be large and $\worstqpi$ can thus be low. Therefore, it is also important to encourage the trained policy to output similar actions for the clean state $s$ and any $\tilde{s}\in\attball(s)$, as is done in prior work~\citep{zhang2020robust,shen2020deep,fischer2019online}. 

But different from these prior methods, we note that different states should be treated differently. Some states are ``critical'' where selecting a bad action will result in catastrophic consequences. For example, when the agent gets close to the bomb in Figure~\ref{fig:example}, we should make the network more resistant to adversarial state perturbations.
To differentiate states based on their impacts on future reward, we propose to measure the importance of states with Definition~\ref{def:weight} below.

\begin{definition}[State Importance Weight]
\label{def:weight}
\!\!Define state importance weight of $s\in\states$ for policy $\pi$ as 
\begin{equation}
\setlength\abovedisplayskip{2pt}
\setlength\belowdisplayskip{2pt}
    w(s) = \max_{a_1 \in \mathcal{A}} Q^\pi(s, a_1) - \min_{a_2 \in \mathcal{A}} Q^\pi(s, a_2).\label{eq:weight}
\end{equation}
\end{definition}
\vspace{-1em}

\begin{wrapfigure}{r}{0.27\textwidth}
\centering
\vspace{-1em}
    \begin{subfigure}[t]{0.13\columnwidth}
        \centering
        \includegraphics[width=\columnwidth]{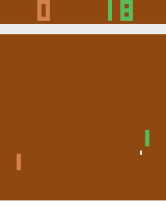}
        \vspace{-0.5em}
    \end{subfigure}
    \hfill
    \begin{subfigure}[t]{0.13\columnwidth}
        \centering
        \includegraphics[width=\columnwidth]{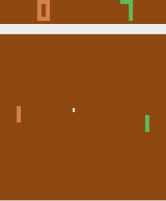}
        \vspace{-0.5em}
    \end{subfigure}
    \vspace{-1em}
    \caption{\small{States in Pong with \\
    \textbf{(left)} high weight $w(s)$ and \\
    \textbf{(right)} low weight $w(s)$.}}
    \label{fig:pong}
    \vspace{-1em}
\end{wrapfigure}
To justify whether Definition~\ref{def:weight} can characterize state importance, we train a DQN network in an Atari game Pong, and show the states with the highest weight and the lowest weight in Figure~\ref{fig:pong}, among many state samples. 
We can see that the state with higher weight in Figure~\ref{fig:pong}(left) is indeed crucial for the game, as the green agent paddle is close to the ball. Conversely, a less-important state in Figure~\ref{fig:pong}(right) does not have significantly different future rewards under different actions.
Computing $w(s)$ is easy in a discrete action space, while in a continuous action space, one can use gradient descent to approximately find the maximal and the minimal Q values for a state. 
Similar to the computation of Equation~\eqref{loss:est} with a continuous action space, we find that a 50-step gradient descent works well in experiments.


By incorporating the state importance weight $w(s)$, we regularize the policy network and let it pay more attention to more crucial states, by minimizing the following loss:
\setlength\abovedisplayskip{2pt}
\setlength\belowdisplayskip{2pt}
\begin{equation}
\label{loss:reg}
    \lossreg(\pi_\theta) = \frac{1}{N} \sum_{t=1}^N w(s_t) \max_{\tilde{s}_t\in\attball(s_t)} \mathsf{Dist} (\pi_\theta(s_t), \pi_\theta(\tilde{s}_t)),
\end{equation}
where $\mathsf{Dist}$ can be any distance measure between two distributions (e.g., KL-divergence). Minimizing $\lossreg$ can result in a smaller $\advaction$, and thus the \worstvname will be closer to the natural value.


\textbf{\ours: A Generic Robust Training Framework}


\begin{wrapfigure}{r}{0.55\textwidth}
\vspace{-2em}
    \centering
    \includegraphics[width=0.54\columnwidth]{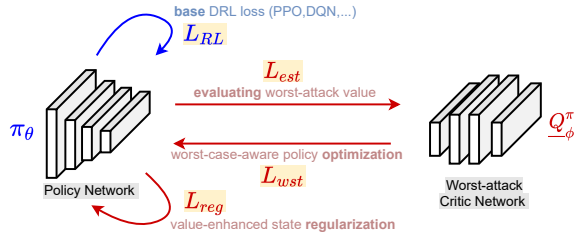}
    \vspace{-0.5em}
    \caption{Training architecture of \ours. (Components proposed in this paper are colored as red.) 
    }
    \label{fig:diagram_loss}
\vspace{-1em}
\end{wrapfigure}

So far we have introduced three key mechanisms and their loss functions, $\lossest$ in Equation~\eqref{loss:est}, $\lossworst$ in Equation~\eqref{loss:worst} and $\lossreg$ in Equation~\eqref{loss:reg}.
Then, our robust training framework \ours combines these losses with any base RL algorithm.
To be more specific, as shown in Figure~\ref{fig:diagram_loss}, for any base RL algorithm that trains policy $\pi_\theta$ using loss $\lossrl$, we learn an extra \worstcriticname network $\worstcritic$ by minimizing
\setlength\abovedisplayskip{2pt}
\setlength\belowdisplayskip{2pt}
\begin{equation}
\label{loss:critic}
    \loss_{\worstcritic} := \lossest(\worstcritic),
\end{equation}
and combine $\lossworst$ and $\lossreg$ with $\lossrl$ to optimize $\pi_\theta$ by minimizing
\setlength\abovedisplayskip{2pt}
\setlength\belowdisplayskip{2pt}
\begin{equation}
\label{loss:policy}
    \loss_{\pi_\theta} := \lossrl(\pi_\theta) + \weightworst \lossworst(\pi_\theta;\worstcritic) + \weightreg \lossreg(\pi_\theta),
\end{equation}
where $\weightworst$ and $\weightreg$ are hyperparameters balancing between natural performance and robustness. 
Note that $\worstcritic$ is trained together but independently with $\pi_\theta$ using historical transition samples, so \ours does not require extra samples from the environment. 
\ours can also be interpreted from a geometric perspective based on prior RL polytope theory~\citep{dadashi2019value,sun2021strongest} as detailed in Appendix~\ref{app:understand}.

Our \ours is a generic robust training framework that can be used to robustify existing DRL algorithms. 
We provide two case studies: 
\textbf{(1)} combining \ours with a policy-based algorithm PPO~\citep{schulman2017proximal}, namely \textit{\ourppo}, and 
\textbf{(2)} combining \ours with a value-based algorithm DQN~\citep{mnih2013playing}, namely \textit{\ourdqn}. 
The pseudocodes of \ourppo and \ourdqn are illustrated in Appendix~\ref{app:ppo} and Appendix~\ref{app:dqn}.
The application of \ours to other DRL methods is then straightforward, since most DRL methods are either policy-based or value-based.
Next, we show by experiments that \ourppo and \ourdqn achieve state-of-the-art robustness with superior efficiency, in various continuous control tasks and video game environments. 
We also empirically verify the effectiveness of each of the 3 mechanisms of \ours and their weights by ablation study in Section~\ref{sec:result_effect}.

\vspace{-0.5em}
\section{Experiments and Discussion}
\label{sec:exp}
\vspace{-0.5em}



In this section, our experimental evaluations on various MuJoCo and Atari environments aim to study the following questions:
\textbf{(1)} Can \ours learn policies with better \textbf{robustness} under existing strong adversarial attacks?
\textbf{(2)} Can \ours maintain \textbf{natural performance} when improving robustness?
\textbf{(3)} Can \ours learn more \textbf{efficiently} during robust training?
\textbf{(4)} Is each mechanism in \ours \textbf{effective}? 
Problem (1), (2) and (3) are answered in Section~\ref{sec:result_main} with detailed empirical results, and problem (4) is studied in Section~\ref{sec:result_effect} via ablation experiments.

\subsection{Experiments and Evaluations}
\label{sec:result_main}

\textbf{Environments.}\quad
Following most prior works~\citep{zhang2020robust,zhang2021robust,oikarinen2020robust} and the released implementation, we apply our \ours to PPO~\citep{schulman2017proximal} on 4 MuJoCo tasks with continuous action spaces, including Hopper, Walker2d, Halfcheetah and Ant, and to DQN~\citep{mnih2013playing} agents on 4 Atari games including Pong, Freeway, BankHeist and RoadRunner, which have high dimensional pixel inputs and discrete action spaces. 

\textbf{Baselines and Implementation.}\quad
We compare our algorithm with several state-of-the-art robust training methods, including
(1) \emph{SA-PPO/SA-DQN}~\cite{zhang2020robust}: regularizing policy networks by convex relaxation.
(2) \emph{ATLA-PPO}~\cite{zhang2021robust}: alternately training an agent and an RL attacker.
(3) \emph{PA-ATLA-PPO}~\cite{sun2021strongest}: alternately training an agent and a more advanced RL attacker PA-AD. 
(4) \emph{RADIAL-PPO/RADIAL-DQN}~\cite{oikarinen2020robust}: optimizing policy network by designed adversarial loss functions based on robustness bounds.
SA and RADIAL have both PPO and DQN versions, which are compared with our \ourppo and \ourdqn.
But ATLA and PA-ATLA do not provide DQN versions, since alternately training on DQN can be expensive as explained in the original papers~\citep{sun2021strongest}. (PA-ATLA has an A2C version, which we compare in Appendix~\ref{app:exp:res}.)
Therefore, we reproduce their ATLA-PPO and PA-ATLA-PPO results and compare them with our \ourppo.
More implementation and hyperparameter details are provided in Appendix~\ref{app:exp:imp}.


\textbf{Case I: Robust PPO for MuJoCo Continuous Control}


\textbf{Evaluation Metrics.}\quad
To reflect both the natural performance and robustness of trained agents, we report the average episodic rewards under no attack and against various attacks. For a comprehensive robustness evaluation, we attack the trained robust models with multiple existing attack methods, including: 
(1) \textit{MaxDiff} \cite{zhang2020robust} (maximal action difference), 
(2) \textit{Robust Sarsa (RS)} \cite{zhang2020robust} (attacking with a robust action-value function), 
(3) \textit{SA-RL} \citep{zhang2020robust} (finding the optimal state adversary) and 
(4) \textit{PA-AD} \citep{sun2021strongest} (the \textit{existing strongest attack} by learning the optimal policy adversary with RL). 
For a clear comparison, we use the same attack radius $\epsilon$ as in most baselines~\citep{zhang2020robust,zhang2021robust,sun2021strongest}.

\textbf{Performance and Robustness of \ourppo}\quad
Figure~\ref{fig:mujoco_curve} (left four columns) shows performance curves during training under four different adversarial attacks. 
Among all four attack algorithms, \ourppo converges much faster than baselines, and often achieves the best asymptotic robust performance, especially under the strongest PA-AD attack. It is worth emphasizing that since we train a robust agent without explicitly learning an RL attacker, our method not only obtains stronger robustness and much higher efficiency, but also a more general defense: \ourppo obtains comprehensively superior performance against a variety of attacks compared against existing SOTA algorithms based on learned attackers (ATLA-PPO, PA-ATLA-PPO). 
Additionally, in our experiments, \ourppo learns relatively more universal defensive behaviors as shown in Figure~\ref{fig:walker}, which can physically explain why our algorithm can defend against diverse attacks. We provide policy demonstrations in multiple tasks in our supplementary materials.\\
The comparison of natural performance and the worst-case performance appears in Figure~\ref{fig:mujoco_curve} (right). We see that \ourppo maintains competitive natural rewards under no attack compared with other baselines, which demonstrates that our algorithm gains more robustness without losing too much natural performance. 
The full results of baselines and our algorithm under different attack evaluations are provided by Table~\ref{tab:mujoco_app} in Appendix~\ref{app:exp:res} (including performance under random attacks).

\begin{figure}[!t]
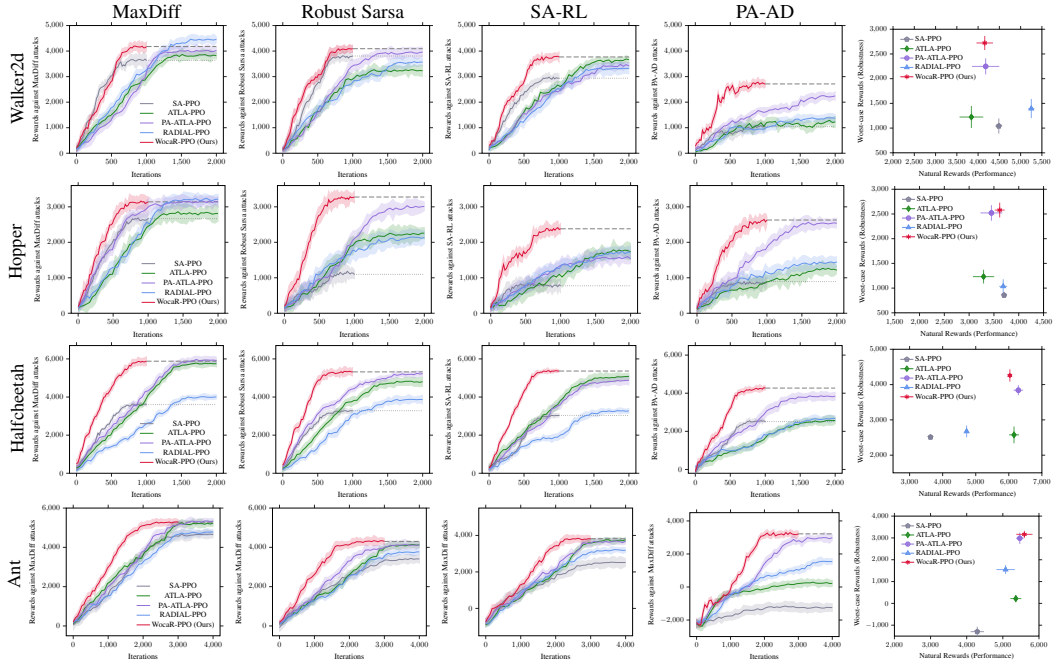

    \centering
    \rotatebox{90}{\scriptsize{\qquad \quad Walker2d}}
    \begin{subfigure}[t]{0.19\textwidth}
        \centering
        \resizebox{\textwidth}{!}{\input{\fighome/curve_atla/walker_mad.tex}}
        \vspace{-1.5em}
    \end{subfigure}
    \begin{subfigure}[t]{0.19\textwidth}
        \centering
        \resizebox{\textwidth}{!}{\input{\fighome/curve_atla/walker_rs.tex}}
        \vspace{-1.5em}
    \end{subfigure}
    \begin{subfigure}[t]{0.19\textwidth}
        \centering
        \resizebox{\textwidth}{!}{\input{\fighome/curve_atla/walker_sa.tex}}
        \vspace{-1.5em}
    \end{subfigure}
    \begin{subfigure}[t]{0.19\textwidth}
        \centering
        \resizebox{\textwidth}{!}{\input{\fighome/curve_atla/walker_pa.tex}}
        \vspace{-1.5em}
    \end{subfigure}
    \begin{subfigure}[t]{0.19\textwidth}
        \centering
        \resizebox{\textwidth}{!}{
\begin{tikzpicture}

\definecolor{color0}{rgb}{0.917647058823529,0.917647058823529,0.949019607843137}
\definecolor{color1}{rgb}{0.505434799735519,0.504732659736887,0.601761241571282}
\definecolor{color2}{rgb}{0.133333333333333,0.545098039215686,0.133333333333333}
\definecolor{color3}{rgb}{0.576470588235294,0.43921568627451,0.858823529411765}
\definecolor{color4}{rgb}{0.392156862745098,0.584313725490196,0.929411764705882}
\definecolor{color5}{rgb}{0.862745098039216,0.0784313725490196,0.235294117647059}

\begin{axis}[
axis background/.style={fill=white},
axis line style={black},
legend cell align={left},
legend style={
  fill opacity=0.8,
  draw opacity=1,
  text opacity=1,
  at={(0.03,0.97)},
  anchor=north west,
  draw=none,
  fill=none
},
tick align=outside,
x grid style={white},
xlabel={Natural Rewards (Performance)},
xmajorgrids,
xmin=2000, xmax=5500,
xtick style={color=white!15!black},
y grid style={white},
ylabel={Worst-case Rewards (Robustness)},
ymajorgrids,
ymin=500, ymax=3000,
ytick style={color=white!15!black}
]
\path [draw=color1, draw opacity=0.8, thick]
(axis cs:4406,1042)
--(axis cs:4568,1042);

\path [draw=color1, draw opacity=0.8, thick]
(axis cs:4487,889)
--(axis cs:4487,1195);

\path [draw=color2, draw opacity=0.8, thick]
(axis cs:3563,1224)
--(axis cs:4121,1224);

\path [draw=color2, draw opacity=0.8, thick]
(axis cs:3842,1001)
--(axis cs:3842,1447);

\path [draw=color3, draw opacity=0.8, thick]
(axis cs:3858,2248)
--(axis cs:4498,2248);

\path [draw=color3, draw opacity=0.8, thick]
(axis cs:4178,2086)
--(axis cs:4178,2410);

\path [draw=color4, draw opacity=0.8, thick]
(axis cs:5181,1395)
--(axis cs:5321,1395);

\path [draw=color4, draw opacity=0.8, thick]
(axis cs:5251,1201)
--(axis cs:5251,1589);

\path [draw=color5, draw opacity=0.8, thick]
(axis cs:3961,2722)
--(axis cs:4351,2722);

\path [draw=color5, draw opacity=0.8, thick]
(axis cs:4156,2580)
--(axis cs:4156,2864);

\addplot [line width=1pt, color1, mark=pentagon*, mark size=3.2, mark options={solid}]
table{%
x  y
4487 1042
4487 1042
};
\addlegendentry{SA-PPO}

\addplot [line width=1pt, color2, mark=diamond*, mark size=3.5, mark options={solid}]
table{%
x  y
3842 1224
3842 1224
};
\addlegendentry{ATLA-PPO}
\addplot [line width=1pt, color3, mark=*, mark size=2.9, mark options={solid}]
table{%
x  y
4178 2248
4178 2248
};
\addlegendentry{PA-ATLA-PPO}
\addplot [line width=1pt, color4, mark=triangle*, mark size=3.3]
table{%
x  y
5251 1395
5251 1395
};
\addlegendentry{RADIAL-PPO}
\addplot [line width=1pt, color5, mark=asterisk, mark size=3.2]
table{%
x  y
4156 2722
4156 2722
};
\addlegendentry{\ourppo (Ours)}
\end{axis}

\end{tikzpicture}}
        \vspace{-1.5em}
    \end{subfigure}\\
    \rotatebox{90}{\scriptsize{\qquad \quad \ Hopper}}
    \begin{subfigure}[t]{0.19\textwidth}
        \centering
        \resizebox{\textwidth}{!}{\input{\fighome/curve_atla/hopper_mad.tex}}
        \vspace{-1.5em}
    \end{subfigure}
    \begin{subfigure}[t]{0.19\textwidth}
        \centering
        \resizebox{\textwidth}{!}{\input{\fighome/curve_atla/hopper_rs.tex}}
        \vspace{-1.5em}
    \end{subfigure}
    \begin{subfigure}[t]{0.19\textwidth}
        \centering
        \resizebox{\textwidth}{!}{\input{\fighome/curve_atla/hopper_sa.tex}}
        \vspace{-1.5em}
    \end{subfigure}
    \begin{subfigure}[t]{0.19\textwidth}
        \centering
        \resizebox{\textwidth}{!}{\input{\fighome/curve_atla/hopper_pa.tex}}
        \vspace{-1.5em}
    \end{subfigure}
    \begin{subfigure}[t]{0.19\textwidth}
        \centering
        \resizebox{\textwidth}{!}{
\begin{tikzpicture}

\definecolor{color0}{rgb}{0.917647058823529,0.917647058823529,0.949019607843137}
\definecolor{color1}{rgb}{0.505434799735519,0.504732659736887,0.601761241571282}
\definecolor{color2}{rgb}{0.133333333333333,0.545098039215686,0.133333333333333}
\definecolor{color3}{rgb}{0.576470588235294,0.43921568627451,0.858823529411765}
\definecolor{color4}{rgb}{0.392156862745098,0.584313725490196,0.929411764705882}
\definecolor{color5}{rgb}{0.862745098039216,0.0784313725490196,0.235294117647059}

\begin{axis}[
axis background/.style={fill=white},
axis line style={black},
legend cell align={left},
legend style={
  fill opacity=0.8,
  draw opacity=1,
  text opacity=1,
  at={(0.03,0.97)},
  anchor=north west,
  draw=none,
  fill=none
},
tick align=outside,
x grid style={white},
xlabel={Natural Rewards (Performance)},
xmajorgrids,
xmin=1500, xmax=4500,
xtick style={color=white!15!black},
y grid style={white},
ylabel={Worst-case Rewards (Robustness)},
ymajorgrids,
ymin=500, ymax=3000,
ytick style={color=white!15!black}
]
\path [draw=color1, draw opacity=0.8, thick]
(axis cs:3700,856)
--(axis cs:3710,856);

\path [draw=color1, draw opacity=0.8, thick]
(axis cs:3705,835)
--(axis cs:3705,877);

\path [draw=color2, draw opacity=0.8, thick]
(axis cs:3077,1232)
--(axis cs:3505,1232);

\path [draw=color2, draw opacity=0.8, thick]
(axis cs:3291,1093)
--(axis cs:3291,1371);

\path [draw=color3, draw opacity=0.8, thick]
(axis cs:3229,2521)
--(axis cs:3669,2521);

\path [draw=color3, draw opacity=0.8, thick]
(axis cs:3449,2363)
--(axis cs:3449,2679);

\path [draw=color4, draw opacity=0.8, thick]
(axis cs:3607,1036)
--(axis cs:3767,1036);

\path [draw=color4, draw opacity=0.8, thick]
(axis cs:3687,894)
--(axis cs:3687,1178);

\path [draw=color5, draw opacity=0.8, thick]
(axis cs:3517,2579)
--(axis cs:3715,2579);

\path [draw=color5, draw opacity=0.8, thick]
(axis cs:3616,2429)
--(axis cs:3616,2729);

\addplot [line width=1pt, color1, mark=pentagon*, mark size=3.2, mark options={solid}]
table{%
x  y
3705 856
3705 856
};
\addlegendentry{SA-PPO}

\addplot [line width=1pt, color2, mark=diamond*, mark size=3.5, mark options={solid}]
table{%
x  y
3291 1232
3291 1232
};
\addlegendentry{ATLA-PPO}
\addplot [line width=1pt, color3, mark=*, mark size=2.9, mark options={solid}]
table{%
x  y
3449 2521
3449 2521
};
\addlegendentry{PA-ATLA-PPO}
\addplot [line width=1pt, color4, mark=triangle*, mark size=3.3]
table{%
x  y
3687 1036
3687 1036
};
\addlegendentry{RADIAL-PPO}
\addplot [line width=1pt, color5, mark=asterisk, mark size=3.2]
table{%
x  y
3616 2579
3616 2579
};
\addlegendentry{\ourppo (Ours)}
\end{axis}

\end{tikzpicture}}
        \vspace{-1.5em}
    \end{subfigure}\\
    \rotatebox{90}{\scriptsize{\quad \quad \ \ Halfcheetah}}
    \begin{subfigure}[t]{0.19\textwidth}
        \centering
        \resizebox{\textwidth}{!}{\input{\fighome/curve_atla/halfcheetah/halfcheetah_mad.tex}}
        \vspace{-1.5em}
    \end{subfigure}
    \begin{subfigure}[t]{0.19\textwidth}
        \centering
        \resizebox{\textwidth}{!}{\input{\fighome/curve_atla/halfcheetah/halfcheetah_rs.tex}}
        \vspace{-1.5em}
    \end{subfigure}
    \begin{subfigure}[t]{0.19\textwidth}
        \centering
        \resizebox{\textwidth}{!}{\input{\fighome/curve_atla/halfcheetah/halfcheetah_sa.tex}}
        \vspace{-1.5em}
    \end{subfigure}
    \begin{subfigure}[t]{0.19\textwidth}
        \centering
        \resizebox{\textwidth}{!}{\input{\fighome/curve_atla/halfcheetah/halfcheetah_pa.tex}}
        \vspace{-1.5em}
    \end{subfigure}
    \begin{subfigure}[t]{0.19\textwidth}
        \centering
        \resizebox{\textwidth}{!}{
\begin{tikzpicture}

\definecolor{color0}{rgb}{0.917647058823529,0.917647058823529,0.949019607843137}
\definecolor{color1}{rgb}{0.505434799735519,0.504732659736887,0.601761241571282}
\definecolor{color2}{rgb}{0.133333333333333,0.545098039215686,0.133333333333333}
\definecolor{color3}{rgb}{0.576470588235294,0.43921568627451,0.858823529411765}
\definecolor{color4}{rgb}{0.392156862745098,0.584313725490196,0.929411764705882}
\definecolor{color5}{rgb}{0.862745098039216,0.0784313725490196,0.235294117647059}
\begin{axis}[
axis background/.style={fill=white},
axis line style={black},
legend cell align={left},
legend style={
  fill opacity=0.8,
  draw opacity=1,
  text opacity=1,
  at={(0.03,0.97)},
  anchor=north west,
  draw=none,
  fill=none
},
tick align=outside,
x grid style={white},
xlabel={Natural Rewards (Performance)},
xmajorgrids,
xmin=2500, xmax=7000,
xtick style={color=white!15!black},
y grid style={white},
ylabel={Worst-case Rewards (Robustness)},
ymajorgrids,
ymin=1500, ymax=5000,
ytick style={color=white!15!black}
]
\path [draw=color1, draw opacity=0.8, thick]
(axis cs:3612, 2512)
--(axis cs:3652, 2512);

\path [draw=color1, draw opacity=0.8, thick]
(axis cs:3632, 2496)
--(axis cs:3632, 2528);

\path [draw=color2, draw opacity=0.8, thick]
(axis cs:6005, 2576)
--(axis cs:6309, 2576);

\path [draw=color2, draw opacity=0.8, thick]
(axis cs:6157, 2338)
--(axis cs:6157, 2814);

\path [draw=color3, draw opacity=0.8, thick]
(axis cs:6147, 3840)
--(axis cs:6431, 3840);

\path [draw=color3, draw opacity=0.8, thick]
(axis cs:6289, 3698)
--(axis cs:6289, 3982);

\path [draw=color4, draw opacity=0.8, thick]
(axis cs:4710, 2674)
--(axis cs:4738, 2674);

\path [draw=color4, draw opacity=0.8, thick]
(axis cs:4724, 2506)
--(axis cs:4724, 2842);

\path [draw=color5, draw opacity=0.8, thick]
(axis cs:5964, 4258)
--(axis cs:6100, 4258);

\path [draw=color5, draw opacity=0.8, thick]
(axis cs:6032, 4086)
--(axis cs:6032, 4430);

\addplot [line width=1pt, color1, mark=pentagon*, mark size=3.2, mark options={solid}]
table{%
x  y
3632 2512
3632 2512
};
\addlegendentry{SA-PPO}

\addplot [line width=1pt, color2, mark=diamond*, mark size=3.5, mark options={solid}]
table{%
x  y
6157 2576
6157 2576
};
\addlegendentry{ATLA-PPO}
\addplot [line width=1pt, color3, mark=*, mark size=2.9, mark options={solid}]
table{%
x  y
6289 3840
6289 3840
};
\addlegendentry{PA-ATLA-PPO}
\addplot [line width=1pt, color4, mark=triangle*, mark size=3.3]
table{%
x  y
4724 2674
4724 2674
};
\addlegendentry{RADIAL-PPO}
\addplot [line width=1pt, color5, mark=asterisk, mark size=3.2]
table{%
x  y
6032 4258
6032 4258
};
\addlegendentry{\ourppo (Ours)}
\end{axis}

\end{tikzpicture}}
        \vspace{-1.5em}
    \end{subfigure}\\
    \rotatebox{90}{\scriptsize{\qquad \qquad Ant}}
    \begin{subfigure}[t]{0.19\textwidth}
        \centering
        \resizebox{\textwidth}{!}{\input{\fighome/curve_atla/ant/ant_mad.tex}}
        \vspace{-1.5em}
    \end{subfigure}
    \begin{subfigure}[t]{0.19\textwidth}
        \centering
        \resizebox{\textwidth}{!}{\input{\fighome/curve_atla/ant/ant_rs.tex}}
        \vspace{-1.5em}
    \end{subfigure}
    \begin{subfigure}[t]{0.19\textwidth}
        \centering
        \resizebox{\textwidth}{!}{\input{\fighome/curve_atla/ant/ant_sa.tex}}
        \vspace{-1.5em}
    \end{subfigure}
    \begin{subfigure}[t]{0.19\textwidth}
        \centering
        \resizebox{\textwidth}{!}{\input{\fighome/curve_atla/ant/ant_pa.tex}}
        \vspace{-1.5em}
    \end{subfigure}
    \begin{subfigure}[t]{0.19\textwidth}
        \centering
        \resizebox{\textwidth}{!}{
\begin{tikzpicture}

\definecolor{color0}{rgb}{0.917647058823529,0.917647058823529,0.949019607843137}
\definecolor{color1}{rgb}{0.505434799735519,0.504732659736887,0.601761241571282}
\definecolor{color2}{rgb}{0.133333333333333,0.545098039215686,0.133333333333333}
\definecolor{color3}{rgb}{0.576470588235294,0.43921568627451,0.858823529411765}
\definecolor{color4}{rgb}{0.392156862745098,0.584313725490196,0.929411764705882}
\definecolor{color5}{rgb}{0.862745098039216,0.0784313725490196,0.235294117647059}

\begin{axis}[
axis background/.style={fill=white},
axis line style={black},
legend cell align={left},
legend style={
  fill opacity=0.8,
  draw opacity=1,
  text opacity=1,
  at={(0.03,0.97)},
  anchor=north west,
  draw=none,
  fill=none
},
tick align=outside,
x grid style={white},
xlabel={Natural Rewards (Performance)},
xmajorgrids,
xmin=2000, xmax=6000,
xtick style={color=white!15!black},
y grid style={white},
ylabel={Worst-case Rewards (Robustness)},
ymajorgrids,
ymin=-1500, ymax=4000,
ytick style={color=white!15!black}
]
\path [draw=color1, draw opacity=0.8, thick]
(axis cs:4108,-1296)
--(axis cs:4476,-1296);

\path [draw=color1, draw opacity=0.8, thick]
(axis cs:4292,-1519)
--(axis cs:4292,-1073);

\path [draw=color2, draw opacity=0.8, thick]
(axis cs:5206,219)
--(axis cs:5512,219);

\path [draw=color2, draw opacity=0.8, thick]
(axis cs:5359,52)
--(axis cs:5359,386);

\path [draw=color3, draw opacity=0.8, thick]
(axis cs:5363,2986)
--(axis cs:5575,2986);

\path [draw=color3, draw opacity=0.8, thick]
(axis cs:5469,2726)
--(axis cs:5469,3246);

\path [draw=color4, draw opacity=0.8, thick]
(axis cs:4822,1549)
--(axis cs:5330,1549);

\path [draw=color4, draw opacity=0.8, thick]
(axis cs:5076,1346)
--(axis cs:5076,1752);

\path [draw=color5, draw opacity=0.8, thick]
(axis cs:5371,3164)
--(axis cs:5821,3164);

\path [draw=color5, draw opacity=0.8, thick]
(axis cs:5596,3001)
--(axis cs:5596,3328);

\addplot [line width=1pt, color1, mark=pentagon*, mark size=3.2, mark options={solid}]
table{%
x  y
4292 -1296
4292 -1296
};
\addlegendentry{SA-PPO}

\addplot [line width=1pt, color2, mark=diamond*, mark size=3.5, mark options={solid}]
table{%
x  y
5359 219 
5359 219 
};
\addlegendentry{ATLA-PPO}
\addplot [line width=1pt, color3, mark=*, mark size=2.9, mark options={solid}]
table{%
x  y
5469 2986
5469 2986
};
\addlegendentry{PA-ATLA-PPO}
\addplot [line width=1pt, color4, mark=triangle*, mark size=3.3]
table{%
x  y
5076 1549
5076 1549
};
\addlegendentry{RADIAL-PPO}
\addplot [line width=1pt, color5, mark=asterisk, mark size=3.2]
table{%
x  y
5596 3164
5596 3164
};
\addlegendentry{\ourppo (Ours)}
\end{axis}

\end{tikzpicture}}
        \vspace{-1.5em}
    \end{subfigure}
    \vspace{-0.5em}
    \caption{\small{ \textbf{Robustness, Efficiency and High Natural Performance of WocaR-PPO.}
    \textbf{(Left four columns)} Learning curves
    of rewards under MaxDiff, Robust Sarsa, SA-RL and PA-AD \textit{(the strongest)} attacks during training on four environments. 
    \textbf{(Rightmost column)} Average episode natural rewards v.s. average worst rewards under attacks.
    Each row shows the performance of baselines and \ourppo on one environment. Shaded regions are computed over 20 random seeds. Results under more attack radius $\epsilon$'s are in Appendix~\ref{app:exp:eps}.}}
    \label{fig:mujoco_curve}
\vspace{-1em}
\end{figure}

\textbf{Efficiency of Training \ourppo.}\quad
The learning curves in Figure~\ref{fig:mujoco_curve} (left) directly show the sample efficiency of \ourppo. Following the optimal settings provided in \citep{zhang2020robust, zhang2021robust, oikarinen2020robust}, our method takes \emph{50\% training steps} required by RADIAL-PPO and ATLA methods on Hopper, Walker2d, and Halfcheetah because RADIAL-PPO needs more steps to ensure convergence and ATLA methods require additional adversary training steps.  When solving high dimensional environments like Ant, \ourppo only requires \emph{75\% steps} compared with all other baselines to converge. We also provide additional results of baselines using the same training steps as \ourppo in Appendix~\ref{app:exp:add}.\\
In terms of time efficiency, \ourppo saves \emph{50\% training time} for convergence on Hopper, Walker2d, and Halfcheetah, and \emph{32\% time} on Ant compared with the SOTA method. Therefore, \textit{\ourppo achieves both higher computational efficiency and higher sample efficiency than SOTA baselines.} Detailed costs in time and sampling are in Appendix~\ref{app:exp:eff}. 


\textbf{Case II: Robust DQN for Atari Video Games}\quad

\textbf{Evaluation Metrics.}\quad
Since Atari games have pixel state spaces and discrete action spaces, the applicable attacking algorithms also differ from those in MuJoCo tasks.
We include the following common attacks: (1) 10-step untargeted \textit{PGD} (projected gradient descent) attack, (2) \textit{MinBest}~\citep{huang2017adversarial}, which minimizes the probability of choosing the “best” action, (3) \textit{PA-AD}~\citep{sun2021strongest}, as the state-of-the-art RL-based adversarial attack algorithm. 

\textbf{Performance and Robustness of \ourdqn.}\quad
Table~\ref{tab:atari} presents the results on four Atari games under attack radius $\epsilon=3/255$, while results and analysis under smaller attack radius $1/255$ are in Appendix~\ref{app:exp:res}. 
We can see that \textit{our \ourdqn consistently outperforms baselines under MinBest and PA-AD attacks in all environments, with a significant advance under the strongest (worst-case) PA-AD attacks compared with other robust agents.} Under PGD attacks, \ourdqn performs comparably with the state-of-the-art in Freeway and Pong (which are simpler games) and gains higher rewards than other agents in BankHeist and Roadrunner. Since SA-DQN and RADIAL-DQN focus on bounding and smoothing the policy network and do not consider the policy's intrinsic vulnerability, they are robust under the PGD attack but still vulnerable against the stronger PA-AD attack. 

\textbf{Efficiency of Training \ourdqn.}\quad The total training time for SA-DQN, RADIAL-DQN, and our \ourdqn are roughly 35, 17, and 18 hours, respectively.  All baselines are trained for 6 million frames on the same hardware. Therefore, \ourdqn is 49\% faster (and is more robust) than SA-DQN. Compared to the more advanced baseline RADIAL-DQN, although \ourdqn is 5\% slower, it achieves better robustness (539\% higher reward than RADIAL-DQN in RoadRunner).
\begin{table}[!t]
\centering
\renewcommand{\arraystretch}{1.2}
\resizebox{0.95\textwidth}{!}{%
\setlength{\tabcolsep}{4pt}
\begin{tabular}{p{3.3cm}<{\centering}| p{2.02cm}<{\centering} | p{2.02cm}<{\centering} | p{2.02cm}<{\centering} | p{2.02cm}<{\centering} | p{2.02cm}<{\centering} | p{2.02cm}<{\centering} | p{2.02cm}<{\centering} | p{2.02cm}<{\centering}}
\toprule
\multirow{3}{*}{\textbf{Model}} & \multicolumn{4}{c|}{\textbf{Pong} } & \multicolumn{4}{c}{\textbf{BankHeist}} \\
\cline{2-9}
& \multirow{2}{*}{\begin{tabular}[c]{@{}c@{}}\textbf{Natural}\\\textbf{Reward}\end{tabular}} & \textbf{PGD} & \textbf{MinBest} & \textbf{PA-AD} & \multirow{2}{*}{\begin{tabular}[c]{@{}c@{}}\textbf{Natural}\\\textbf{Reward}\end{tabular}} & \textbf{PGD} & \textbf{MinBest} & \textbf{PA-AD} \\
\cline{3-5}\cline{7-9}
& & \multicolumn{3}{c|}{$\epsilon$= 3/255} & & \multicolumn{3}{c}{$\epsilon$= 3/255} \\
\hline
DQN & 21.0 $\pm$ 0.0 & -21.0 $\pm$ 0.0 & -9.7 $\pm$ 4.0 & -19.0 $\pm$ 2.2 & \textbf{1308 $\pm$ 24} &  0 $\pm$ 0 &  119 $\pm$ 65 & 102 $\pm$ 92 \\
SA-DQN & 21.0 $\pm$ 0.0 & 21.0 $\pm$ 0.0 & 20.6 $\pm$ 3.5 & 18.7 $\pm$ 2.6& 1245 $\pm$ 14  & 1176 $\pm$ 63 &  1024 $\pm$ 31 &  489 $\pm$ 106 \\
RADIAL-DQN & 21.0 $\pm$ 0.0 & 21.0 $\pm$ 0.0 & 19.5 $\pm$ 2.1 & 13.2 $\pm$ 1.8 & 1178 $\pm$ 4 &  1176 $\pm$ 63 & 928 $\pm$ 113 & 508 $\pm$ 85  \\
\cellcolor{lightgray}{\textbf{\ourdqn (Ours)}} & 
\cellcolor{lightgray}{\textbf{21.0 $\pm$ 0.0}} &
\cellcolor{lightgray}{\textbf{21.0 $\pm$ 0.0}} &
\cellcolor{lightgray}{\textbf{20.8 $\pm$ 3.3}} &
\cellcolor{lightgray}{\textbf{19.7 $\pm$ 2.4}} & \cellcolor{lightgray}{1220 $\pm$ 12} & \cellcolor{lightgray}{\textbf{1214 $\pm$ 7}} &  \cellcolor{lightgray}{\textbf{1045 $\pm$ 20}} & \cellcolor{lightgray}{\textbf{754 $\pm$ 102}}\\
\midrule
\multirow{3}{*}{\textbf{Model}} & \multicolumn{4}{c|}{\textbf{Freeway} } & \multicolumn{4}{c}{\textbf{RoadRunner}} \\
\cline{2-9}
& \multirow{2}{*}{\begin{tabular}[c]{@{}c@{}}\textbf{Natural}\\\textbf{Reward}\end{tabular}} & \textbf{PGD} & \textbf{MinBest} & \textbf{PA-AD} & \multirow{2}{*}{\begin{tabular}[c]{@{}c@{}}\textbf{Natural}\\\textbf{Reward}\end{tabular}} & \textbf{PGD} & \textbf{MinBest} & \textbf{PA-AD} \\
\cline{3-5}\cline{7-9}
& & \multicolumn{3}{c|}{$\epsilon$= 3/255} & & \multicolumn{3}{c}{$\epsilon$= 3/255} \\
\hline
DQN & \textbf{34.0 $\pm$ 0.1} & 0.0 $\pm$ 0.0 & 5.5 $\pm$ 1.8 &  4.7 $\pm$ 2.9 & \textbf{45527 $\pm$ 4894} &  0 $\pm$ 0 &  2985 $\pm$ 1440 & 203 $\pm$ 65 \\
SA-DQN & 30.0 $\pm$ 0.0 & 30.0 $\pm$ 0.0 & 18.3 $\pm$ 3.0 & 9.5 $\pm$ 3.8 & 44638 $\pm$ 2367 &  20678 $\pm$ 1563 &  4214 $\pm$ 2587 &  5516 $\pm$ 4684 \\
RADIAL-DQN & 33.1 $\pm$ 0.2 & \textbf{33.2 $\pm$ 0.2} & 16.4 $\pm$ 2.3 & 10.8 $\pm$ 3.6 & 44675 $\pm$ 5854  & 38576 $\pm$ 1960 & 8476 $\pm$ 3964 & 1290 $\pm$ 4015 \\
\cellcolor{lightgray}{\textbf{\ourdqn (Ours)}} & 
\cellcolor{lightgray}{31.2 $\pm$ 0.4} &
\cellcolor{lightgray}{31.4 $\pm$ 0.3} &
\cellcolor{lightgray}{\textbf{19.8 $\pm$ 3.8}} &
\cellcolor{lightgray}{\textbf{12.3 $\pm$ 3.2}} & \cellcolor{lightgray}{44156 $\pm$ 2279} & \cellcolor{lightgray}{\textbf{38720 $\pm$ 1765}} &  \cellcolor{lightgray}{\textbf{10545 $\pm$ 2984}} & \cellcolor{lightgray}{\textbf{8239 $\pm$ 2766}}\\
\bottomrule
\end{tabular}}
\vspace{0.5em}
\caption{\textbf{Robustness and High Natural Performance of WocaR-DQN.}
Average episode rewards $\pm$ standard deviation over 50 episodes on three baselines and \ourdqn on four Atari environments. Best results (natural reward of under attacks for each column) on each environment boldfaced. \ourdqn outperforms all the baselines in most cases or gains similar performance in the other metrics. We highlight the most robust agent as \colorbox{lightgray}{gray}. Each result is obtained with 10 random seeds.}
\vspace{-1em}
\label{tab:atari}
\end{table}

\begin{figure}[t]
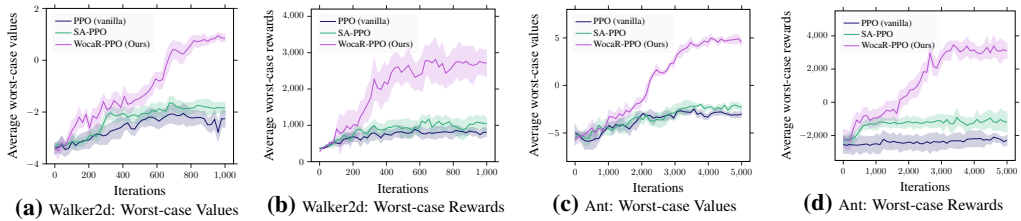

    \centering
    \begin{subfigure}[!t]{0.24\textwidth}
        \centering
        \resizebox{0.95\textwidth}{!}{\input{\fighome/curve/walker-value.tex}}
        \vspace{-0.7em}
        \caption{\fontsize{6.1pt}{\baselineskip}\selectfont Walker2d: Worst-case Values}
        \label{sfig:curves_walker_value_less}
    \end{subfigure}
    \begin{subfigure}[!t]{0.24\textwidth}
        \centering
        \resizebox{0.97\textwidth}{!}{\input{\fighome/curve/walker-worst.tex}}
        \vspace{-0.7em}
        \caption{\fontsize{6.1pt}{\baselineskip}\selectfont Walker2d: Worst-case Rewards}
        \label{sfig:curves_walker_reward_less}
    \end{subfigure}
    \begin{subfigure}[!t]{0.24\textwidth}
        \centering
        \resizebox{0.92\textwidth}{!}{\input{\fighome/curve/ant-value.tex}}
        \vspace{-0.7em}
        \caption{\fontsize{6.1pt}{\baselineskip}\selectfont Ant: Worst-case Values}
        \label{sfig:curves_ant_value_less}
    \end{subfigure}
    \begin{subfigure}[!t]{0.24\textwidth}
        \centering
        \resizebox{0.97\textwidth}{!}{\input{\fighome/curve/ant-worst.tex}}
        \vspace{-0.7em}
        \caption{\fontsize{6.1pt}{\baselineskip}\selectfont Ant: Worst-case Rewards}
        \label{sfig:curves_ant_reward_less}
    \end{subfigure}
    \vspace{-0.5em}
    \caption{\small{\textbf{(a)\&(b)} Comparison between estimated worst-attack action values {\scriptsize{$\worstcritic$}} and Actual worst-case rewards under the strongest attacksduring training on Walker2d; \textbf{(c)\&(d)} The comparison between worst-case values and rewards to verify worst-attack value estimation on Ant.
    }}
    \label{fig:curves_less}
\vspace{-1em}
\end{figure}
\begin{figure}[t]
    \centering
    \begin{subfigure}[!t]{0.45\textwidth}
        \centering
        \resizebox{\textwidth}{!}{
\begin{tikzpicture}[scale=0.6]

\definecolor{color0}{rgb}{0.168966293897257,0.475610408705207,0.783522853775186}
\definecolor{color1}{rgb}{0.691639317416794,0.806706471648465,0.92224956246861}
\definecolor{color2}{rgb}{0.95715994546835,0.680184738372337,0.77629698219994}
\definecolor{color3}{rgb}{0.885434799735519,0.144732659736887,0.401761241571282}

\begin{axis}[
height=3cm,
width=\textwidth,
scale only axis,
axis line style={white!15!black},
legend cell align={left},
legend columns=2,
legend style={
  fill opacity=0,
  draw opacity=1,
  text opacity=1,
  at={(0.53,1.01)},
  anchor=north,
  draw=none,
  font=\scriptsize
},
xtick pos = lower,
tick align=outside,
x grid style={white!80!black},
xmin=-0.636, xmax=9.726,
xtick style={color=white!15!black},
xtick={0.595,2.595,4.595,6.595,8.595},
xticklabels={\scriptsize{No Attack},\scriptsize{MaxDiff},\scriptsize{Robust Sarsa},\scriptsize{SA-RL},\scriptsize{PA-AD}},
y grid style={white!80!black},
ylabel={\scriptsize{Average episode rewards}},
ymin=2000, ymax=7800,
ytick style={color=white!15!black},
y grid style={white},
tick label style={font=\tiny},
ymajorgrids,
]
\draw[draw=white!93.3333333333333!black,fill=color0,opacity=0.7,very thin,postaction={pattern=crosshatch, pattern color=white!93.3333333333333!black, fill opacity=0.7}] (axis cs:-0.165,0) rectangle (axis cs:0.165,3632);
\addlegendimage{ybar,ybar legend,draw=white!93.3333333333333!black,fill=color0,opacity=0.7,very thin,postaction={pattern=crosshatch, pattern color=white!93.3333333333333!black, fill opacity=0.7}}
\addlegendentry{SA-PPO}

\draw[draw=white!93.3333333333333!black,fill=color0,opacity=0.7,very thin,postaction={pattern=crosshatch, pattern color=white!93.3333333333333!black, fill opacity=0.7}] (axis cs:1.835,0) rectangle (axis cs:2.165,3624);
\draw[draw=white!93.3333333333333!black,fill=color0,opacity=0.7,very thin,postaction={pattern=crosshatch, pattern color=white!93.3333333333333!black, fill opacity=0.7}] (axis cs:3.835,0) rectangle (axis cs:4.165,3283);
\draw[draw=white!93.3333333333333!black,fill=color0,opacity=0.7,very thin,postaction={pattern=crosshatch, pattern color=white!93.3333333333333!black, fill opacity=0.7}] (axis cs:5.835,0) rectangle (axis cs:6.165,3028);
\draw[draw=white!93.3333333333333!black,fill=color0,opacity=0.7,very thin,postaction={pattern=crosshatch, pattern color=white!93.3333333333333!black, fill opacity=0.7}] (axis cs:7.835,0) rectangle (axis cs:8.165,2512);
\draw[draw=white!93.3333333333333!black,fill=color1,opacity=0.8,very thin,postaction={pattern=north west lines, pattern color=white!93.3333333333333!black, fill opacity=0.8}] (axis cs:0.165,0) rectangle (axis cs:0.495,3854);
\addlegendimage{ybar,ybar legend,draw=white!93.3333333333333!black,fill=color1,opacity=0.8,very thin,postaction={pattern=north west lines, pattern color=white!93.3333333333333!black, fill opacity=0.8}}
\addlegendentry{SA-PPO + $w(s)$}

\draw[draw=white!93.3333333333333!black,fill=color1,opacity=0.8,very thin,postaction={pattern=north west lines, pattern color=white!93.3333333333333!black, fill opacity=0.8}] (axis cs:2.165,0) rectangle (axis cs:2.495,3802);
\draw[draw=white!93.3333333333333!black,fill=color1,opacity=0.8,very thin,postaction={pattern=north west lines, pattern color=white!93.3333333333333!black, fill opacity=0.8}] (axis cs:4.165,0) rectangle (axis cs:4.495,3416);
\draw[draw=white!93.3333333333333!black,fill=color1,opacity=0.8,very thin,postaction={pattern=north west lines, pattern color=white!93.3333333333333!black, fill opacity=0.8}] (axis cs:6.165,0) rectangle (axis cs:6.495,3145);
\draw[draw=white!93.3333333333333!black,fill=color1,opacity=0.8,very thin,postaction={pattern=north west lines, pattern color=white!93.3333333333333!black, fill opacity=0.8}] (axis cs:8.165,0) rectangle (axis cs:8.495,2782);
\draw[draw=white!93.3333333333333!black,fill=color2,opacity=0.8,very thin,postaction={pattern=north east lines, pattern color=white!93.3333333333333!black, fill opacity=0.8}] (axis cs:0.595,0) rectangle (axis cs:0.925,5426);
\addlegendimage{ybar,ybar legend,draw=white!93.3333333333333!black,fill=color2,opacity=0.8,very thin,postaction={pattern=north east lines, pattern color=white!93.3333333333333!black, fill opacity=0.8}}
\addlegendentry{\ourppo - $w(s)$}

\draw[draw=white!93.3333333333333!black,fill=color2,opacity=0.8,very thin,postaction={pattern=north east lines, pattern color=white!93.3333333333333!black, fill opacity=0.8}] (axis cs:2.595,0) rectangle (axis cs:2.925,5409);
\draw[draw=white!93.3333333333333!black,fill=color2,opacity=0.8,very thin,postaction={pattern=north east lines, pattern color=white!93.3333333333333!black, fill opacity=0.8}] (axis cs:4.595,0) rectangle (axis cs:4.925,5193);
\draw[draw=white!93.3333333333333!black,fill=color2,opacity=0.8,very thin,postaction={pattern=north east lines, pattern color=white!93.3333333333333!black, fill opacity=0.8}] (axis cs:6.595,0) rectangle (axis cs:6.925,5178);
\draw[draw=white!93.3333333333333!black,fill=color2,opacity=0.8,very thin,postaction={pattern=north east lines, pattern color=white!93.3333333333333!black, fill opacity=0.8}] (axis cs:8.595,0) rectangle (axis cs:8.925,4098);
\draw[draw=white!93.3333333333333!black,fill=color3,opacity=0.8,very thin] (axis cs:0.925,0) rectangle (axis cs:1.255,6032);
\addlegendimage{ybar,ybar legend,draw=white!93.3333333333333!black,fill=color3,opacity=0.8,very thin}
\addlegendentry{\ourppo (Ours)}

\draw[draw=white!93.3333333333333!black,fill=color3,opacity=0.8,very thin] (axis cs:2.925,0) rectangle (axis cs:3.255,5850);
\draw[draw=white!93.3333333333333!black,fill=color3,opacity=0.8,very thin] (axis cs:4.925,0) rectangle (axis cs:5.255,5319);
\draw[draw=white!93.3333333333333!black,fill=color3,opacity=0.8,very thin] (axis cs:6.925,0) rectangle (axis cs:7.255,5365);
\draw[draw=white!93.3333333333333!black,fill=color3,opacity=0.8,very thin] (axis cs:8.925,0) rectangle (axis cs:9.255,4269);
\path [draw=white!41.1764705882353!black, line width=0.48pt]
(axis cs:0,3612)
--(axis cs:0,3652);

\path [draw=white!41.1764705882353!black, line width=0.48pt]
(axis cs:2,3601)
--(axis cs:2,3647);

\path [draw=white!41.1764705882353!black, line width=0.48pt]
(axis cs:4,3263)
--(axis cs:4,3303);

\path [draw=white!41.1764705882353!black, line width=0.48pt]
(axis cs:6,3005)
--(axis cs:6,3051);

\path [draw=white!41.1764705882353!black, line width=0.48pt]
(axis cs:8,2496)
--(axis cs:8,2528);

\path [draw=white!41.1764705882353!black, line width=0.48pt]
(axis cs:0.33,3776)
--(axis cs:0.33,3932);

\path [draw=white!41.1764705882353!black, line width=0.48pt]
(axis cs:2.33,3744)
--(axis cs:2.33,3860);

\path [draw=white!41.1764705882353!black, line width=0.48pt]
(axis cs:4.33,3382)
--(axis cs:4.33,3450);

\path [draw=white!41.1764705882353!black, line width=0.48pt]
(axis cs:6.33,3108)
--(axis cs:6.33,3182);

\path [draw=white!41.1764705882353!black, line width=0.48pt]
(axis cs:8.33,2688)
--(axis cs:8.33,2876);

\path [draw=white!41.1764705882353!black, line width=0.48pt]
(axis cs:0.76,5378)
--(axis cs:0.76,5474);

\path [draw=white!41.1764705882353!black, line width=0.48pt]
(axis cs:2.76,5193)
--(axis cs:2.76,5625);

\path [draw=white!41.1764705882353!black, line width=0.48pt]
(axis cs:4.76,5016)
--(axis cs:4.76,5370);

\path [draw=white!41.1764705882353!black, line width=0.48pt]
(axis cs:6.76,5110)
--(axis cs:6.76,5246);

\path [draw=white!41.1764705882353!black, line width=0.48pt]
(axis cs:8.76,3944)
--(axis cs:8.76,4252);

\path [draw=white!41.1764705882353!black, line width=0.48pt]
(axis cs:1.09,5964)
--(axis cs:1.09,6100);

\path [draw=white!41.1764705882353!black, line width=0.48pt]
(axis cs:3.09,5622)
--(axis cs:3.09,6078);

\path [draw=white!41.1764705882353!black, line width=0.48pt]
(axis cs:5.09,5099)
--(axis cs:5.09,5539);

\path [draw=white!41.1764705882353!black, line width=0.48pt]
(axis cs:7.09,5311)
--(axis cs:7.09,5419);

\path [draw=white!41.1764705882353!black, line width=0.48pt]
(axis cs:9.09,4097)
--(axis cs:9.09,4441);

\addplot [line width=0.48pt, white!41.1764705882353!black, opacity=1, mark=-, mark size=1.3, mark options={solid}, only marks]
table {%
0 3612
2 3601
4 3263
6 3005
8 2496
};
\addplot [line width=0.48pt, white!41.1764705882353!black, opacity=1, mark=-, mark size=1.3, mark options={solid}, only marks]
table {%
0 3652
2 3647
4 3303
6 3051
8 2528
};

\addplot [line width=0.48pt, white!41.1764705882353!black, opacity=1, mark=-, mark size=1.3, mark options={solid}, only marks]
table {%
0.33 3776
2.33 3744
4.33 3382
6.33 3108
8.33 2688
};

\addplot [line width=0.48pt, white!41.1764705882353!black, opacity=1, mark=-, mark size=1.3, mark options={solid}, only marks]
table {%
0.33 3932
2.33 3860
4.33 3450
6.33 3182
8.33 2876
};

\addplot [line width=0.48pt, white!41.1764705882353!black, opacity=1, mark=-, mark size=1.3, mark options={solid}, only marks]
table {%
0.76 5378
2.76 5193
4.76 5016
6.76 5110
8.76 3944
};

\addplot [line width=0.48pt, white!41.1764705882353!black, opacity=1, mark=-, mark size=1.3, mark options={solid}, only marks]
table {%
0.76 5474
2.76 5625
4.76 5370
6.76 5246
8.76 4252
};

\addplot [line width=0.48pt, white!41.1764705882353!black, opacity=1, mark=-, mark size=1.3, mark options={solid}, only marks]
table {%
1.09 5964
3.09 5622
5.09 5099
7.09 5311
9.09 4097
};

\addplot [line width=0.48pt, white!41.1764705882353!black, opacity=1, mark=-, mark size=1.3, mark options={solid}, only marks]
table {%
1.09 6100
3.09 6078
5.09 5539
7.09 5419
9.09 4441
};

\end{axis}

\end{tikzpicture}}
        \vspace{-1.8em}
        \caption{\scriptsize{Halfcheetah: Ablation studies for $w(s)$}}
        \label{sfig:abl_h1}
    \end{subfigure}
    \begin{subfigure}[!t]{0.45\textwidth}
        \centering
        \resizebox{\textwidth}{!}{
\begin{tikzpicture}

\definecolor{color0}{rgb}{0.552941176470588,0.707450980392157,0.896078431372549}
\definecolor{color1}{rgb}{0.703057937269964,0.856470725837709,0.708208636049633}
\definecolor{color2}{rgb}{0.744102947491485,0.669256898711933,0.879184450886271}
\definecolor{color3}{rgb}{0.905434799735519,0.144732659736887,0.401761241571282}

\begin{axis}[
height=3cm,
width=\textwidth,
scale only axis,
axis line style={white!15!black},
legend cell align={left},
legend columns=2,
legend style={
  fill opacity=0,
  draw opacity=1,
  text opacity=1,
  at={(0.53,1.01)},
  anchor=north,
  draw=none,
  font=\scriptsize
},
xtick pos = lower,
tick align=outside,
x grid style={white!80!black},
xmin=-0.636, xmax=9.726,
xtick style={color=white!15!black},
xtick={0.595,2.595,4.595,6.595,8.595},
xticklabels={\scriptsize{No Attack},\scriptsize{MaxDiff},\scriptsize{Robust Sarsa},\scriptsize{SA-RL},\scriptsize{PA-AD}},
y grid style={white!80!black},
ylabel={\scriptsize{Average episode rewards}},
ymin=500, ymax=5500,
ytick style={color=white!15!black},
y grid style={white},
tick label style={font=\tiny},
ymajorgrids,
]
\draw[draw=white!93.3333333333333!black,fill=color0,opacity=0.8,very thin,postaction={pattern=crosshatch, pattern color=white!93.3333333333333!black, fill opacity=0.8}] (axis cs:-0.165,0) rectangle (axis cs:0.165,3487);
\addlegendimage{ybar,ybar legend,draw=white!93.3333333333333!black,fill=color0,opacity=0.8,very thin,postaction={pattern=crosshatch, pattern color=white!93.3333333333333!black, fill opacity=0.8}}
\addlegendentry{ATLA - $\lossreg$}

\draw[draw=white!93.3333333333333!black,fill=color0,opacity=0.8,very thin,postaction={pattern=crosshatch, pattern color=white!93.3333333333333!black, fill opacity=0.8}] (axis cs:1.835,0) rectangle (axis cs:2.165,3081);
\draw[draw=white!93.3333333333333!black,fill=color0,opacity=0.8,very thin,postaction={pattern=crosshatch, pattern color=white!93.3333333333333!black, fill opacity=0.8}] (axis cs:3.835,0) rectangle (axis cs:4.165,1567);
\draw[draw=white!93.3333333333333!black,fill=color0,opacity=0.8,very thin,postaction={pattern=crosshatch, pattern color=white!93.3333333333333!black, fill opacity=0.8}] (axis cs:5.835,0) rectangle (axis cs:6.165,1224);
\draw[draw=white!93.3333333333333!black,fill=color0,opacity=0.8,very thin,postaction={pattern=crosshatch, pattern color=white!93.3333333333333!black, fill opacity=0.8}] (axis cs:7.835,0) rectangle (axis cs:8.165,987);
\draw[draw=white!93.3333333333333!black,fill=color1,opacity=0.8,very thin,postaction={pattern=north west lines, pattern color=white!93.3333333333333!black, fill opacity=0.8}] (axis cs:0.165,0) rectangle (axis cs:0.495,3512);
\addlegendimage{ybar,ybar legend,draw=white!93.3333333333333!black,fill=color1,opacity=0.8,very thin,postaction={pattern=north west lines, pattern color=white!93.3333333333333!black, fill opacity=0.8}}
\addlegendentry{PA-ATLA-PPO - $\lossreg$}

\draw[draw=white!93.3333333333333!black,fill=color1,opacity=0.8,very thin,postaction={pattern=north west lines, pattern color=white!93.3333333333333!black, fill opacity=0.8}] (axis cs:2.165,0) rectangle (axis cs:2.495,3024);
\draw[draw=white!93.3333333333333!black,fill=color1,opacity=0.8,very thin,postaction={pattern=north west lines, pattern color=white!93.3333333333333!black, fill opacity=0.8}] (axis cs:4.165,0) rectangle (axis cs:4.495,2032);
\draw[draw=white!93.3333333333333!black,fill=color1,opacity=0.8,very thin,postaction={pattern=north west lines, pattern color=white!93.3333333333333!black, fill opacity=0.8}] (axis cs:6.165,0) rectangle (axis cs:6.495,1142);
\draw[draw=white!93.3333333333333!black,fill=color1,opacity=0.8,very thin,postaction={pattern=north west lines, pattern color=white!93.3333333333333!black, fill opacity=0.8}] (axis cs:8.165,0) rectangle (axis cs:8.495,1679);
\draw[draw=white!93.3333333333333!black,fill=color2,opacity=0.8,very thin,postaction={pattern=north east lines, pattern color=white!93.3333333333333!black, fill opacity=0.8}] (axis cs:0.595,0) rectangle (axis cs:0.925,3568);
\addlegendimage{ybar,ybar legend,draw=white!93.3333333333333!black,fill=color2,opacity=0.8,very thin,postaction={pattern=north east lines, pattern color=white!93.3333333333333!black, fill opacity=0.8}}
\addlegendentry{\ourppo - $\lossreg$}

\draw[draw=white!93.3333333333333!black,fill=color2,opacity=0.8,very thin,postaction={pattern=north east lines, pattern color=white!93.3333333333333!black, fill opacity=0.8}] (axis cs:2.595,0) rectangle (axis cs:2.925,3114);
\draw[draw=white!93.3333333333333!black,fill=color2,opacity=0.8,very thin,postaction={pattern=north east lines, pattern color=white!93.3333333333333!black, fill opacity=0.8}] (axis cs:4.595,0) rectangle (axis cs:4.925,2729);
\draw[draw=white!93.3333333333333!black,fill=color2,opacity=0.8,very thin,postaction={pattern=north east lines, pattern color=white!93.3333333333333!black, fill opacity=0.8}] (axis cs:6.595,0) rectangle (axis cs:6.925,1615);
\draw[draw=white!93.3333333333333!black,fill=color2,opacity=0.8,very thin,postaction={pattern=north east lines, pattern color=white!93.3333333333333!black, fill opacity=0.8}] (axis cs:8.595,0) rectangle (axis cs:8.925,2035);
\draw[draw=white!93.3333333333333!black,fill=color3,opacity=0.8,very thin] (axis cs:0.925,0) rectangle (axis cs:1.255,3616);
\addlegendimage{ybar,ybar legend,draw=white!93.3333333333333!black,fill=color3,opacity=0.8,very thin}
\addlegendentry{\ourppo (Ours)}

\draw[draw=white!93.3333333333333!black,fill=color3,opacity=0.8,very thin] (axis cs:2.925,0) rectangle (axis cs:3.255,3541);
\draw[draw=white!93.3333333333333!black,fill=color3,opacity=0.8,very thin] (axis cs:4.925,0) rectangle (axis cs:5.255,3277);
\draw[draw=white!93.3333333333333!black,fill=color3,opacity=0.8,very thin] (axis cs:6.925,0) rectangle (axis cs:7.255,2390);
\draw[draw=white!93.3333333333333!black,fill=color3,opacity=0.8,very thin] (axis cs:8.925,0) rectangle (axis cs:9.255,2579);
\path [draw=white!41.1764705882353!black, line width=0.48pt]
(axis cs:0,3035)
--(axis cs:0,3939);

\path [draw=white!41.1764705882353!black, line width=0.48pt]
(axis cs:2,2827)
--(axis cs:2,3335);

\path [draw=white!41.1764705882353!black, line width=0.48pt]
(axis cs:4,1220)
--(axis cs:4,1914);

\path [draw=white!41.1764705882353!black, line width=0.48pt]
(axis cs:6,1033)
--(axis cs:6,1415);

\path [draw=white!41.1764705882353!black, line width=0.48pt]
(axis cs:8,863)
--(axis cs:8,1111);

\path [draw=white!41.1764705882353!black, line width=0.48pt]
(axis cs:0.33,3225)
--(axis cs:0.33,3799);

\path [draw=white!41.1764705882353!black, line width=0.48pt]
(axis cs:2.33,2865)
--(axis cs:2.33,3183);

\path [draw=white!41.1764705882353!black, line width=0.48pt]
(axis cs:4.33,1829)
--(axis cs:4.33,2235);

\path [draw=white!41.1764705882353!black, line width=0.48pt]
(axis cs:6.33,944)
--(axis cs:6.33,1340);

\path [draw=white!41.1764705882353!black, line width=0.48pt]
(axis cs:8.33,1492)
--(axis cs:8.33,1866);

\path [draw=white!41.1764705882353!black, line width=0.48pt]
(axis cs:0.76,3372)
--(axis cs:0.76,3764);

\path [draw=white!41.1764705882353!black, line width=0.48pt]
(axis cs:2.76,2885)
--(axis cs:2.76,3343);

\path [draw=white!41.1764705882353!black, line width=0.48pt]
(axis cs:4.76,2365)
--(axis cs:4.76,3093);

\path [draw=white!41.1764705882353!black, line width=0.48pt]
(axis cs:6.76,1326)
--(axis cs:6.76,1904);

\path [draw=white!41.1764705882353!black, line width=0.48pt]
(axis cs:8.76,1844)
--(axis cs:8.76,2226);

\path [draw=white!41.1764705882353!black, line width=0.48pt]
(axis cs:1.09,3517)
--(axis cs:1.09,3715);

\path [draw=white!41.1764705882353!black, line width=0.48pt]
(axis cs:3.09,3334)
--(axis cs:3.09,3748);

\path [draw=white!41.1764705882353!black, line width=0.48pt]
(axis cs:5.09,3118)
--(axis cs:5.09,3436);

\path [draw=white!41.1764705882353!black, line width=0.48pt]
(axis cs:7.09,2245)
--(axis cs:7.09,2535);

\path [draw=white!41.1764705882353!black, line width=0.48pt]
(axis cs:9.09,2350)
--(axis cs:9.09,2808);

\addplot [line width=0.48pt, white!41.1764705882353!black, opacity=1, mark=-, mark size=1.3, mark options={solid}, only marks]
table {%
0 3035
2 2827
4 1220
6 1033
8 863
};

\addplot [line width=0.48pt, white!41.1764705882353!black, opacity=1, mark=-, mark size=1.3, mark options={solid}, only marks]
table {%
0 3939
2 3335
4 1914
6 1415
8 1111
};

\addplot [line width=0.48pt, white!41.1764705882353!black, opacity=1, mark=-, mark size=1.3, mark options={solid}, only marks]
table {%
0.33 3225
2.33 2865
4.33 1829
6.33 944
8.33 1492
};

\addplot [line width=0.48pt, white!41.1764705882353!black, opacity=1, mark=-, mark size=1.3, mark options={solid}, only marks]
table {%
0.33 3799
2.33 3183
4.33 2235
6.33 1340
8.33 1866
};

\addplot [line width=0.48pt, white!41.1764705882353!black, opacity=1, mark=-, mark size=1.3, mark options={solid}, only marks]
table {%
0.76 3372
2.76 2885
4.76 2365
6.76 1326
8.76 1844
};

\addplot [line width=0.48pt, white!41.1764705882353!black, opacity=1, mark=-, mark size=1.3, mark options={solid}, only marks]
table {%
0.76 3764
2.76 3343
4.76 3093
6.76 1904
8.76 2226
};

\addplot [line width=0.48pt, white!41.1764705882353!black, opacity=1, mark=-, mark size=1.3, mark options={solid}, only marks]
table {%
1.09 3517
3.09 3334
5.09 3118
7.09 2245
9.09 2350
};

\addplot [line width=0.48pt, white!41.1764705882353!black, opacity=1, mark=-, mark size=1.3, mark options={solid}, only marks]
table {%
1.09 3715
3.09 3748
5.09 3436
7.09 2535
9.09 2808
};

\end{axis}
\end{tikzpicture}}
        \vspace{-1.8em}
        \caption{\scriptsize{Hopper: Ablation studies for $\lossreg$}}
        \label{sfig:reg_h}
    \end{subfigure}
    \vspace{-0.5em}
    \caption{\small{\textbf{(a)} Ablation evaluations for state importance weight $w(s)$ under no attack and four types of attacks on Halfcheetah; \textbf{(b)} Ablation studies for state regularization $\lossreg$ under different evaluation metrics on Hopper. Ablated results on other environments are in Appendix~\ref{app:exp:reg}.
    }}
    \label{fig:ablation_less}
\vspace{-1em}
\end{figure}
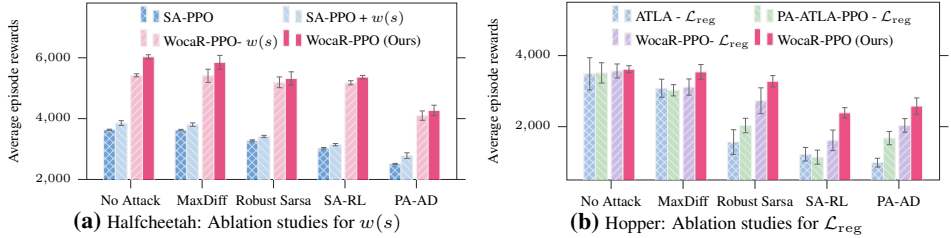

\subsection{Verifying Effectiveness of \ours}
\label{sec:result_effect}

Now we dive deeper into the algorithmic design and verify the effectiveness of \ours by ablation studies on \ourppo.

\textbf{(1) Worst-attack value estimation.}
We show the learned worst-attack value estimation, {\small{$\worstcritic$}}, during the training process in Figure~\ref{sfig:curves_walker_value_less} and \ref{sfig:curves_ant_value_less}, in comparison with the actual reward under the strongest attack (PA-AD~\citep{sun2021strongest}) in Figure~\ref{sfig:curves_walker_reward_less} and \ref{sfig:curves_ant_reward_less}. The pink curves in both plots suggest that \emph{our worst-attack value estimation matches the trend of actual worst-case reward under attacks,} although the network estimated value and the real reward have different scales due to the commonly-used reward normalization for learning stability. Therefore, the effectiveness of our proposed worst-attack value estimation ($\lossest$) is verified. 

\textbf{(2) Worst-case-aware policy optimization.}
Compared to vanilla PPO and SA-PPO, we can see that \emph{\ourppo improves the worst-attack value and the worst-case reward during training}, suggesting the effectiveness of our worst-attack value improvement ($\lossworst$).
The comparison of natural rewards, as well as curves in other environments, are provided in Appendix~\ref{app:exp:curve}. 
Moreover, the adjustable weight $\weightworst$ in Equation~\eqref{loss:policy} controls the trade-off between natural value and worst-attack value in policy optimization. When $\weightworst$ is high, the policy pays more attention to its worst-attack value. Appendix~\ref{app:exp:prefer} verifies that \emph{\ours, with different values of weight $\weightworst$, produces different robustness and natural performance while consistently dominating other robust agents.}

\textbf{(3) Value-enhanced state regularization.}
We conduct ablation experiments to analyze the effect of two techniques: our proposed state importance weight $w(s)$ and the state regularization loss $\lossreg$ \citep{zhang2020robust}. In Figure~\ref{sfig:abl_h1}, we compare the performance of the original \ourppo to a variant of \ourppo without the state importance weight $w(s)$ on Halfcheetah, which visually indicates that $w(s)$ can help agents boost the robustness. Since \sappo~\citep{zhang2020robust} also uses a state regularization technique, the improvement of \sappo added with $w(s)$ also show the universal effectiveness of our state importance. Without $w(s)$, our algorithm also achieves similar or better performance than baselines, but including this inexpensive technique $w(s)$ gives \ours a greater advantage, especially under learned strong attacks SA-RL and PA-AD.
Figure~\ref{sfig:reg_h} presents the performance of ATLA methods and our algorithm without $\lossreg$ on Hopper, which verifies that \ourppo also yields the superior performance when removing the regularization technique. And the comparison between \ourppo and \ourppo without $\lossreg$ demonstrates that the weighted state regularization is beneficial to enhancing the robustness in our algorithm. Detailed ablation studies for $w(s)$ and $\lossreg$ on four MuJoCo environments are shown in Appendix~\ref{app:exp:reg}.

\section{Conclusion and Discussion}
\label{sec:conclusion}

This paper proposes a robust RL training framework, \ours, that evaluates and improves the long-term robustness of a policy via \estname, \worstname, and \regname. 
Different from recent state-of-the-art adversarial training methods~\citep{sun2021strongest,zhang2021robust} which train an extra adversary to improve the robustness of an agent, we directly estimate and improve the lower bound of the agent's cumulative reward. As a result, \ours not only achieves better robustness than state-of-the-art robust RL approaches, but also halves the total sample complexity and computation complexity, in a wide range of Atari and MuJoCo tasks.


There are several aspects to improve or extend the current approach. 
First, the proposed \bellmanname in theory gives the exact worst-case value of a policy under $\ell_p$ bounded attacks. But in practice, it is hard to compute the set $\advaction$ directly, so we use convex relaxation to obtain a superset of it, $\aadvaction$. As a result, the fixed point of \bellmanname with $\advaction$ being replaced by $\aadvaction$ is a lower bound of the worst-case value. Then, our algorithm increases the worst-case value by improving its lower bound, as visualized and explained in Figure~\ref{fig:training} in Appendix~\ref{app:understand}. Therefore, one potential way of further improving the robustness is using a tighter relaxation.
In addition, this paper only considers the $\ell_p$ threat model as is common in most related works. But in real-world applications, other attack models could exist (e.g. patch attacks~\citep{brown2017adversarial}), and improving the robustness of RL agents in these scenarios is another important research direction.

\section*{Acknowledgments}
This work is supported by DOD-ONR-Office of Naval Research,
DOD-DARPA-Defense Advanced Research Projects Agency Guaranteeing AI Robustness against Deception (GARD), and Adobe, Capital One and JP Morgan faculty fellowships.

\bibliographystyle{plain}


\newpage
\appendix
\onecolumn
\begin{center}{\centering \bf \Large
    Supplementary Material:
}\end{center}

\begin{center}{\centering \Large
     \mytitle
}\end{center}

\section{Theoretical Analysis}
\label{app:theory}

Similar to the \worstqname, we can define the \worstvname as below:
\begin{definition}[Worst-attack Value]
\label{def:worstv}
For a given policy $\pi$, define the \worstvname of $\pi$ as
\begin{equation}
    \worstvpi(s) := \mathbb{E}_{P} [\sum_{t=0}^{\infty} \gamma^{t} R\left(s_t, \pi(h^*(s_t))\right) \mid s_0=s],
\end{equation}
where $h^*$ is the optimal attacker which minimizes the victim's cumulative reward under the $\epsilon$ constraint. 
\end{definition}

\begin{proof}[Proof of Theorem~\ref{thm:main}]


First, we show that $\underline{\mathcal{T}}^\pi$ is a contraction. 

For any two Q functions $Q_1: \states\times\actions \to \mathbb{R}$ and $Q_2: \states\times\actions \to \mathbb{R}$, we have
\begin{small}
\begin{equation*}
\begin{split}
&\left\|\underline{\mathcal{T}}^\pi Q_1 - \underline{\mathcal{T}}^\pi Q_2\right\|_{\infty} \\
&=\max _{s,a} \left|\sum_{s^{\prime} \in \mathcal{S}} P\left(s^{\prime} \mid s, a\right)\left[R(s, a)+\gamma \min _{a^{\prime} \in \advaction(s^{\prime}, \pi)} Q_1\left(s^{\prime}, a^{\prime}\right) - R(s, a)+\gamma \min _{a^{\prime} \in \advaction(s^{\prime}, \pi)} Q_2\left(s^{\prime}, a^{\prime}\right)\right] \right|\\
&=\gamma \max _{s,a}\left|\sum_{s^{\prime} \in \mathcal{S}} P\left(s^{\prime} \mid s, a\right) \left[\min _{a^{\prime} \in \advaction(s^{\prime}, \pi)} Q_1\left(s^{\prime}, a^{\prime}\right) - \min _{a^{\prime} \in \advaction(s^{\prime}, \pi)} Q_2\left(s^{\prime}, a^{\prime}\right)\right]\right|\\
&\leq \gamma \max _{s,a}\sum_{s^{\prime} \in \mathcal{S}} P\left(s^{\prime} \mid s, a\right) \left|\min _{a^{\prime} \in \advaction(s^{\prime}, \pi)} Q_1\left(s^{\prime}, a^{\prime}\right) - \min _{a^{\prime} \in \advaction(s^{\prime}, \pi)} Q_2\left(s^{\prime}, a^{\prime}\right)\right|\\
&\leq \gamma \max _{s,a}\sum_{s^{\prime} \in \mathcal{S}} P\left(s^{\prime} \mid s, a\right) \max _{a^{\prime} \in \advaction(s^{\prime}, \pi)}\left|Q_1\left(s^{\prime}, a^{\prime}\right) - Q_2\left(s^{\prime}, a^{\prime}\right)\right|\\
&=\gamma \max _{s,a}\sum_{s^{\prime} \in \mathcal{S}} P\left(s^{\prime} \mid s, a\right)\left\|Q_1 - Q_2\right\|_{\infty}\\
&=\gamma\left\|Q_1 - Q_2\right\|_{\infty}\\
\end{split}
\end{equation*}
\end{small}

The second inequality comes from the fact that,
\begin{equation*}
\left|\min _{x_1}f(x_1) - \min _{x_2}g(x_2)\right| \leq \max _{x}\left|f(x)-g(x)\right|
\end{equation*}
The operator $\underline{\mathcal{T}}^\pi$ satisfies,
\begin{equation*}
\left\|\underline{\mathcal{T}}^\pi Q_1 - \underline{\mathcal{T}}^\pi Q_2\right\|_{\infty} \leq \gamma\left\|Q_1 - Q_2\right\|_{\infty}
\end{equation*}
so it is a contraction in the sup-norm.

Recall the definition of \worstqname:
\begin{equation}
    \worstqpi(s,a) := \mathbb{E}_{P} [\sum_{t=0}^{\infty} \gamma^{t} R\left(s_t, \pi(h^*(s_t))\right) \mid s_0=s, a_0=a ],
\end{equation}
where $h^*$ is the optimal attacker which minimizes the victim's cumulative reward under the $\epsilon$ constraint. 
That is, the optimal attacker $h^*$ lets the agent select the worst possible action among all achievable actions in $\advaction$.
Hence, we have $\underline{Q}^{\pi}(s, a)= \underline{\mathcal{T}}^\pi \underline{Q}^{\pi}(s, a)$. Therefore, $\underline{Q}^{\pi}(s, a)$ is the fixed point of the Bellman operator $\underline{\mathcal{T}}^\pi$.

\end{proof}

\section{Geometric Understanding of \ours}
\label{app:understand}

\subsection{A Closer Look at Robust RL}

\begin{figure}[!htbp]
    \centering
    \begin{subfigure}[t]{0.22\columnwidth}
        \centering
        \includegraphics[width=\columnwidth]{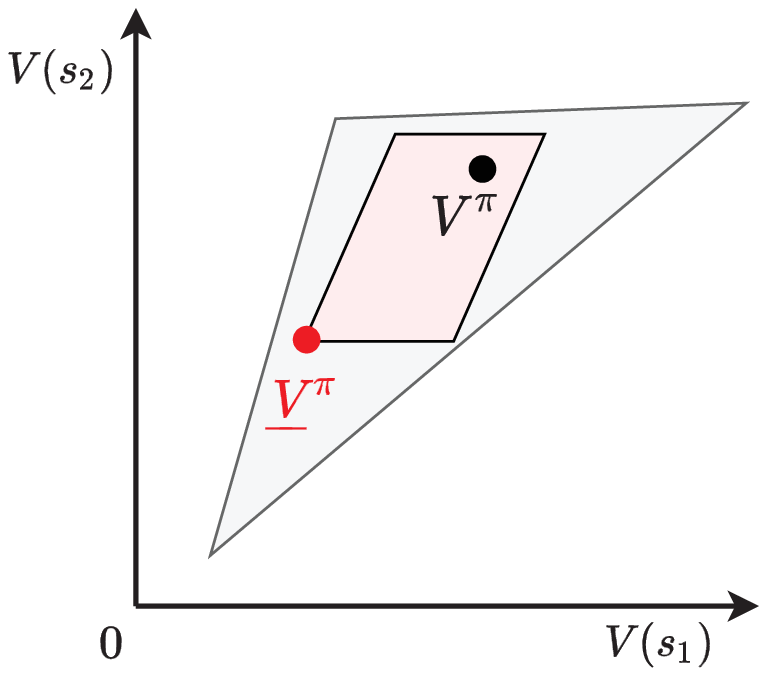}
        \vspace{-2em}
        \caption{\small{Vanilla Training}}
        \label{sfig:train_vanilla}
    \end{subfigure}
    \hfill
    \begin{subfigure}[t]{0.22\columnwidth}
        \centering
        \includegraphics[width=\columnwidth]{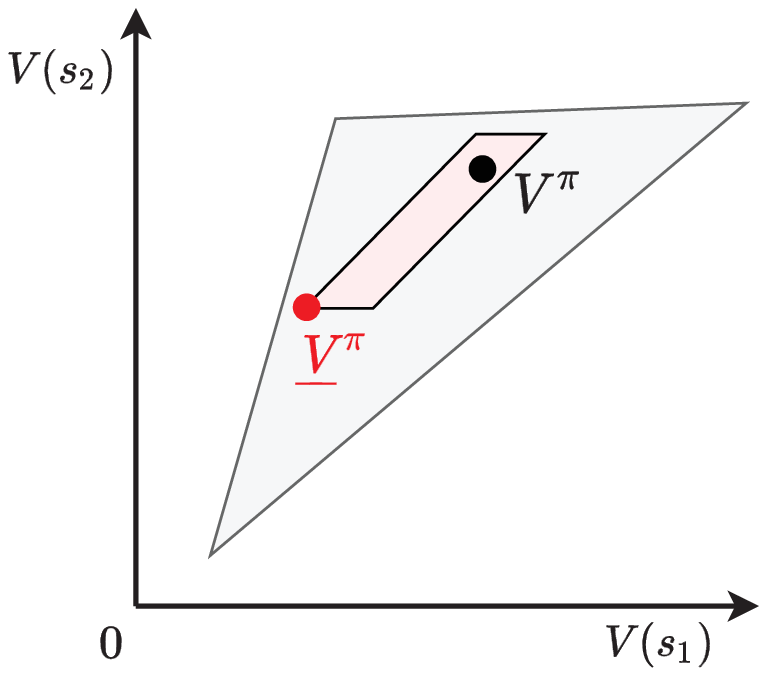}
        \vspace{-2em}
        \caption{\small{Lipschitz-driven}}
        \label{sfig:train_consistent}
    \end{subfigure}
    \hfill
    \begin{subfigure}[t]{0.22\columnwidth}
        \centering
        \includegraphics[width=\columnwidth]{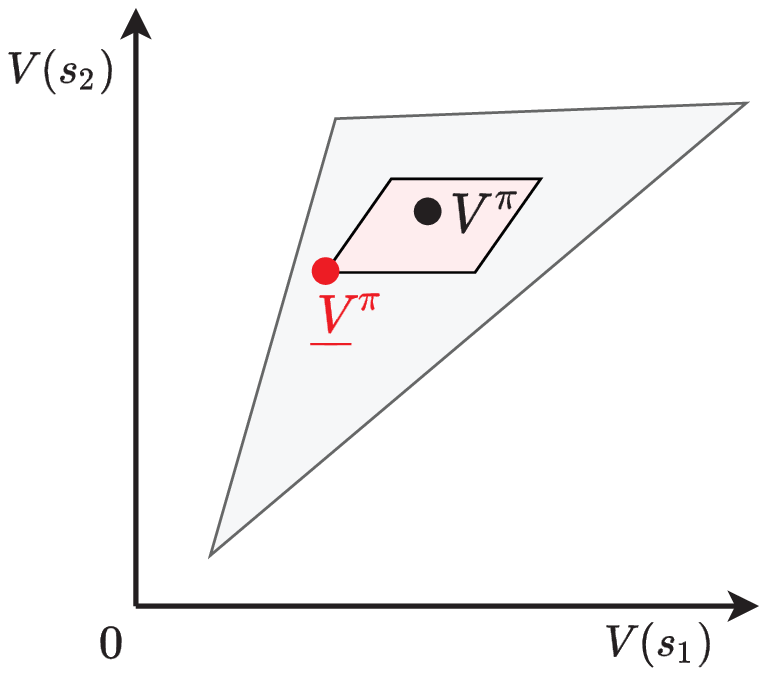}
        \vspace{-2em}
        \caption{\small{Attack-driven}}
        \label{sfig:train_worst}
    \end{subfigure}
    \hfill
    \begin{subfigure}[t]{0.22\columnwidth}
        \centering
        \includegraphics[width=\columnwidth]{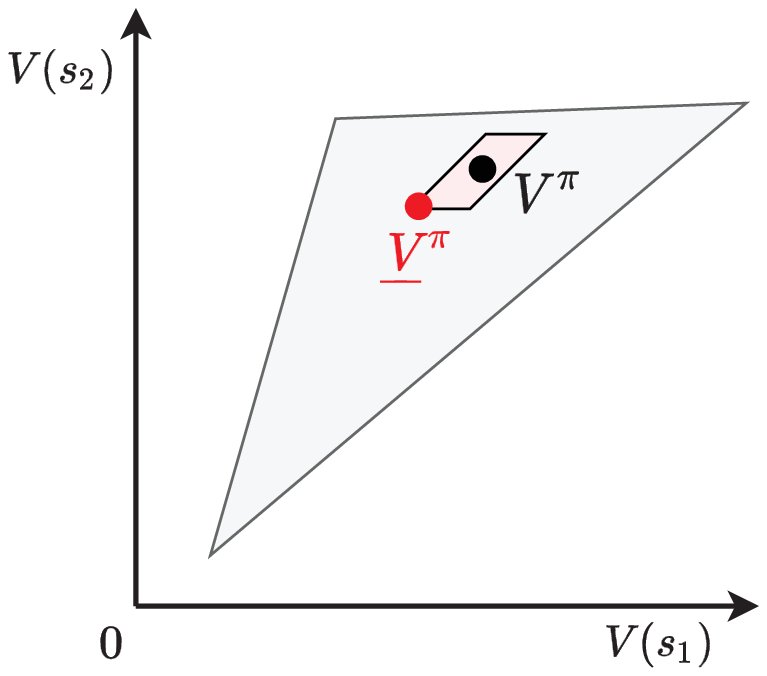}
        \vspace{-2em}
        \caption{\small{Our \ours}}
        \label{sfig:train_ours}
    \end{subfigure}
    \caption{Geometric understanding of different training methods following the polytope theory by~\cite{dadashi2019value} and~\cite{sun2021strongest}. 
    $x,y$ axes represent the policy value for $s_1\in\states$ and $s_2\in\states$.
    The grey polytope depicts the value space of all policies, while the pink polytope (referred to as value perturbation polytope) contains the values of policy $\pi$ under all attacks with given constraint ($\epsilon$-radius $\ell_p$ perturbations on the state input to the policy). 
    $V^\pi$ denotes the value of a learned policy, and $\underline{V}^\pi$ stands for the \worstvname of this policy $\pi$ (located at the bottom leftmost vertex of the value perturbation polytope). \\
    \textbf{Two relations between the value perturbation polytope and policy robustness:}
    The more distant the pink value perturbation polytope's bottom leftmost vertex is from the origin, the higher \worstvname $\pi$ has.
    The smaller the pink value perturbation polytope is, the less vulnerable the policy is (i.e., an $\epsilon$-bounded state perturbation can not lead to a drastic change of the policy value). \\
    \textbf{Our method:} \ours makes a policy more robust via \estname, \worstname and \regname, which shrink the value perturbation polytope and move the value perturbation polytope's bottom leftmost vertex away from the origin.}
    \label{fig:training}
\end{figure}

In real-world applications where observations may be noisy or perturbed, it is important to ensure that the agent not only makes good decisions, but also makes safe decisions. 



\textbf{Existing Robust RL Approaches.}\quad
There are many existing robust training methods for RL, and we summarize the common ideas as the following two categories.\\
\textit{\textbf{(1)} Lipschitz-driven methods:} 
encourage the policy to output similar actions for any pair of clean state and perturbed state, i.e., $\min_{\theta} \max_{s\in\states, \tilde{s}\in\attball(s)} \mathsf{Dist} (\pi_\theta(s), \pi_\theta(\tilde{s}))$, where $\mathsf{Dist}$ can be any distance metric. Therefore, the policy function (network) has small local Lipschitz constant at each clean state. 
Note that this idea is similar to many certifiable robust training methods~\citep{gowal2018effectiveness} in supervised learning. 
For example, Fischer et al.~\cite{fischer2019online} achieve provable robustness for DQN by applying the DiffAI~\citep{mirman2018differentiable} approach, so that the DQN agent selects the same action for any element inside $\attball(s)$.
Zhang et al.~\cite{zhang2020robust} propose to minimize the total variance between $\pi(s)$ and $\pi(\tilde{s})$ using convex
relaxations of NNs. 
Although Lipschitz-driven methods are relatively efficient in training, 
they usually treat all states equally, and do not explicitly consider long-term rewards. 
Therefore, it is hard to obtain a non-vacuous reward certification, especially in continuous-action environments. 
\\
\textit{\textbf{(2)} Attack-driven methods:} 
train the agent under adversarial attacks, which is analogous to Adversarial Training (AT)~\citep{madry2018towards}. 
However, different from AT, a PGD attacker may not induce a robust policy in an RL problem due to the uncertainty and complexity of the environment.
Zhang et al.~\cite{zhang2021robust} propose to alternately train an agent and an RL-based ``optimal'' adversary, so that the agent can adapt to the worst-case input perturbation. Therefore, attack-driven method can be formulated as $\max_\theta \underline{V}^{\pi_\theta}$. Zhang et al.~\cite{zhang2021robust} and a follow-up work by Sun et al.~\cite{sun2021strongest} apply the alternate training approach and obtain state-of-the-art robust performance. 
However, learning the optimal attacker using RL algorithms doubles the learning complexity and the required samples, making it hard to apply these methods to large-scale problems. Moreover, although these attack-driven methods improve the worst-case performance of an agent, the natural reward can be sacrificed.\\
Note that we discuss methods that improve the robustness of deep policies during training. Therefore, the focus is different from some important works~\citep{lutjens2020certified,wu2021crop,kumar2021policy} that directly use non-robust policies and execute them in a robust way.

\textbf{Our Motivation: Geometric Understanding of Robust RL.}\quad
The robustness of a learned RL policy can be understood from a geometric perspective. 
Dadashi et al.~\cite{dadashi2019value} point out that the value functions of all policies in a finite MDP form a polytope, as shown by the grey area in Figure~\ref{fig:training}. Sun et al.~\cite{sun2021strongest} further find that $V^{\tilde{\pi}}$, possible values of a policy $\pi$ under all $\epsilon$-constrained $\ell_p$ perturbations, also form a polytope (pink area in Figure~\ref{fig:training}), which we refer to as the \textit{value perturbation polytope}.
Recall that in robust RL, we pursue a high natural value $V^\pi$, and a high worst-case value $\worstvpi$ which is the lower leftmost vertex of the value perturbation polytope. A vulnerable policy that outputs a different action for a perturbed state as a larger value perturbation polytope.
Lipschitz-driven methods, as Figure~\ref{fig:training}(a) shows, attempts to shrink the size of the value perturbation polytope, but does not necessarily result in a high $\worstvpi$. Attack-driven methods, as Figure~\ref{fig:training} shows, improves $\worstvpi$, but have no control over the size of the value perturbation polytope, and may not obtain a high natural value $V^\pi$.

\textbf{Our Proposed Robust RL Principle.}\quad
In contrast to prior Lipschitz-driven methods and Attack-driven methods, we propose to both ``lift the position'' and ``shrink the size'' of the value perturbation polytope. 
To achieve the above principle in an efficient way, we propose to 
(1) directly estimate and optimize the worst-case value of a policy without training the optimal attacker (\estname and \worstname mechanisms of \ours), and
(2) regularize the local Lipschitz constants of the policy with value-enhanced weights (\regname mechanism of \ours). 
See Section~\ref{sec:alg} for more details of the proposed algorithm.

\section{Algorithm Details}
\label{app:algo}

\subsection{Computing $\advaction$ by Network Bounding Techniques}
\label{app:ibp}
Recall that $\mathcal{A}_{\mathrm{adv}}(s, \pi) = \{a\in\mathcal{A}: \exists \tilde{s}\in\mathcal{B}_\epsilon(s) \text{ s.t. } \pi(\tilde{s}) = a \}$ is the set of actions that $\pi$ may be misled to select in state $s$. Computing the exact $\advaction$ is difficult due to the complexity of neural networks, so we use relaxations of network such as Interval Bound Propagation (IBP)~\cite{wong2018provable, gowal2019scalable} to approximately calculate $\advaction$. 

\textbf{A Brief Introduction to Convex Relaxation Methods.}\quad
Convex relaxation methods are techniques to bound a neural network that provide the upper and lower bound of the neural network output given a bounded $l_p$ perturbation to the input. In particular, we take $l_\infty$ as an example, which has been studied extensively in prior works. Formally, let $f_\theta$ be a real-valued function parameterized by a neural network $\theta$, and let $f_\theta(s)$ denote the output of the neural network with the input $s$. Given an $l_\infty$ perturbation budget $\epsilon$, convex relaxation method outputs $(\underline{f_\theta(s)}, \overline{f_\theta(s)})$ such that 
\begin{equation*}
    \underline{f_\theta(s)}\leq \min_{\|s^\prime-s\|_\infty\leq\epsilon}f_\theta(s')\leq
    \max_{\|s^\prime-s\|_\infty\leq\epsilon}f_\theta(s')
    \leq\overline{f_\theta(s)}
\end{equation*}

Recall that we use $\pi_\theta$ to denote the parameterized policy being trained that maps a state observation to a distribution over the action space, and $\pi$ denotes the deterministic policy refined from $\pi_\theta$ with $\pi(s)=\mathrm{argmax}_{a\in\actions} \pi_\theta(a|s)$. 
$\advaction(s,\pi)$ contains actions that could be selected by $\pi$ (with the highest probability in $\pi_\theta$'s output) when $s$ is perturbed within a $\epsilon$-radius ball. 
Our goal is to approximately identify a superset of $\advaction(s,\pi)$, i.e., $\aadvaction(s,\pi)$, via the convex relaxation of networks introduced above.\\
\textbf{Computing $\advaction$ in Continuous Action Space.}\quad
The most common policy parameterization in a continuous action space is through a Gaussian distribution. Let $\mu_\theta(s)$ be the mean of Gaussian computed by $\pi_\theta(s)$, then $\pi=\mu(s)$. Therefore, we can use network relaxation to compute an upper bound and a lower bound of $\mu_\theta$ with input $\attball(s)$. Then, $\aadvaction(s,\pi)=[\underline{\mu_\theta(s)}, \overline{\mu_\theta(s)}]$, i.e., a set of actions that are coordinate-wise bounded by $\underline{\mu_\theta(s)}$ and $\overline{\mu_\theta(s)}$. 
For other continuous distributions, e.g., Beta distribution, the computation is similar, as we only need to find the largest and smallest actions. 
In summary, we can compute $\aadvaction(s,\pi)=[\underline{\pi_\theta(s)}, \overline{\pi_\theta(s)}]$. 


\textbf{Computing $\advaction$ in Discrete Action Space.}\quad
For a discrete action space, the output of $\pi_\theta$ is a categorical distribution, and $\pi$ selects the action with the highest probability. Or equivalently, in value-based algorithms like DQN, the Q network (can be regarded as $\pi_\theta$) outputs the Q estimates for each action, and $\pi$ selects the action with the highest Q value. 
In this case, we can compute the upper and lower bound of $\pi_\theta$ in every dimension (corresponding to an action), denoted as $\overline{a}_i,\underline{a}_i$, $\forall 1 \leq i \leq |\actions|$. 
Then, an action $a_i\in\actions$ is in $\aadvaction$ if for all $1 \leq j \leq |\actions|, j\neq i$, we have $\overline{a}_i > \underline{a}_j$.

\textbf{Implementation details of $\advaction$}
For a continuous action space, interval bound propagation (IBP) is the cheapest method to implement convex relaxation. We use IBP+Backward relaxation provided by \textit{auto\_LiRPA} library, following \citep{zhang2020robust} to efficiently produce tighter bounds $\advaction$ for the policy networks $\pi_{theta}$. For a discrete action space, we compute the layer-wise output bounds for the Q-network by applying robustness verification algorithms from \citep{oikarinen2020robust}.

\subsection{Worst-case-aware Robust PPO (\ourppo)}
\label{app:ppo}

In policy-based DRL methods~\citep{schulman2015trust,lillicrap2015continuous,schulman2017proximal} such as PPO, the actor policy $\pi_\theta$ is optimized so that it increases the probability of selecting actions with higher critic values. Therefore, we combine our \worstcriticname and the original critic function, and optimize $\pi_\theta$ such that both the natural value ($\lossrl$) and the \worstqname $\worstcritic$ ($\lossworst$) can be increased ($\lossrl$ and $\lossworst$). At the same time, $\pi_\theta$ is also regularized by $\lossreg$.  

We provide the full algorithm of \ourppo in Algorithm \ref{alg:ppo} and highlight the differences with the prior method SA-PPO. \ourppo needs to train an additional \worstcriticname $\worstcritic$ to provide the robust-PPO-clip objective. The perturbation budget $\epsilon_t$ increases slowly during training. 
The implementation of $\lossreg$ is the same as the SA-regularizer~\cite{zhang2020robust}. For computing the state importance weight $w_{s_t}$, because there is no Q-value network in PPO, we provide a different formula to measure the state importance without extra calculation (Line 11 in Algorithm~\ref{alg:ppo}).
\begin{algorithm}[!h]
	\renewcommand{\algorithmicrequire}{\textbf{Input:}}
	\renewcommand{\algorithmicensure}{\textbf{Output:}}
	\caption{Worst-case-aware Robust PPO (\ourppo). We highlight the difference compares with SA-PPO \cite{zhang2020robust} in \textcolor{teal}{blue}.}
	\begin{algorithmic}[1]
	    \REQUIRE Number of iterations $T$, a schedule $\epsilon_t$ for the perturbation radius $\epsilon$, weights $\weightworst,\weightreg$
		\STATE Initialize policy network $\pi_{\theta_\pi}(a\mid s)$ , value network $V_{\theta_V}(s)$ and \worstcriticname network $\worstcritic(s, a)$ with parameters $\theta_{\pi}$, $\theta_V$ and $\phi$
		\FOR {$k=0, 1, ..., T$}
		\STATE Collect a set of trajectories $\mathcal{D}=\{\tau_k\}$ by running $\pi_{\theta_\pi}$ in the environment, each trajectory $\tau_k$ contains $\tau_{k}:=\left\{\left(s_t, a_t, r_t, s_{t+1}\right)\right\}, t \in\left[\left|\tau_{k}\right|\right]$
		\STATE Compute rewards-to-go $\hat{R}_{t}$ for each step $t$ in every trajectory $k$ with discount factor $\gamma$
		\STATE Compute advantage estimation $\hat{A}_t$ based on the current value function $V_{\theta_V}(s_t)$ and cumulative reward  $\hat{R}_{t}$ for each step $t$
		\STATE Update parameters of value function $\theta_V$ by regression on mean-squared error:
		$$\theta_V \leftarrow \mathop{\arg\min}\limits_{\theta_V} \frac{1}{\left|\mathcal{D}\right||\tau_{k}|} \sum_{\tau_k \in \mathcal{D}} \sum_{t=0}^{|\tau_{k}|}\left(V_{\theta_V}\left(s_{t}\right)-\hat{R}_{t}\right)^{2}$$
		\STATE \textcolor{teal}{Use IBP to compute bounds of current policy network $\pi$:}\\
		   Find the upper bound $\overline{\pi}\left(s_{t+1}, \epsilon ; \theta\right)$ and lower bound $\underline{\pi}\left(s_{t+1}, \epsilon ; \theta\right)$ of the policy network $\pi_{\theta_\pi}$
		\STATE \textcolor{teal}{Select the worst action for next states:}\\
		Calculate the action satisfied $\hat{a}_{t+1} = \mathop{\arg\min}\limits_{a\in [\underline{\pi}, \overline{\pi}]} \worstcritic(s_{t+1}, a)$ with the \worstcriticname network $\worstcritic$ using gradient descent.
		\STATE \textcolor{teal}{Compute next worst-case value:}\\
		Set $\underline{y}_{t}=\left\{\begin{array}{ll}
        r_{t} & \mbox { for terminal } s_{t+1} \\
        r_{t}+\gamma \worstcritic(s_{t+1}, \hat{a}_{t+1}) & \mbox { for non-terminal } s_{t+1}
        \end{array}\right.$
        \STATE \textcolor{teal}{Update parameters of \worstcriticname network $\phi$} by minimizing the TD-error ($\lossest$):
        $$\phi \leftarrow \mathop{\arg\min}\limits_{\phi} \frac{1}{\left|\mathcal{D}\right||\tau_{k}|} \sum_{\tau_k \in \mathcal{D}} \sum_{t=0}^{|\tau_{k}|}(\underline{y}_t -\worstcritic(s_t, a_t))^2$$
        \STATE For each state $s_t$, calculate a \textcolor{teal}{state importance weight $w_{s_t}$} by $V_{\theta_V}\left(s_{t}\right) - \min\limits_{a}\worstcritic(s_t, a)$ for $s_t$
        \STATE Solve the \textcolor{teal}{value-enhanced} state regularization loss by SGLD (Stochastic gradient Langevin dynamics~\citep{gelfand1991recursive}) (from SA-PPO \citep{zhang2020robust}):
        $$\lossreg(\pi_\theta) = \frac{1}{N} \sum_{t=1}^N w(s_t) \max_{\tilde{s}_t\in\attball(s_t)} \mathsf{Dist} (\pi_\theta(s_t), \pi_\theta(\tilde{s}_t))$$
        \STATE Update the policy network by minimizing the \textcolor{teal}{Robust-PPO-Clip objective} (via ADAM):
        \begin{equation*}
        \begin{array}{l}
        \resizebox{0.91\textwidth}{!}{$\theta_{\pi} \leftarrow
        \mathop{\arg\min}\limits_{\theta_{\pi}^{\prime}} \frac{1}{|\mathcal{D}||\tau_{k}|}
        \Big[ \sum_{\tau_k \in \mathcal{D}} \sum_{t=0}^{|\tau_{k}|} \min \Big(\rho_{\theta_{\pi}^{\prime}}(a_t \mid s_t) (\hat{A}_t+\weightworst\worstcritic(s_t, a_t)), g(\rho_{\theta_{\pi}^{\prime}}(a_t \mid s_t)\Big) (\hat{A}_t+\weightworst\worstcritic(s_t, a_t)))+\weightreg w(s_i)\lossreg(\pi_\theta)\Big]$} \\
        \mbox { where } \rho_{\theta_{\pi}^{\prime}}(a_t \mid s_t):=\frac{\pi_{\theta^{\prime}_{\pi}}\left(a_t \mid s_t\right)}{\pi_{\theta_{\pi}}\left(a_t \mid s_t\right)}, g(\rho):={\mathrm{clip}}\left(\rho_{\theta_{\pi}^{\prime}}\left(a_t \mid s_t\right), 1-\epsilon_{\mathrm{clip}}, 1+\epsilon_{\mathrm{clip}}\right)
        \end{array}
        \end{equation*}
		\ENDFOR
	\end{algorithmic}  
\label{alg:ppo}
\end{algorithm}

\subsection{Worst-case-aware Robust DQN (\ourdqn)}
\label{app:dqn}

For value-based DRL methods~\citep{mnih2013playing,guez2015deep,wang2016dueling} such as DQN, a Q network is learned to evaluate the natural action value. Although the policy is not directly modeled by a network, the Q network induces a greedy policy by $\pi(s)=\mathrm{argmax}_a Q(s,a)$. 
To distinguish the acting policy and the natural action value, we keep the original Q network, and learn a new Q network that serves as a robust policy. This new Q network is called a \textit{robust Q network}, denoted by $Q_r$, which is used to take greedy actions $a=\pi(s):=\mathrm{argmax}_a Q_r(s,a)$
In addition to the original vanilla Q network $Q_v$ and the robust Q network $Q_r$, we learn the \worstcriticname network $\worstcritic$, which evaluates the \worstqname of the greedy policy induced by $Q_r$. Then, we update $Q_r$ by assigning higher values for actions with both high natural Q value and high \worstqname ($\lossrl$ and $\lossworst$), while enforcing the network to output the same action under bounded state perturbations ($\lossreg$). 

\ourdqn is presented in Algorithm \ref{alg:dqn}. \ourdqn trains three Q-value functions including a vanilla Q network, a worst-case Q network, and a robust Q network. The worst-case Q $\worstcritic$ is learned to estimate the worst-case performance and the robust Q is updated using the vanilla value and worst-case value together. Moreover, a target Q network is used as the original DQN implementation, to compute the target value when updating the vanilla Q network (Line 8 to 10 in Algorithm~\ref{alg:dqn}). To learn the worst-case critic $\worstcritic$, we select the worst-attack action from the estimated possible perturbed action set $\aadvaction$ to compute the worst-case TD loss $\lossest$ (Line 11 to 15). The implementation of $\lossreg$ is the same as the SA-regularizer~\cite{zhang2020robust}, where the robust Q network is regularized. To update the robust Q, we use a special $y^r_i$ which combines the target Q $Q_{v^{\prime}}$ and $Q_r$ for the next state to compute the TD loss, and minimize the $\lossreg$ weighted by the state importance $w(s_i)$ (Line 16 to 17). In \ourdqn, we use an increasing $\epsilon_t$ schedule and a more slowly increasing worst-case schedule $\kappa_{wst}(t)$ for robust Q training. 
\begin{algorithm}[!h]
	\renewcommand{\algorithmicrequire}{\textbf{Input:}}
	\renewcommand{\algorithmicensure}{\textbf{Output:}}
	\caption{Worst-case-aware Robust DQN (\ourdqn). We highlight the difference compares with SA-DQN \cite{zhang2020robust} in \textcolor{teal}{blue}.}
	\begin{algorithmic}[1]
	    \REQUIRE Number of iterations $T$, target network update coefficient $\tau$, a schedule $\epsilon_t$ for the perturbation radius $\epsilon$, a worst-case schedule $\weightworst(t)$ for weight $\weightworst$, regularization weight $\weightreg$
		\STATE Initialize a vanilla Q network $Q_v(s, a)$, target Q network $Q_{v^\prime}(s,a)$ , a robust Q network $Q_r(s, a)$, and a \worstcriticname $\worstcritic(s, a)$ with parameters $\theta_{Q_v}$, $\theta_{Q_{v^\prime}}$, $\theta_{Q_r}$, and $\phi$
		\STATE Initialize replay buffer $\mathcal{B}$
		\FOR {$k=0, 1, ..., T$}
		\STATE With probability $\beta$ select random action $a_t$, otherwise select $a_t = \mathop{\arg\max}\limits_{a}Q_r(s_t, a|\theta_{Q_r})$
		\STATE Execute action $a_t$ in environment and observe reward $r_t$ and the next state $s_{t+1}$.
		\STATE Store transition $\{s_t, a_t, r_t, s_{t+1}\}$ in $\mathcal{B}$
		\STATE Sample random a minibatch of $N$ transitions $\{s_i, a_i, r_i, s_{i+1}\}$ from $\mathcal{B}$
		\STATE Set $y_{i}=\left\{\begin{array}{ll}r_{i} & \text { for terminal } s_{i+1} \\ r_{i}+\gamma \max _{a^{\prime}} Q_{v^\prime}\left(s_{i+1}, a^{\prime} ; \theta\right) & \text { for non-terminal } s_{i+1}\end{array}\right.$
		\STATE Compute TD-loss for the vanilla Q network: $L(s_i, a_i, s_{i+1};\theta)=(y_i-Q_v(s_i, a_i;\theta))^2$ and optimize $\theta_{Q_v}$
		\STATE Soft update the target action-value network: $\theta_{Q_{v^\prime}}\leftarrow \tau\theta_{Q_v}+(1-\tau)\theta_{Q_{v^\prime}}$
		\STATE \textcolor{teal}{Computing bounds of robust action-value function:}\\
		For each action $a$ in action space $\mathcal{A}$, calculate the output bounds of robust action-value function $Q_r$ under $\epsilon_t$-bounded perturbations using IBP to input $s_{i+1}$: ${Q}_l(s_{t+1}, a, \epsilon_t)$ and ${Q}_u(s_{t+1}, a,\epsilon_t)$.
		\STATE \textcolor{teal}{Find the possible perturbed action set:}\\
		 For every action $a\in\mathcal{A}$, if ${Q}_u(s_{t+1}, a,\epsilon_t)> {Q}_l(s_{t+1}, a', \epsilon_t), \forall a' \in \mathcal{A}$, then add $a$ in the perturbed action set $\aadvaction$
		 \STATE \textcolor{teal}{Calculate the worst-attack action}: $\hat{a}_{i+1}=\mathop{\arg\min}\limits_{a \in \aadvaction}\worstcritic(s_{i+1}, a)$.
		\STATE Set $\underline{y}_{i}=\left\{\begin{array}{ll}r_{i} & \text { for terminal } s_{i+1} \\ r_i+\gamma\worstcritic(s_{i+1},\hat{a}_{i+1};\theta) & \text { for non-terminal } s_{i+1}\end{array}\right.$
		\STATE \textcolor{teal}{Compute TD-loss for \worstcriticname}: $\lossest=(\underline{y}_i-\worstcritic(s_i, a_i;\phi))^2$ and perform a gradient descent step with respect to the parameters $\phi$
		\STATE \textcolor{teal}{Calculate the state importance $w_{s_i}$} for each $s_i$ by normalizing $\max\limits_{a}Q_v(s_t, a) - \min\limits_{a}Q_v(s_t, a)$
		\STATE \textcolor{teal}{Update the robust Q function $Q_r$} based on the modified TD-Loss and \textcolor{teal}{value-enhanced} state regularization: $$L(s_i, a_i, s_{i+1};\theta_{Q_r})=(y^r_i-Q_r(s_i, a_i;\theta))^2 + \weightreg w(s_i)\lossreg(\theta_{Q_r})$$
		where $y^r_i=r_{i}+\gamma \max _{a^{\prime}} \Big[\weightworst(t)Q_{v^\prime}\left(s_{i+1}, a^{\prime};\theta\right)+(1-\weightworst(t))\worstcritic\left(s_{i+1}, a^{\prime};\theta\right)\Big]$ if $s_{i+1}$ is a non-terminal state, otherwise $y^r_i=r_i$
		\ENDFOR
	\end{algorithmic}  
\label{alg:dqn}
\end{algorithm}

\subsection{Worst-case-aware Robust A2C (\ourac)}
\label{app:a2c}
We also provide WocaR-A2C based on A2C implementation in Algorithm~\ref{alg:a2c}. Differ from the original A2C, \ourac needs to learn an additional $\worstcritic$ similar to \ourppo. To learn $\worstcritic$, we compute the output bounds for the policy network $\pi_{\theta_{\pi}}$ under $\epsilon$-bounded perturbations and then select the worst action $\hat{a}_{t+1}$ to calculate the TD-loss $\lossest$ (Line 6 to 9). The solutions for state importance weight $w(s_t)$ and regularization $\lossreg$ are same as \ourppo (Line 10-11). To learn the policy network $\pi_{\theta_{\pi}}$, we minimize the $\worstcritic$ value together with the original actor loss (Line 12).
\begin{algorithm}[!h]
	\renewcommand{\algorithmicrequire}{\textbf{Input:}}
	\renewcommand{\algorithmicensure}{\textbf{Output:}}
	\caption{Worst-case-aware Robust A2C (WocaR-A2C). We highlight the difference compares with SA-A2C \cite{zhang2020robust} in \textcolor{teal}{blue}.}
	\begin{algorithmic}[1]
	    \REQUIRE Number of iterations $T$, a schedule $\epsilon_t$ for the perturbation radius $\epsilon$, weights $\weightworst,\weightreg$
		\STATE Initialize policy network $\pi_{\theta_\pi}(a\mid s)$ , value network $V_{\theta_V}(s)$ and \worstcriticname network $\worstcritic(s, a)$ with parameters $\theta_{\pi}$, $\theta_V$ and $\phi$
		\FOR {$k=0, 1, ..., T$}
		\STATE Collect a set of trajectories $\mathcal{D}=\{\tau_k\}$ by running $\pi_{\theta_\pi}$ in the environment, each trajectory $\tau_k$ contains $\tau_{k}:=\left\{\left(s_t, a_t, r_t, s_{t+1}\right)\right\}, t \in\left[\left|\tau_{k}\right|\right]$
		\STATE Compute advantage function $A_t$ by
		$$A_t = r_t + \gamma V_{\theta_V}(s_{t+1}) - V_{\theta_V}(s_t)$$
		\STATE Update parameters of value function $\theta_V$ by regression on mean-squared error:
		$$\theta_V \leftarrow \mathop{\arg\min}\limits_{\theta_V} \frac{1}{\left|\mathcal{D}\right||\tau_{k}|} \sum_{\tau_k \in \mathcal{D}} \sum_{t=0}^{|\tau_{k}|}A_t^2$$
		\STATE \textcolor{teal}{Use IBP to compute bounds of current policy network $\pi$:}\\
		   Find the upper bound $\overline{\pi}\left(s_{t+1}, \epsilon ; \theta\right)$ and lower bound $\underline{\pi}\left(s_{t+1}, \epsilon ; \theta\right)$ of the policy network $\pi_{\theta_\pi}$
		\STATE \textcolor{teal}{Select the worst action for next states:}\\
		Calculate the action satisfied $\hat{a}_{t+1} = \mathop{\arg\min}\limits_{a\in [\underline{\pi}, \overline{\pi}]} \worstcritic(s_{t+1}, a)$ with the \worstcriticname network $\worstcritic$ using gradient descent.
		\STATE \textcolor{teal}{Compute next worst-case value:}\\
		Set $\underline{y}_{t}=\left\{\begin{array}{ll}
        r_{t} & \mbox { for terminal } s_{t+1} \\
        r_{t}+\gamma \worstcritic(s_{t+1}, \hat{a}_{t+1}) & \mbox { for non-terminal } s_{t+1}
        \end{array}\right.$
        \STATE \textcolor{teal}{Update parameters of \worstcriticname network $\phi$} by minimizing the TD-error ($\lossest$):
        $$\phi \leftarrow \mathop{\arg\min}\limits_{\phi} \frac{1}{\left|\mathcal{D}\right||\tau_{k}|} \sum_{\tau_k \in \mathcal{D}} \sum_{t=0}^{|\tau_{k}|}(\underline{y}_t -\worstcritic(s_t, a_t))^2$$
        \STATE For each state $s_t$, calculate a \textcolor{teal}{state importance weight $w(s_t)$} by $V_{\theta_V}\left(s_{t}\right) - \min\limits_{a}\worstcritic(s_t, a)$ for $s_t$
        \STATE Solve the \textcolor{teal}{value-enhanced} state regularization loss~\citep{zhang2020towards} by SGLD (Stochastic gradient Langevin dynamics~\citep{gelfand1991recursive}):
        $$\lossreg(\pi_{\theta_\pi}) = \frac{1}{N} \sum_{t=1}^N w(s_t) \max_{\tilde{s}_t\in\attball(s_t)} \mathsf{Dist} (\pi_{\theta_\pi}(s_t), \pi_{\theta_\pi}(\tilde{s}_t))$$
        \STATE Update the policy network by (via ADAM)
        $$\begin{array}{l}
        \theta_{\pi} \leftarrow
        \mathop{\arg\min}\limits_{\theta_{\pi}^{\prime}} \frac{1}{|\mathcal{D}||\tau_{k}|}
        \left[ \sum_{\tau_k \in \mathcal{D}} \sum_{t=0}^{|\tau_{k}|} (A_t log \pi_{\theta_\pi}(s_t) + \weightworst\worstcritic(s_t, a_t))\right]
        \end{array}$$
		\ENDFOR
	\end{algorithmic}  
\label{alg:a2c}
\end{algorithm}

\subsection{Extension to Action Attacks}
\label{app:extension}
Although our paper mainly focuses on state attack, our proposed techniques and algorithms based on the worst-attack Bellman operator can be easily extended to action attack, which is another threat model studied in previous works ~\citep{pinto2017robust,tan2020robustifying,tessler2019action}. In fact, for action attack, we even do not need to apply IBP for the worst-attack Bellman backup. We could just simply replace $\advaction$ with the set of actions that the agent could take under attack, then the rest of the algorithms will follow the exact same as the ones presented here.

\section{Experiment Details and Additional Results}
\label{app:exp}

\subsection{Implementation Details}
\label{app:exp:imp}
For reproducibility, the reported results are selected from 30 agents for different training methods with medium performance due to the high variance in RL training.

\subsubsection{PPO in MuJoCo}

\textbf{(a) PPO Baselines}\quad 

\textbf{Vanilla PPO}\quad We use the optimal hyperparameters from \cite{zhang2020robust} with the original fully connected (MLP) structure as the policy network for vanilla PPO training on all environments. On Hopper, Walker2d and Halfcheetah, we train for 2 million steps (976 iterations) , and 10 million steps (4882 iterations) on Ant to ensure convergence, which are consistent with other baselines (except ATLA methods).\\
\textbf{SA-PPO}\quad We use the hyperparameters using a grid search and solve the regularizer using convex relaxation with the IBP+Backward scheme to solve the regularizer. The regularization parameter $kappa$ is chosen in $\{0.01, 0.03, 0.1, 0.3, 1.0\}$. \\
\textbf{ATLA-PPO}\quad The hyperparameters for both policy and adversary are tuned for vanilla PPO with LSTM models. A larger entropy bonus coefficient is set to allow sufficient exploration. We set $N_v = N_{\pi} = 1$ for all experiments. We train 2441 iterations for Hopper, Walker2d, and Halfcheetah as well as 4882 iterations for Ant.\\
\textbf{PA-ATLA-PPO}\quad We use the hyperparameters similar to ATLA-PPO and conduct a grid search for a part of adversary hyperparameters including the learning rate and the entropy bonus coefficient.\\
\textbf{RADIAL-PPO}\quad RADIAL-PPO applies the same value of hyperparameters from \citep{oikarinen2020robust}. We train agents with the same iterations  aligning vanilla PPO for fair comparison. 

\textbf{(b) PPO Attackers}\quad 

For \textbf{Random} and \textbf{MaxDiff} attack, we directly use the implementation from \cite{zhang2021robust}. The reported rewards under RS attack are from 30 trained robust value function, which is used to attack agents.\\
For \textbf{SA-RL} attack, a grid search of the optimal hyperparameters for each robust agents is conducted to find the strongest attacker. The strength of the regularization $\kappa$ is set as $1 \times 10^{-6}$ to 1.\\
For \textbf{PA-AD} attack, the adversaries are trained by PPO with a grid search of hyperparameters to obtain the strongest adversary.\\
For different types of RL-based attacks, we respectively train 100 adversaries and report the worst rewards among all trained adversaries.

\textbf{(c) \ourppo} \quad We use the same LSTM structure (single layer with 64 hidden neurons as in vanilla PPO agents. With a grid search experiment, we find the optimal hyperparameters for \ourppo. Specially, we use PGD to compute bounds for the policy network and convex relaxation to solve the state regularization. The number of \ourppo training steps in all environments are the same as those in vanilla PPO. We tune the adjustable weight $\kappa_{wst}$ and increase $\kappa_{wst}$ from $0$ to the target value. For Hopper, Walker2d and Halfcheetah, $\kappa_{wst}$ is linearly increasing and we set the target value as 0.8. For Ant, we choose the exponential increase and the target value as 0.5.

\subsubsection{DQN in Atari}

\textbf{(a) DQN Baselines} \quad

\textbf{Vanilla DQN}\quad We follow \cite{zhang2020robust} and \citep{oikarinen2020robust} in hyperparameters and network structures for vanilla DQN training. The implementation of all our baselines applies Double DQN~\citep{van2016deep} and Prioritized Experience Replay~\citep{schaul2015prioritized}. For each Atari environment without framestack, we normalize the pixel values to $[0, 1]$ and clip rewards to $[-1, +1]$. For reliably convergence, we run $6\times 10^6$ steps for all baselines on all environments. Additionally, we use a replay buffer with a capacity of $5\times 10$. During testing, we evaluate agents without epsilon greedy exploration for 1000 episodes.\\
\textbf{SA-DQN} \quad 
SA-DQN use the same settings of network structures and hyperparameters as in vanilla DQN. The regularization parameter $\kappa$ is chosen from ${0.005, 0.01, 0.02}$ and the schedule of $\epsilon$ during training also follows \cite{zhang2020robust}. \\
\textbf{RADIAL-DQN}\quad
Following the original implementation from \cite{oikarinen2020robust}, we reproduce the results of RADIAL-DQN with our environment settings.

\textbf{(b) DQN Attackers} \quad 

For \textbf{PGD} attacks, we apply 10-step untargeted PGD attacks. We also try 50-step PGD attacks, but we find that the rewards of robust agents do not further reduce.\\
For \textbf{MinBest} attacks, we use FGSM to compute state perturbations following \cite{huang2017adversarial}.\\
For \textbf{PA-AD} attacks, the PA-AD attackers are learned with the ACKTR algorithm. We use a learning rate $0.0001$ and train the attackers for 5 million frames. 

\textbf{(c) \ourdqn} \quad 
For \ourdqn, we keep the same network architectures and hyperparameters as in vanilla DQN agents. During training, we set the adjustable weight $\kappa_{wst}$ as 0 for the first $2 \times 10^6$ steps, and then exponentially increase it from 0 to 0.5 for $4 \times 10^6$ steps.

\subsection{Additional Experiment Results on Robustness Performance}
\label{app:exp:res}

\begin{table*}[!t]
\vspace{-0.5em}
\centering
\renewcommand{\arraystretch}{1.2}
\resizebox{\textwidth}{!}{%
\setlength{\tabcolsep}{4pt}
\begin{tabular}{c|c|p{2cm}<{\centering} p{2cm}<{\centering} p{2cm}<{\centering} p{2cm}<{\centering} p{2cm}<{\centering} p{2cm}<{\centering}}
\toprule
\textbf{Environment}                  & \textbf{Model}          & \begin{tabular}[c]{@{}c@{}}{\bf Natural}\\{\bf Reward}\end{tabular} & \textbf{Random}      & \textbf{MAD}          & \textbf{RS}          & \textbf{SA-RL}       & \textbf{PA-AD}        \\
 \hline
\multirow{6}{5.5em}{\begin{tabular}[c]{@{}c@{}}{\textbf{Halfcheetah}}\\state-dim: 17\\$\epsilon$=0.15\end{tabular}}    & PPO (vanilla)    & \textbf{7117 $\pm$ 98}    & 5486 $\pm$ 1378 & 1836 $\pm$ 866  & 489 $\pm$ 758   & -660 $\pm$ 218  & -356 $\pm$ 407     \\  
 & SA-PPO           & 3632 $\pm$ 20      & 3619 $\pm$ 18   & 3624 $\pm$ 23   & 3283 $\pm$ 20   & 3028 $\pm$ 23  & 2512 $\pm$ 16  \\
 & ATLA-PPO            & 6157 $\pm$ 852     & 6164 $\pm$ 603  & 5790 $\pm$ 174  & 4806 $\pm$ 392  & 5058 $\pm$ 418  & 2576 $\pm$ 548  \\
& PA-ATLA-PPO      & 6289 $\pm$ 342     & \textbf{6215 $\pm$ 346}  & \textbf{5961 $\pm$ 253}  & 5226 $\pm$ 114  & 4872 $\pm$ 379  & 3840 $\pm$ 273  \\
& RADIAL-PPO & 4724 $\pm$ 14 & 4731 $\pm$ 42 & 3994 $\pm$ 156 & 3864 $\pm$ 232 & 3253 $\pm$ 131 & 2674 $\pm$ 168\\
& \cellcolor{lightgray}{\textbf{\ourppo (Ours)}} & \cellcolor{lightgray}{6032 $\pm$ 68}    & \cellcolor{lightgray}{5969 $\pm$ 149}  & \cellcolor{lightgray}{5850 $\pm$ 228}  & \cellcolor{lightgray}{\textbf{5319 $\pm$ 220}}  & \cellcolor{lightgray}{\textbf{5365 $\pm$ 54}}  & \cellcolor{lightgray}{\textbf{4269 $\pm$ 172}}  \\
\midrule
\multirow{6}{5.5em}{\begin{tabular}[c]{@{}c@{}}{\textbf{Hopper} }\\state-dim: 11\\$\epsilon$=0.075\end{tabular}}    & PPO (vanilla)    & 3167 $\pm$ 542      & 2101 $\pm$ 793  & 1410 $\pm$ 655  & 794 $\pm$ 238   & 636 $\pm$ 9  & 160 $\pm$ 136     \\
 & SA-PPO           & 3705 $\pm$ 2     & 2710 $\pm$ 801  & 2652 $\pm$ 835  & 1130 $\pm$ 42     & 1076 $\pm$ 791    & 856 $\pm$ 21  \\
 & ATLA-PPO            & 3291 $\pm$ 600    & 3165 $\pm$ 576  & 2814 $\pm$ 725  & 2244 $\pm$ 618  & 1772 $\pm$ 802 & 1232 $\pm$ 350  \\
& PA-ATLA-PPO      & 3449 $\pm$ 237      & 3325 $\pm$ 239  & 3145 $\pm$ 546  & 3002 $\pm$ 329  & 1529 $\pm$ 284   & 2521 $\pm$ 325  \\
& RADIAL-PPO & \textbf{3740 $\pm$ 44} & \textbf{3729 $\pm$ 100} & 3214 $\pm$ 142 & 2141 $\pm$ 232 & 1722 $\pm$ 186 & 1439  $\pm$ 204\\
 & \cellcolor{lightgray}{\textbf{\ourppo (Ours)}} & \cellcolor{lightgray}{3616 $\pm$ 99}     & \cellcolor{lightgray}{3633 $\pm$ 30}   & \cellcolor{lightgray}{\textbf{3541 $\pm$ 207}}  & \cellcolor{lightgray}{\textbf{3277 $\pm$ 159}}  & \cellcolor{lightgray}{\textbf{2390 $\pm$ 145}}   & \cellcolor{lightgray}{\textbf{2579 $\pm$ 229}}  \\
 \midrule
\multirow{6}{5.5em}{\begin{tabular}[c]{@{}c@{}}{\textbf{Walker2d}}\\state-dim: 17\\$\epsilon$=0.05\end{tabular}}    & PPO (vanilla)    & 4472 $\pm$ 635     & 3007 $\pm$ 1200 & 2869 $\pm$ 1271 & 1336 $\pm$ 654  & 1086 $\pm$ 516    & 804 $\pm$ 130     \\
 & SA-PPO           & 4487 $\pm$ 61      & 4465 $\pm$ 39  & 3668 $\pm$ 689  & 3808 $\pm$ 138  & 2908 $\pm$ 336  & 1042 $\pm$ 353  \\
 & ATLA-PPO            & 3842 $\pm$ 475     & 3927 $\pm$ 368  & 3836 $\pm$ 492  & 3239 $\pm$ 294  & 3663 $\pm$ 707  & 1224 $\pm$ 770  \\
& PA-ATLA-PPO      & 4178 $\pm$ 529     & 4129 $\pm$ 78   & 4024 $\pm$ 272  & 3966 $\pm$ 307  & 3450 $\pm$ 178  & 2248 $\pm$ 131  \\
& RADIAL-PPO & \textbf{5251 $\pm$ 12} & \textbf{5184 $\pm$ 42} &  \textbf{4494 $\pm$ 150}  & 3572  $\pm$ 239 & 3320 $\pm$ 245 & 1395  $\pm$ 194\\
 & \cellcolor{lightgray}{\textbf{\ourppo (Ours)}} & \cellcolor{lightgray}{4156 $\pm$ 495}     & \cellcolor{lightgray}{4244 $\pm$ 157}  & \cellcolor{lightgray}{4177 $\pm$ 176}  & \cellcolor{lightgray}{\textbf{4093 $\pm$ 138}}  & \cellcolor{lightgray}{\textbf{3770 $\pm$ 196}}  & \cellcolor{lightgray}{\textbf{2722 $\pm$ 173}}  \\
 \midrule
 \multirow{6}{5.5em}{\begin{tabular}[c]{@{}c@{}}{\textbf{Ant}}\\state-dim: 111\\$\epsilon$=0.15\end{tabular}}    & PPO (vanilla)    & \textbf{5687 $\pm$ 758}     & 5261 $\pm$ 1005 & 1759 $\pm$ 828 & 268 $\pm$ 227  & -872 $\pm$ 436   & -2580 $\pm$ 872     \\
 & SA-PPO           & 4292 $\pm$ 384     & 4986 $\pm$ 452  & 4662 $\pm$ 522  & 3412 $\pm$ 1755 & 2511 $\pm$ 1117 & -1296 $\pm$ 923  \\
 & ATLA-PPO            & 5359 $\pm$ 153     & 5366 $\pm$ 104  & 5240 $\pm$ 170  & 4136 $\pm$ 149  & 3765 $\pm$ 101  & 220 $\pm$ 338  \\
& PA-ATLA-PPO      & 5469 $\pm$ 106     & 5496 $\pm$ 158  & \textbf{5328 $\pm$ 196}  & 4124 $\pm$ 291  & 3694 $\pm$ 188  & 2986 $\pm$ 364  \\
& RADIAL-PPO & 5076 $\pm$ 254 & 5031 $\pm$ 142 & 4777 $\pm$ 156 & 3731 $\pm$ 177 & 3188 $\pm$ 115 & 1544 $\pm$ 194\\
 & \cellcolor{lightgray}{\textbf{\ourppo (Ours)}} & \cellcolor{lightgray}{5596 $\pm$ 225}     & \cellcolor{lightgray}{\textbf{5558 $\pm$ 241}}  & \cellcolor{lightgray}{5284 $\pm$ 182}  & \cellcolor{lightgray}{\textbf{4339 $\pm$ 160}}  & \cellcolor{lightgray}{\textbf{3822 $\pm$ 185}}  & \cellcolor{lightgray}{\textbf{3164 $\pm$ 163}}  \\
 
\bottomrule
\end{tabular}}
\caption{Average episode rewards $\pm$ standard deviation over 50 episodes on five baselines and \ourppo on Hopper, Walker2d, Halfcheetah, and Ant. Natural reward and rewards under five types of attacks are reported. Under each column corresponding to an evaluation metric, we bold the best results. And the row for the most robust agent is highlighted as \colorbox{lightgray}{gray}. Note that \textit{ATLA-PPO, PA-ATLA-PPO and RADIAL-PPO are trained with more than $2\times$ steps than \ourppo}, as reported in Table~\ref{efficiency}.}
\label{tab:mujoco_app}
\vspace{-0.5em}
\end{table*}

\textbf{MuJoCo Experiments}\quad We reported all results in Table \ref{tab:mujoco_app} including episode rewards of well-trained robust models under various adversarial attacks.  Under this full adversarial evaluation, we provide a robustness comparison between baselines and our algorithm from a comprehensive angle. We report the attack performance under a common chosen perturbation budget $\epsilon$ following \cite{zhang2020robust, zhang2021robust}. Results in all four MuJoCo environments show that our \ourppo is the most robust method. 
We emphasize that Table~\ref{tab:mujoco_app} reports the final performance of all robust training baselines after convergence, but some baselines takes much more steps than our \ourppo. Table~\ref{tab:mujoco_app_less} in Appendix~\ref{app:exp:add} compares all methods under the same number of training steps, where \ourppo outperforms baselines more significantly.

\textbf{Atari Experiments}\quad In Table \ref{tab:atari_app}, we present performance based on DQN on four Atari environments under 1/255 and 3/255 $\epsilon$ attack. Under $\epsilon$ of 1/255, our \ourdqn achieves competitive performance under PGD attacks and outperforms all baselines under MinBest and PA-AD attacks, which shows better robustness of \ourdqn under weaker attacks.\\
Based on vanilla A2C, we implement SA-A2C\citep{zhang2020robust} and PA-ATLA-A2C\citep{sun2021strongest} as robust baselines. We implement \ourac to compare with ATLA methods on Atari. In Table \ref{tab:a2c_app}, under any $\epsilon$ value, our WocaR-A2C outperforms other robust baselines across different attacks. We can conclude that our method considerably enhance more robustness than ATLA methods on Atari.
\begin{table*}[!h]
\centering
\renewcommand{\arraystretch}{1.2}
\resizebox{\textwidth}{!}{%
\setlength{\tabcolsep}{3pt}
\begin{tabular}{c|p{3.1cm}<{\centering}|p{2.1cm}<{\centering}| cc|cc|cc }
\toprule
\multirow{2}{*}{\textbf{Environment}}                  & \multirow{2}{*}{\textbf{Model}}          & \multirow{2}{*}{\begin{tabular}[c]{@{}c@{}}{\bf Natural}\\{\bf Reward}\end{tabular}} &  \multicolumn{2}{c|}{\textbf{PGD} (10 steps)} &  \multicolumn{2}{c|}{\textbf{MinBest}}  & \multicolumn{2}{c}{\textbf{PA-AD}}        \\
 & & & $\epsilon$=1/255 & $\epsilon$=3/255 & $\epsilon$=1/255 & $\epsilon$=3/255 & $\epsilon$=1/255 & $\epsilon$=3/255 \\
 \hline
\multirow{4}{*}{\textbf{Pong}} & DQN & 21.0 $\pm$ 0.0 & -21.0 $\pm$ 0.0  & -21.0 $\pm$ 0.0 & -7.4 $\pm$ 2.8 & -9.7 $\pm$ 4.0  & -18.2 $\pm$ 2.3 & -19.0 $\pm$ 2.2 \\
 & SA-DQN & 21.0 $\pm$ 0.0 & 21.0 $\pm$ 0.0 & 21.0 $\pm$ 0.0 & 21.0 $\pm$ 0.0 & 20.6 $\pm$ 3.5  & 20.4 $\pm$ 1.8 & 18.7 $\pm$ 2.6  \\
 & RADIAL-DQN & 21.0 $\pm$ 0.0 & 21.0 $\pm$ 0.0 & 21.0 $\pm$ 0.0 & 21.0 $\pm$ 0.0 & 19.5 $\pm$ 2.1 & 20.3 $\pm$ 2.5 & 13.2 $\pm$ 1.8  \\
 & \cellcolor{lightgray}{\textbf{\ourdqn (Ours)}} & \cellcolor{lightgray}{\textbf{21.0 $\pm$ 0.0}}  & \cellcolor{lightgray}{\textbf{21.0 $\pm$ 0.0}} & \cellcolor{lightgray}{\textbf{21.0 $\pm$ 0.0}} & 
 \cellcolor{lightgray}{\textbf{21.0 $\pm$ 0.0}} & 
 \cellcolor{lightgray}{\textbf{20.8 $\pm$ 3.3}} & 
 \cellcolor{lightgray}{\textbf{21.0 $\pm$ 0.2}} & 
 \cellcolor{lightgray}{\textbf{19.7 $\pm$ 2.4 }}\\
 \midrule
\multirow{4}{*}{\textbf{Freeway}} & DQN & \textbf{34.0 $\pm$ 0.1} & 0.0 $\pm$ 0.0 & 0.0 $\pm$ 0.0 & 9.5 $\pm$ 3.0 & 5.5 $\pm$ 1.8 & 9.3 $\pm$ 2.7 & 4.7 $\pm$ 2.9 \\
 & SA-DQN & 30.0 $\pm$ 0.0 & 30.0 $\pm$ 0.0  & 30.0 $\pm$ 0.0 & 27.2 $\pm$ 3.4 & 18.3 $\pm$ 3.0 & 20.1 $\pm$ 4.0 & 9.5 $\pm$ 3.8 \\
 & RADIAL-DQN & 33.1 $\pm$ 0.2 & \textbf{33.1 $\pm$ 0.2} & \textbf{33.2 $\pm$ 0.2} & 22.6 $\pm$ 3.3 & 16.4 $\pm$ 2.3 & 18.5 $\pm$ 4.2 & 10.8 $\pm$ 3.6 \\
 & \cellcolor{lightgray}{\textbf{\ourdqn (Ours)}} & \cellcolor{lightgray}{31.2 $\pm$ 0.4}  & \cellcolor{lightgray}{31.2 $\pm$ 0.5} & \cellcolor{lightgray}{31.4 $\pm$ 0.3} & \cellcolor{lightgray}{\textbf{29.6 $\pm$ 2.5}} & \cellcolor{lightgray}{\textbf{19.8 $\pm$ 3.8 }} & \cellcolor{lightgray}{\textbf{24.9 $\pm$ 3.7}} & \cellcolor{lightgray}{\textbf{12.3 $\pm$ 3.2}}\\
\midrule
\multirow{4}{*}{\textbf{BankHeist}} & DQN & \textbf{1308 $\pm$ 24} & 54 $\pm$ 20 & 0 $\pm$ 0 & 210 $\pm$ 79 & 119 $\pm$ 65 & 213 $\pm$ 111 & 102 $\pm$ 92 \\
 & SA-DQN & 1245 $\pm$ 14 & \textbf{1245 $\pm$ 10} & 1176 $\pm$ 63 & 1148 $\pm$ 36 & 1024 $\pm$ 31 & 1054 $\pm$ 11 & 489 $\pm$ 106 \\
 & RADIAL-DQN & 1178 $\pm$ 4 & 1178 $\pm$ 4 & 1176 $\pm$ 63 & 1049 $\pm$ 27 & 928 $\pm$ 113 & 1035 $\pm$ 46 & 508 $\pm$ 85 \\
 & \cellcolor{lightgray}{\textbf{\ourdqn (Ours)}} & \cellcolor{lightgray}{1220 $\pm$ 12}  & \cellcolor{lightgray}{1220 $\pm$ 3} & \cellcolor{lightgray}{\textbf{1214 $\pm$ 7}} & \cellcolor{lightgray}{\textbf{1192 $\pm$ 12}} & \cellcolor{lightgray}{\textbf{1045 $\pm$ 20}} & \cellcolor{lightgray}{\textbf{1096 $\pm$ 19}} & \cellcolor{lightgray}{\textbf{754 $\pm$ 102}}\\
 \midrule
 \multirow{4}{*}{\textbf{RoadRunner}} & DQN & \textbf{45527 $\pm$ 4894} & 0 $\pm$ 0 & 0 $\pm$ 0 & 14962 $\pm$ 6431 & 2985 $\pm$ 1440 & 842 $\pm$ 41 & 203 $\pm$ 65 \\
 & SA-DQN & 44638 $\pm$ 2367 & 43970 $\pm$ 975 & 20678 $\pm$ 1563 & 39736 $\pm$ 2315 & 4214 $\pm$ 2587 & 38432 $\pm$ 3574 & 5516 $\pm$ 4684 \\
 & RADIAL-DQN & 44675 $\pm$ 5854 &  \textbf{44605 $\pm$ 1094} & 38576 $\pm$ 1960 & 38060 $\pm$ 1799 & 8476 $\pm$ 3964 & 36310 $\pm$ 9149 & 1290 $\pm$ 4015 \\
 & \cellcolor{lightgray}{\textbf{\ourdqn (Ours)}} & \cellcolor{lightgray}{44156 $\pm$ 2279} & \cellcolor{lightgray}{44079 $\pm$ 2154} & \cellcolor{lightgray}{\textbf{38720 $\pm$ 1765}} & \cellcolor{lightgray}{\textbf{40758 $\pm$ 3369}} & \cellcolor{lightgray}{\textbf{10545 $\pm$ 2984}} & \cellcolor{lightgray}{\textbf{38954 $\pm$ 3647}} & \cellcolor{lightgray}{\textbf{8239 $\pm$ 2766}}\\
\bottomrule
\end{tabular}}
\caption{Average episode rewards $\pm$ standard deviation over 1000 episodes on baselines and \ourdqn on Pong, Freeway, BankHeist, and RoadRunner. Natural reward and rewards under different attacks with $\epsilon$ of 1/255 and 3/255 are reported. We bold the best results for each evaluation metric. And the row for the most robust agents on all environments are highlighted by \textcolor{gray}{gray}.}
\label{tab:atari_app}
\vspace{-0.5em}
\end{table*}

\begin{table*}[!h]
\centering
\renewcommand{\arraystretch}{1.2}
\resizebox{\textwidth}{!}{%
\setlength{\tabcolsep}{3pt}
\begin{tabular}{c|p{3cm}<{\centering}|p{2.1cm}<{\centering}| cc|cc|cc }
\toprule
\multirow{2}{*}{\textbf{Environment}}                  & \multirow{2}{*}{\textbf{Model}}          & \multirow{2}{*}{\begin{tabular}[c]{@{}c@{}}{\bf Natural}\\{\bf Reward}\end{tabular}} &  \multicolumn{2}{c|}{\textbf{PGD} (10 steps)} &  \multicolumn{2}{c|}{\textbf{MinBest}}  & \multicolumn{2}{c}{\textbf{PA-AD}}        \\
 & & & $\epsilon$=1/255 & $\epsilon$=3/255 & $\epsilon$=1/255 & $\epsilon$=3/255 & $\epsilon$=1/255 & $\epsilon$=3/255 \\
 \hline
\multirow{4}{*}{\textbf{BankHeist}} & A2C & \textbf{1228 $\pm$ 93} & 67 $\pm$ 14 & 0 $\pm$ 0 & 972 $\pm$ 99 & 697 $\pm$ 153  & 636 $\pm$ 74 & 314 $\pm$ 116 \\
 & SA-A2C & 1029 $\pm$ 152 & 1029 $\pm$ 156 & 976 $\pm$ 54 & 902 $\pm$ 89 & 786 $\pm$ 52  & 836 $\pm$ 70 & 644 $\pm$ 153  \\
 & PA-ATLA-A2C & 1076 $\pm$ 56 & 1075 $\pm$ 79 & 1013 $\pm$ 69 & 957 $\pm$ 78 & 842 $\pm$ 154 & 862 $\pm$ 106 & 757 $\pm$ 132 \\
 & \cellcolor{lightgray}{\textbf{\ourac (Ours)}} &
 \cellcolor{lightgray}{1089 $\pm$ 34}  & \cellcolor{lightgray}{\textbf{1089 $\pm$ 78}} & \cellcolor{lightgray}{\textbf{ 1035 $\pm$ 102 }} & \cellcolor{lightgray}{\textbf{1043 $\pm$ 29}} & \cellcolor{lightgray}{\textbf{937 $\pm$ 65}} & \cellcolor{lightgray}{\textbf{1004 $\pm$ 94}} & \cellcolor{lightgray}{\textbf{879 $\pm$ 128}}\\
\bottomrule
\end{tabular}}
\caption{Average episode rewards $\pm$ standard deviation over 1000 episodes on baselines and VaR-A2C on BankHeist. Natural reward and rewards under different attacks with $\epsilon$ of 1/255 and 3/255 are reported. We bold the best results for each evaluation metric. And the row for the most robust agents on all environments are highlighted by \textcolor{gray}{gray}.}
\label{tab:a2c_app}
\vspace{-0.5em}
\end{table*}

\subsection{Additional Evaluation and Ablation Studies}
\label{app:exp:abl}
\subsubsection{Robustness Evaluation Using Multiple $\epsilon$} 
\label{app:exp:eps}
To study how \ourppo performs under attacks with different value of $\epsilon$, Figure \ref{fig:eps} shows the evaluation of our algorithms under different $\epsilon$ attacks compared with the baselines in Hopper and Walker2d. We can conclude that our robustly trained model universally and significantly outperforms other robust agents considering various attack budget $\epsilon$.
\begin{figure*}[!t]
    \centering
    \vspace{-0.5em}
    \begin{subfigure}[t]{0.31\textwidth}
        \centering
        \resizebox{\textwidth}{!}{
\begin{tikzpicture}[scale=0.65]

\definecolor{color0}{rgb}{0.917647058823529,0.917647058823529,0.949019607843137}
\definecolor{color1}{rgb}{0.729411764705882,0.333333333333333,0.827450980392157}
\definecolor{color2}{rgb}{0.0980392156862745,0.0980392156862745,0.43921568627451}
\definecolor{color3}{rgb}{0.201253172212011,0.690792081537903,0.479667611892753}
\definecolor{color4}{rgb}{0.8,0.270588235294118,0}

\begin{axis}[
axis background/.style={fill=white},
axis line style={black},
legend cell align={left},
legend style={
  fill opacity=0.2,
  draw opacity=1,
  text opacity=1,
  at={(0.47,0.97)},
  anchor=north west,
  draw=none,
  fill=color0
},
tick align=outside,
x grid style={white},
xlabel={Epsilon of Robust Sarsa attacker},
xmajorgrids,
xmin=-0.01, xmax=0.21,
xtick style={color=white!15!black},
y grid style={white},
ylabel={Average episode rewards},
ymajorgrids,
xtick={-0.025,0,0.025,0.05,0.075,0.1,0.125,0.15,0.175,0.2,0.225},
xticklabels={, 0.00, ,0.05, ,0.10, ,0.15, ,0.20,},
ymin=400, ymax=5000,
ytick style={color=white!15!black}
]
\path [draw=color1, fill=color1, opacity=0.2, line width=0.12pt]
(axis cs:0,3707)
--(axis cs:0,3703)
--(axis cs:0.05,1302)
--(axis cs:0.075,1088)
--(axis cs:0.1,1056)
--(axis cs:0.15,835)
--(axis cs:0.2,741)
--(axis cs:0.2,777)
--(axis cs:0.2,777)
--(axis cs:0.15,877)
--(axis cs:0.1,1078)
--(axis cs:0.075,1172)
--(axis cs:0.05,2654)
--(axis cs:0,3707)
--cycle;

\path [draw=color2, fill=color2, opacity=0.2, line width=0.12pt]
(axis cs:0,3891)
--(axis cs:0,2691)
--(axis cs:0.05,2574)
--(axis cs:0.075,1626)
--(axis cs:0.1,1411)
--(axis cs:0.15,1086)
--(axis cs:0.2,640)
--(axis cs:0.2,2482)
--(axis cs:0.2,2482)
--(axis cs:0.15,2580)
--(axis cs:0.1,2727)
--(axis cs:0.075,2862)
--(axis cs:0.05,3792)
--(axis cs:0,3891)
--cycle;

\path [draw=color3, fill=color3, opacity=0.2, line width=0.12pt]
(axis cs:0,3686)
--(axis cs:0,3212)
--(axis cs:0.05,3103)
--(axis cs:0.075,2873)
--(axis cs:0.1,2443)
--(axis cs:0.15,2437)
--(axis cs:0.2,2006)
--(axis cs:0.2,2506)
--(axis cs:0.2,2506)
--(axis cs:0.15,2645)
--(axis cs:0.1,3109)
--(axis cs:0.075,3131)
--(axis cs:0.05,3447)
--(axis cs:0,3686)
--cycle;

\path [draw=color4, fill=color4, opacity=0.2, line width=0.12pt]
(axis cs:0,3715)
--(axis cs:0,3517)
--(axis cs:0.05,2954)
--(axis cs:0.075,2942)
--(axis cs:0.1,2616)
--(axis cs:0.15,2276)
--(axis cs:0.2,2114)
--(axis cs:0.2,2756)
--(axis cs:0.2,2756)
--(axis cs:0.15,2872)
--(axis cs:0.1,3364)
--(axis cs:0.075,3526)
--(axis cs:0.05,3854)
--(axis cs:0,3715)
--cycle;

\addplot [line width=0.48pt, color1, mark=square*, mark size=3, mark options={solid}]
table {%
0 3705
0.05 1978
0.075 1130
0.1 1067
0.15 856
0.2 759
};
\addlegendentry{SA-PPO}
\addplot [line width=0.48pt, color2, mark=*, mark size=3, mark options={solid}]
table {%
0 3291
0.05 3183
0.075 2244
0.1 2069
0.15 1833
0.2 1561
};
\addlegendentry{ATLA-PPO}
\addplot [line width=0.48pt, color3, mark=pentagon*, mark size=3, mark options={solid}]
table {%
0 3449
0.05 3275
0.075 3002
0.1 2776
0.15 2541
0.2 2256
};
\addlegendentry{PA-ATLA-PPO}
\addplot [line width=0.48pt, color4, mark=asterisk, mark size=3, mark options={solid}]
table {%
0 3616
0.05 3404
0.075 3277
0.1 2990
0.15 2574
0.2 2435
};
\addlegendentry{\ourppo (Ours)}
\end{axis}

\end{tikzpicture}}
        \vspace{-1.5em}
        \caption{\small{Hopper: Robust Sarsa}}
    \end{subfigure}
    \hfill
    \begin{subfigure}[t]{0.31\textwidth}
        \centering
        \resizebox{\textwidth}{!}{
\begin{tikzpicture}[scale=0.65]

\definecolor{color0}{rgb}{0.917647058823529,0.917647058823529,0.949019607843137}
\definecolor{color1}{rgb}{0.729411764705882,0.333333333333333,0.827450980392157}
\definecolor{color2}{rgb}{0.0980392156862745,0.0980392156862745,0.43921568627451}
\definecolor{color3}{rgb}{0.201253172212011,0.690792081537903,0.479667611892753}
\definecolor{color4}{rgb}{0.8,0.270588235294118,0}

\begin{axis}[
axis background/.style={fill=white},
axis line style={black},
legend cell align={left},
legend style={
  fill opacity=0.2,
  draw opacity=1,
  text opacity=1,
  at={(0.47, 0.97)},
  anchor=north west,
  draw=none,
  fill=color0
},
tick align=outside,
x grid style={white},
xlabel={Epsilon of SA-RL attacker},
xmajorgrids,
xmin=-0.01, xmax=0.21,
xtick style={color=white!15!black},
y grid style={white},
ylabel={Average episode rewards},
ymajorgrids,
xtick={-0.025,0,0.025,0.05,0.075,0.1,0.125,0.15,0.175,0.2,0.225},
xticklabels={, 0.00, ,0.05, ,0.10, ,0.15, ,0.20,},
ymin=0, ymax=5000,
ytick style={color=white!15!black}
]

\path [draw=color1, fill=color1, opacity=0.2, line width=0.12pt]
(axis cs:0,3707)
--(axis cs:0,3703)
--(axis cs:0.05,3534)
--(axis cs:0.075,685)
--(axis cs:0.1,435)
--(axis cs:0.15,295)
--(axis cs:0.2,586)
--(axis cs:0.2,978)
--(axis cs:0.2,978)
--(axis cs:0.15,1159)
--(axis cs:0.1,1549)
--(axis cs:0.075,1467)
--(axis cs:0.05,3540)
--(axis cs:0,3707)
--cycle;

\path [draw=color2, fill=color2, opacity=0.2, line width=0.12pt]
(axis cs:0,3891)
--(axis cs:0,2691)
--(axis cs:0.05,1183)
--(axis cs:0.075,970)
--(axis cs:0.1,431)
--(axis cs:0.15,109)
--(axis cs:0.2,141)
--(axis cs:0.2,1711)
--(axis cs:0.2,1711)
--(axis cs:0.15,1813)
--(axis cs:0.1,2105)
--(axis cs:0.075,2574)
--(axis cs:0.05,2573)
--(axis cs:0,3891)
--cycle;

\path [draw=color3, fill=color3, opacity=0.2, line width=0.12pt]
(axis cs:0,3686)
--(axis cs:0,3212)
--(axis cs:0.05,1138)
--(axis cs:0.075,1245)
--(axis cs:0.1,384)
--(axis cs:0.15,499)
--(axis cs:0.2,160)
--(axis cs:0.2,1624)
--(axis cs:0.2,1624)
--(axis cs:0.15,1269)
--(axis cs:0.1,1978)
--(axis cs:0.075,1813)
--(axis cs:0.05,2806)
--(axis cs:0,3686)
--cycle;

\path [draw=color4, fill=color4, opacity=0.2, line width=0.12pt]
(axis cs:0,3715)
--(axis cs:0,3517)
--(axis cs:0.05,2311)
--(axis cs:0.075,2686)
--(axis cs:0.1,2093)
--(axis cs:0.15,2353)
--(axis cs:0.2,1251)
--(axis cs:0.2,2021)
--(axis cs:0.2,2021)
--(axis cs:0.15,2827)
--(axis cs:0.1,2995)
--(axis cs:0.075,3294)
--(axis cs:0.05,3617)
--(axis cs:0,3715)
--cycle;

\addplot [line width=0.48pt, color1, mark=square*, mark size=3, mark options={solid}]
table {%
0 3705
0.05 3537
0.075 1076
0.1 992
0.15 727
0.2 782
};
\addlegendentry{SA-PPO}
\addplot [line width=0.48pt, color2, mark=*, mark size=3, mark options={solid}]
table {%
0 3291
0.05 1878
0.075 1772
0.1 1268
0.15 961
0.2 926
};
\addlegendentry{ATLA-PPO}
\addplot [line width=0.48pt, color3, mark=pentagon*, mark size=3, mark options={solid}]
table {%
0 3449
0.05 1972
0.075 1529
0.1 1181
0.15 884
0.2 892
};
\addlegendentry{PA-ATLA-PPO}
\addplot [line width=0.48pt, color4, mark=asterisk, mark size=3, mark options={solid}]
table {%
0 3616
0.05 2964
0.075 2990
0.1 2544
0.15 2590
0.2 1636
};
\addlegendentry{\ourppo (Ours)}
\end{axis}

\end{tikzpicture}}
        \vspace{-1.5em}
        \caption{\small{Hopper: SA-RL}}
    \end{subfigure}
    \hfill
    \begin{subfigure}[t]{0.31\textwidth}
        \centering
        \resizebox{\textwidth}{!}{
\begin{tikzpicture}[scale=0.65]

\definecolor{color0}{rgb}{0.917647058823529,0.917647058823529,0.949019607843137}
\definecolor{color1}{rgb}{0.729411764705882,0.333333333333333,0.827450980392157}
\definecolor{color2}{rgb}{0.0980392156862745,0.0980392156862745,0.43921568627451}
\definecolor{color3}{rgb}{0.201253172212011,0.690792081537903,0.479667611892753}
\definecolor{color4}{rgb}{0.8,0.270588235294118,0}

\begin{axis}[
axis background/.style={fill=white},
axis line style={black},
legend cell align={left},
legend style={
  fill opacity=0.2,
  draw opacity=1,
  text opacity=1,
  at={(0.47,0.97)},
  anchor=north west,
  draw=none,
  fill=color0
},
tick align=outside,
x grid style={white},
xlabel={Epsilon of PA-AD attacker},
xmajorgrids,
xmin=-0.01, xmax=0.21,
xtick style={color=white!15!black},
y grid style={white},
ylabel={Average episode rewards},
ymajorgrids,
xtick={-0.025,0,0.025,0.05,0.075,0.1,0.125,0.15,0.175,0.2,0.225},
xticklabels={, 0.00, ,0.05, ,0.10, ,0.15, ,0.20,},
ymin=0, ymax=4600,
ytick style={color=white!15!black}
]
\path [draw=color1, fill=color1, opacity=0.2, line width=0.12pt]
(axis cs:0,3707)
--(axis cs:0,3703)
--(axis cs:0.05,2369)
--(axis cs:0.075,835)
--(axis cs:0.1,393)
--(axis cs:0.15,246)
--(axis cs:0.2,170)
--(axis cs:0.2,1158)
--(axis cs:0.2,1158)
--(axis cs:0.15,1148)
--(axis cs:0.1,1191)
--(axis cs:0.075,877)
--(axis cs:0.05,2713)
--(axis cs:0,3707)
--cycle;

\path [draw=color2, fill=color2, opacity=0.2, line width=0.12pt]
(axis cs:0,3891)
--(axis cs:0,2691)
--(axis cs:0.05,1061)
--(axis cs:0.075,882)
--(axis cs:0.1,397)
--(axis cs:0.15,207)
--(axis cs:0.2,349)
--(axis cs:0.2,1043)
--(axis cs:0.2,1043)
--(axis cs:0.15,1231)
--(axis cs:0.1,1527)
--(axis cs:0.075,1582)
--(axis cs:0.05,1987)
--(axis cs:0,3891)
--cycle;

\path [draw=color3, fill=color3, opacity=0.2, line width=0.12pt]
(axis cs:0,3686)
--(axis cs:0,3212)
--(axis cs:0.05,2714)
--(axis cs:0.075,2196)
--(axis cs:0.1,1376)
--(axis cs:0.15,1330)
--(axis cs:0.2,1059)
--(axis cs:0.2,2115)
--(axis cs:0.2,2115)
--(axis cs:0.15,2200)
--(axis cs:0.1,2702)
--(axis cs:0.075,2846)
--(axis cs:0.05,3258)
--(axis cs:0,3686)
--cycle;

\path [draw=color4, fill=color4, opacity=0.2, line width=0.12pt]
(axis cs:0,3715)
--(axis cs:0,3517)
--(axis cs:0.05,2289)
--(axis cs:0.075,2334)
--(axis cs:0.1,1952)
--(axis cs:0.15,1855)
--(axis cs:0.2,1475)
--(axis cs:0.2,2179)
--(axis cs:0.2,2179)
--(axis cs:0.15,2419)
--(axis cs:0.1,2874)
--(axis cs:0.075,2824)
--(axis cs:0.05,3071)
--(axis cs:0,3715)
--cycle;

\addplot [line width=0.48pt, color1, mark=square*, mark size=3, mark options={solid}]
table {%
0 3705
0.05 2541
0.075 856
0.1 792
0.15 697
0.2 664
};
\addlegendentry{SA-PPO}
\addplot [line width=0.48pt, color2, mark=*, mark size=3, mark options={solid}]
table {%
0 3291
0.05 1524
0.075 1232
0.1 962
0.15 719
0.2 696
};
\addlegendentry{ATLA-PPO}
\addplot [line width=0.48pt, color3, mark=pentagon*, mark size=3, mark options={solid}]
table {%
0 3449
0.05 2986
0.075 2521
0.1 2039
0.15 1765
0.2 1587
};
\addlegendentry{PA-ATLA-PPO}
\addplot [line width=0.48pt, color4, mark=asterisk, mark size=3, mark options={solid}]
table {%
0 3616
0.05 2680
0.075 2579
0.1 2413
0.15 2137
0.2 1827
};
\addlegendentry{\ourppo (Ours)}
\end{axis}

\end{tikzpicture}}
        \vspace{-1.5em}
        \caption{\small{Hopper: PA-AD}}
    \end{subfigure}\\
    \begin{subfigure}[t]{0.31\textwidth}
        \centering
        \resizebox{\textwidth}{!}{
\begin{tikzpicture}[scale=0.65]

\definecolor{color0}{rgb}{0.917647058823529,0.917647058823529,0.949019607843137}
\definecolor{color1}{rgb}{0.729411764705882,0.333333333333333,0.827450980392157}
\definecolor{color2}{rgb}{0.0980392156862745,0.0980392156862745,0.43921568627451}
\definecolor{color3}{rgb}{0.201253172212011,0.690792081537903,0.479667611892753}
\definecolor{color4}{rgb}{0.8,0.270588235294118,0}

\begin{axis}[
axis background/.style={fill=white},
axis line style={black},
legend cell align={left},
legend style={
  fill opacity=0.2,
  draw opacity=1,
  text opacity=1,
  at={(0.47,0.97)},
  anchor=north west,
  draw=none,
  fill=color0
},
tick align=outside,
x grid style={white},
xlabel={Epsilon of Robust Sarsa attacker},
xmajorgrids,
xmin=-0.015, xmax=0.315,
xtick style={color=white!15!black},
y grid style={white},
ylabel={Average episode rewards},
ymajorgrids,
xtick={-0.05,0,0.05,0.1,0.15,0.2,0.25,0.3,0.35},
xticklabels={, 0.00, ,0.10, ,0.20,  ,0.30,  },
ymin=0, ymax=6500,
ytick style={color=white!15!black}
]
\path [draw=color1, fill=color1, opacity=0.2, line width=0.12pt]
(axis cs:0,4548)
--(axis cs:0,4426)
--(axis cs:0.05,3106)
--(axis cs:0.1,3349)
--(axis cs:0.15,3670)
--(axis cs:0.2,309)
--(axis cs:0.25,299)
--(axis cs:0.3,282)
--(axis cs:0.3,686)
--(axis cs:0.3,686)
--(axis cs:0.25,597)
--(axis cs:0.2,621)
--(axis cs:0.15,3946)
--(axis cs:0.1,4409)
--(axis cs:0.05,5334)
--(axis cs:0,4548)
--cycle;

\path [draw=color2, fill=color2, opacity=0.2, line width=0.12pt]
(axis cs:0,4317)
--(axis cs:0,3367)
--(axis cs:0.05,3238)
--(axis cs:0.1,2688)
--(axis cs:0.15,2945)
--(axis cs:0.2,683)
--(axis cs:0.25,676)
--(axis cs:0.3,424)
--(axis cs:0.3,758)
--(axis cs:0.3,758)
--(axis cs:0.25,864)
--(axis cs:0.2,1085)
--(axis cs:0.15,3533)
--(axis cs:0.1,4216)
--(axis cs:0.05,4592)
--(axis cs:0,4317)
--cycle;

\path [draw=color3, fill=color3, opacity=0.2, line width=0.12pt]
(axis cs:0,4707)
--(axis cs:0,3649)
--(axis cs:0.05,3833)
--(axis cs:0.1,3811)
--(axis cs:0.15,3659)
--(axis cs:0.2,1441)
--(axis cs:0.25,771)
--(axis cs:0.3,1005)
--(axis cs:0.3,1679)
--(axis cs:0.3,1679)
--(axis cs:0.25,2239)
--(axis cs:0.2,2505)
--(axis cs:0.15,4273)
--(axis cs:0.1,4163)
--(axis cs:0.05,4203)
--(axis cs:0,4707)
--cycle;

\path [draw=color4, fill=color4, opacity=0.2, line width=0.12pt]
(axis cs:0,4571)
--(axis cs:0,3581)
--(axis cs:0.05,4382)
--(axis cs:0.1,3727)
--(axis cs:0.15,3769)
--(axis cs:0.2,2192)
--(axis cs:0.25,1675)
--(axis cs:0.3,898)
--(axis cs:0.3,2586)
--(axis cs:0.3,2586)
--(axis cs:0.25,3159)
--(axis cs:0.2,3410)
--(axis cs:0.15,4401)
--(axis cs:0.1,4989)
--(axis cs:0.05,5012)
--(axis cs:0,4571)
--cycle;

\addplot [line width=0.48pt, color1, mark=square*, mark size=3, mark options={solid}]
table {%
0 4487
0.05 4220
0.1 3879
0.15 3808
0.2 465
0.25 448
0.3 484
};
\addlegendentry{SA-PPO}
\addplot [line width=0.48pt, color2, mark=*, mark size=3, mark options={solid}]
table {%
0 3842
0.05 3915
0.1 3452
0.15 3239
0.2 884
0.25 770
0.3 591
};
\addlegendentry{ATLA-PPO}
\addplot [line width=0.48pt, color3, mark=pentagon*, mark size=3, mark options={solid}]
table {%
0 4178
0.05 4018
0.1 3987
0.15 3966
0.2 1973
0.25 1505
0.3 1342
};
\addlegendentry{PA-ATLA-PPO}
\addplot [line width=0.48pt, color4, mark=asterisk, mark size=3, mark options={solid}]
table {%
0 4076
0.05 4697
0.1 4358
0.15 4093
0.2 2801
0.25 2417
0.3 1742
};
\addlegendentry{\ourppo (Ours)}
\end{axis}

\end{tikzpicture}}
        \vspace{-1.5em}
        \caption{\small{Walker2d: Robust Sarsa}}
    \end{subfigure}
    \hfill
    \begin{subfigure}[t]{0.31\textwidth}
        \centering
        \resizebox{\textwidth}{!}{
\begin{tikzpicture}[scale=0.65]

\definecolor{color0}{rgb}{0.917647058823529,0.917647058823529,0.949019607843137}
\definecolor{color1}{rgb}{0.729411764705882,0.333333333333333,0.827450980392157}
\definecolor{color2}{rgb}{0.0980392156862745,0.0980392156862745,0.43921568627451}
\definecolor{color3}{rgb}{0.201253172212011,0.690792081537903,0.479667611892753}
\definecolor{color4}{rgb}{0.8,0.270588235294118,0}

\begin{axis}[
axis background/.style={fill=white},
axis line style={black},
legend cell align={left},
legend style={
  fill opacity=0.2,
  draw opacity=1,
  text opacity=1,
  at={(0.47,0.97)},
  anchor=north west,
  draw=none,
  fill=color0
},
tick align=outside,
x grid style={white},
xlabel={Epsilon of SA-RL attacker},
xmajorgrids,
xmin=-0.015, xmax=0.315,
xtick style={color=white!15!black},
y grid style={white},
ylabel={Average episode rewards},
ymajorgrids,
xtick={-0.05,0,0.05,0.1,0.15,0.2,0.25,0.3,0.35},
xticklabels={, 0.00, ,0.10, ,0.20,  ,0.30,  },
ymin=0, ymax=6000,
ytick style={color=white!15!black}
]
\path [draw=color1, fill=color1, opacity=0.2, line width=0.12pt]
(axis cs:0,4548)
--(axis cs:0,4426)
--(axis cs:0.05,2125)
--(axis cs:0.1,2074)
--(axis cs:0.15,2572)
--(axis cs:0.2,984)
--(axis cs:0.25,894)
--(axis cs:0.3,849)
--(axis cs:0.3,1491)
--(axis cs:0.3,1491)
--(axis cs:0.25,1126)
--(axis cs:0.2,1292)
--(axis cs:0.15,3244)
--(axis cs:0.1,3842)
--(axis cs:0.05,4363)
--(axis cs:0,4548)
--cycle;

\path [draw=color2, fill=color2, opacity=0.2, line width=0.12pt]
(axis cs:0,4317)
--(axis cs:0,3367)
--(axis cs:0.05,3026)
--(axis cs:0.1,3254)
--(axis cs:0.15,2956)
--(axis cs:0.2,1830)
--(axis cs:0.25,1164)
--(axis cs:0.3,934)
--(axis cs:0.3,1520)
--(axis cs:0.3,1520)
--(axis cs:0.25,2996)
--(axis cs:0.2,3222)
--(axis cs:0.15,4370)
--(axis cs:0.1,4212)
--(axis cs:0.05,4510)
--(axis cs:0,4317)
--cycle;

\path [draw=color3, fill=color3, opacity=0.2, line width=0.12pt]
(axis cs:0,4707)
--(axis cs:0,3649)
--(axis cs:0.05,3215)
--(axis cs:0.1,3553)
--(axis cs:0.15,3272)
--(axis cs:0.2,2107)
--(axis cs:0.25,1580)
--(axis cs:0.3,533)
--(axis cs:0.3,1801)
--(axis cs:0.3,1801)
--(axis cs:0.25,2758)
--(axis cs:0.2,2811)
--(axis cs:0.15,3628)
--(axis cs:0.1,4149)
--(axis cs:0.05,4655)
--(axis cs:0,4707)
--cycle;

\path [draw=color4, fill=color4, opacity=0.2, line width=0.12pt]
(axis cs:0,4571)
--(axis cs:0,3581)
--(axis cs:0.05,3161)
--(axis cs:0.1,3256)
--(axis cs:0.15,3399)
--(axis cs:0.2,2312)
--(axis cs:0.25,1388)
--(axis cs:0.3,689)
--(axis cs:0.3,2021)
--(axis cs:0.3,2021)
--(axis cs:0.25,2872)
--(axis cs:0.2,3678)
--(axis cs:0.15,4125)
--(axis cs:0.1,4506)
--(axis cs:0.05,4705)
--(axis cs:0,4571)
--cycle;

\addplot [line width=0.48pt, color1, mark=square*, mark size=3, mark options={solid}]
table {%
0 4487
0.05 3244
0.1 2958
0.15 2908
0.2 1138
0.25 1010
0.3 1170
};
\addlegendentry{SA-PPO}
\addplot [line width=0.48pt, color2, mark=*, mark size=3, mark options={solid}]
table {%
0 3842
0.05 3768
0.1 3733
0.15 3663
0.2 2526
0.25 2080
0.3 1227
};
\addlegendentry{ATLA-PPO}
\addplot [line width=0.48pt, color3, mark=pentagon*, mark size=3, mark options={solid}]
table {%
0 4178
0.05 3935
0.1 3851
0.15 3450
0.2 2459
0.25 2169
0.3 1167
};
\addlegendentry{PA-ATLA-PPO}
\addplot [line width=0.48pt, color4, mark=asterisk, mark size=3, mark options={solid}]
table {%
0 4076
0.05 3933
0.1 3881
0.15 3770
0.2 2995
0.25 2130
0.3 1355
};
\addlegendentry{\ourppo (Ours)}
\end{axis}

\end{tikzpicture}}
        \vspace{-1.5em}
        \caption{\small{Walker2d: SA-RL}}
    \end{subfigure}
    \hfill
    \begin{subfigure}[t]{0.31\textwidth}
        \centering
        \resizebox{\textwidth}{!}{
\begin{tikzpicture}[scale=0.65]

\definecolor{color0}{rgb}{0.917647058823529,0.917647058823529,0.949019607843137}
\definecolor{color1}{rgb}{0.729411764705882,0.333333333333333,0.827450980392157}
\definecolor{color2}{rgb}{0.0980392156862745,0.0980392156862745,0.43921568627451}
\definecolor{color3}{rgb}{0.201253172212011,0.690792081537903,0.479667611892753}
\definecolor{color4}{rgb}{0.8,0.270588235294118,0}

\begin{axis}[
axis background/.style={fill=white},
axis line style={black},
legend cell align={left},
legend style={
  fill opacity=0.2,
  draw opacity=1,
  text opacity=1,
  at={(0.47,0.97)},
  anchor=north west,
  draw=none,
  fill=color0
},
tick align=outside,
x grid style={white},
xlabel={Epsilon of PA-AD attacker},
xmajorgrids,
xmin=-0.015, xmax=0.315,
xtick style={color=white!15!black},
y grid style={white},
ylabel={Average episode rewards},
ymajorgrids,
xtick={-0.05,0,0.05,0.1,0.15,0.2,0.25,0.3,0.35},
xticklabels={, 0.00, ,0.10, ,0.20,  ,0.30,  },
ymin=0, ymax=5000,
ytick style={color=white!15!black}
]
\path [draw=color1, fill=color1, opacity=0.2, line width=0.12pt]
(axis cs:0,4548)
--(axis cs:0,4426)
--(axis cs:0.05,2275)
--(axis cs:0.1,1383)
--(axis cs:0.15,689)
--(axis cs:0.2,192)
--(axis cs:0.25,368)
--(axis cs:0.3,171)
--(axis cs:0.3,803)
--(axis cs:0.3,803)
--(axis cs:0.25,616)
--(axis cs:0.2,1092)
--(axis cs:0.15,1395)
--(axis cs:0.1,1621)
--(axis cs:0.05,3863)
--(axis cs:0,4548)
--cycle;

\path [draw=color2, fill=color2, opacity=0.2, line width=0.12pt]
(axis cs:0,4317)
--(axis cs:0,3367)
--(axis cs:0.05,2971)
--(axis cs:0.1,3017)
--(axis cs:0.15,454)
--(axis cs:0.2,783)
--(axis cs:0.25,285)
--(axis cs:0.3,174)
--(axis cs:0.3,1168)
--(axis cs:0.3,1168)
--(axis cs:0.25,1463)
--(axis cs:0.2,1487)
--(axis cs:0.15,1994)
--(axis cs:0.1,4213)
--(axis cs:0.05,4379)
--(axis cs:0,4317)
--cycle;

\path [draw=color3, fill=color3, opacity=0.2, line width=0.12pt]
(axis cs:0,4707)
--(axis cs:0,3649)
--(axis cs:0.05,3270)
--(axis cs:0.1,3615)
--(axis cs:0.15,2117)
--(axis cs:0.2,1857)
--(axis cs:0.25,607)
--(axis cs:0.3,286)
--(axis cs:0.3,1092)
--(axis cs:0.3,1092)
--(axis cs:0.25,1531)
--(axis cs:0.2,2293)
--(axis cs:0.15,2379)
--(axis cs:0.1,4081)
--(axis cs:0.05,4340)
--(axis cs:0,4707)
--cycle;

\path [draw=color4, fill=color4, opacity=0.2, line width=0.12pt]
(axis cs:0,4571)
--(axis cs:0,3581)
--(axis cs:0.05,3244)
--(axis cs:0.1,3623)
--(axis cs:0.15,2370)
--(axis cs:0.2,2412)
--(axis cs:0.25,1156)
--(axis cs:0.3,88)
--(axis cs:0.3,1360)
--(axis cs:0.3,1360)
--(axis cs:0.25,1504)
--(axis cs:0.2,2978)
--(axis cs:0.15,3074)
--(axis cs:0.1,4141)
--(axis cs:0.05,4374)
--(axis cs:0,4571)
--cycle;

\addplot [line width=0.48pt, color1, mark=square*, mark size=3, mark options={solid}]
table {%
0 4487
0.05 3069
0.1 1502
0.15 1042
0.2 642
0.25 492
0.3 487
};
\addlegendentry{SA-PPO}
\addplot [line width=0.48pt, color2, mark=*, mark size=3, mark options={solid}]
table {%
0 3842
0.05 3675
0.1 3615
0.15 1224
0.2 1135
0.25 874
0.3 671
};
\addlegendentry{ATLA-PPO}
\addplot [line width=0.48pt, color3, mark=pentagon*, mark size=3, mark options={solid}]
table {%
0 4178
0.05 3805
0.1 3848
0.15 2248
0.2 2075
0.25 1069
0.3 689
};
\addlegendentry{PA-ATLA-PPO}
\addplot [line width=0.48pt, color4, mark=asterisk, mark size=3, mark options={solid}]
table {%
0 4076
0.05 3809
0.1 3882
0.15 2722
0.2 2695
0.25 1330
0.3 724
};
\addlegendentry{\ourppo(Ours)}
\end{axis}

\end{tikzpicture}}
        \vspace{-1.5em}
        \caption{\small{Walker2d: PA-AD}}
    \end{subfigure}
    \caption{\small{Comparisons under different attacks w.r.t. different budget $\epsilon$’s on Hopper and Walker2d.
    }}
    \label{fig:eps}
\end{figure*}
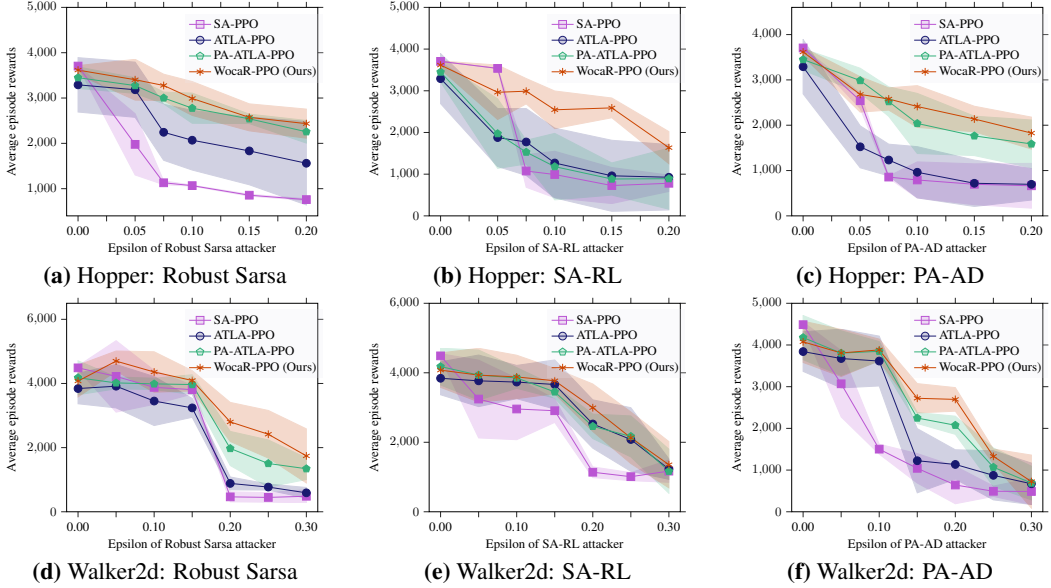

\subsubsection{Additional Evaluation on Sample Efficiency}
\label{app:exp:add}
\begin{table*}[!t]
\centering
\renewcommand{\arraystretch}{1.2}
\resizebox{\textwidth}{!}{%
\setlength{\tabcolsep}{4pt}
\begin{tabular}{c|c|p{2cm}<{\centering} p{2cm}<{\centering} p{2cm}<{\centering} p{2cm}<{\centering} p{2cm}<{\centering} p{2cm}<{\centering}}
\toprule
\textbf{Environment}                  & \textbf{Model}          & \begin{tabular}[c]{@{}c@{}}{\bf Natural}\\{\bf Reward}\end{tabular} & \textbf{Random}      & \textbf{MAD}          & \textbf{RS}          & \textbf{SA-RL}       & \textbf{PA-AD}        \\
 \hline
\multirow{4}{5.5em}{\begin{tabular}[c]{@{}c@{}}{\textbf{Halfcheetah}}\\state-dim: 17\\$\epsilon$=0.15\end{tabular}}    
 & ATLA-PPO            & 4817 $\pm$ 277     & 4809 $\pm$ 186  & 4584 $\pm$ 100  & 4074 $\pm$ 285  & 4129 $\pm$ 348  & 1856 $\pm$ 294  \\
& PA-ATLA-PPO      & 5023 $\pm$ 282     & 5076 $\pm$ 149  & 4720 $\pm$ 334  & 4392 $\pm$ 158  & 4159 $\pm$ 248  & 3085 $\pm$ 295  \\
& RADIAL-PPO & 4683 $\pm$ 97 & 4625 $\pm$ 190 & 3674 $\pm$ 222 & 3529 $\pm$ 173 & 2893 $\pm$ 165 & 2197 $\pm$ 251\\
& \cellcolor{lightgray}{\textbf{\ourppo (Ours)}} & \cellcolor{lightgray}{\textbf{6032 $\pm$ 68}}    & \cellcolor{lightgray}{\textbf{5969 $\pm$ 149}}  & \cellcolor{lightgray}{\textbf{5850 $\pm$ 228}}  & \cellcolor{lightgray}{\textbf{5319 $\pm$ 220}}  & \cellcolor{lightgray}{\textbf{5365 $\pm$ 54}}  & \cellcolor{lightgray}{\textbf{4269 $\pm$ 172}}  \\
\midrule
\multirow{4}{5.5em}{\begin{tabular}[c]{@{}c@{}}{\textbf{Hopper} }\\state-dim: 11\\$\epsilon$=0.075\end{tabular}}   
 & ATLA-PPO            & 3265 $\pm$ 342    & 3195 $\pm$ 275  & 2675 $\pm$ 332  & 2098 $\pm$ 398  & 1542 $\pm$ 639 & 1135 $\pm$ 289  \\
& PA-ATLA-PPO      & 3429 $\pm$ 196      & 3455 $\pm$ 315 & 3072 $\pm$ 478  & 2889 $\pm$ 258  & 1458 $\pm$ 274   & 2032 $\pm$ 244  \\
& RADIAL-PPO & \textbf{3687 $\pm$ 80} & 3627 $\pm$ 106 & 2952 $\pm$ 126 & 1094 $\pm$ 248 & 1243 $\pm$ 187 & 1036 $\pm$ 142\\
 & \cellcolor{lightgray}{\textbf{\ourppo (Ours)}} & \cellcolor{lightgray}{3616 $\pm$ 99}     & \cellcolor{lightgray}{\textbf{3633 $\pm$ 30}}   & \cellcolor{lightgray}{\textbf{3541 $\pm$ 207}}  & \cellcolor{lightgray}{\textbf{3277 $\pm$ 159}}  & \cellcolor{lightgray}{\textbf{2390 $\pm$ 145}}   & \cellcolor{lightgray}{\textbf{2579 $\pm$ 229}}  \\
 \midrule
\multirow{4}{5.5em}{\begin{tabular}[c]{@{}c@{}}{\textbf{Walker2d}}\\state-dim: 17\\$\epsilon$=0.05\end{tabular}}    
 & ATLA-PPO            & 2664 $\pm$ 366     & 2695 $\pm$ 320  & 2547 $\pm$   210 & 2439 $\pm$  174 & 2092 $\pm$ 144  & 1544 $\pm$ 280  \\
& PA-ATLA-PPO      & 3047 $\pm$ 223     & 3112 $\pm$ 111   & 2865 $\pm$ 230  & 2742 $\pm$ 177  & 2450 $\pm$ 229  & 1987 $\pm$ 246  \\
& RADIAL-PPO & 2143 $\pm$ 153 & 2231 $\pm$ 89 & 2095 $\pm$ 121 & 1680 $\pm$ 193 & 1078 $\pm$ 115 & 1274 $\pm$ 117\\
 & \cellcolor{lightgray}{\textbf{\ourppo (Ours)}} & \cellcolor{lightgray}{\textbf{4156 $\pm$ 495}}     & \cellcolor{lightgray}{\textbf{4244 $\pm$ 157}}  & \cellcolor{lightgray}{\textbf{4177 $\pm$ 176}}  & \cellcolor{lightgray}{\textbf{4093 $\pm$ 138}}  & \cellcolor{lightgray}{\textbf{3770 $\pm$ 196}}  & \cellcolor{lightgray}{\textbf{2722 $\pm$ 173}}  \\
 \midrule
 \multirow{4}{5.5em}{\begin{tabular}[c]{@{}c@{}}{\textbf{Ant}}\\state-dim: 111\\$\epsilon$=0.15\end{tabular}}    
 & ATLA-PPO            & 4249 $\pm$ 243     & 4218 $\pm$ 161  & 4036 $\pm$ 173  & 3391 $\pm$ 158  & 2045 $\pm$ 203  & -349 $\pm$ 175  \\
& PA-ATLA-PPO      & 4533 $\pm$ 238     & 4492 $\pm$ 190  & 4232 $\pm$ 203 & 3579 $\pm$ 261  & 2762 $\pm$ 152  & 1765 $\pm$ 185  \\
& RADIAL-PPO & 4379 $\pm$ 230 & 4194 $\pm$ 52 & 3278 $\pm$ 138 & 2348 $\pm$ 232 & 1380 $\pm$ 145 & 157 $\pm$ 124\\
 & \cellcolor{lightgray}{\textbf{\ourppo (Ours)}} & \cellcolor{lightgray}{\textbf{5596 $\pm$ 225}}     & \cellcolor{lightgray}{\textbf{5558 $\pm$ 241}}  & \cellcolor{lightgray}{\textbf{5284 $\pm$ 182}}  & \cellcolor{lightgray}{\textbf{4339 $\pm$ 160}}  & \cellcolor{lightgray}{\textbf{3822 $\pm$ 185}}  & \cellcolor{lightgray}{\textbf{3164 $\pm$ 163}}  \\
 
\bottomrule
\end{tabular}}
\caption{Average episode rewards $\pm$ standard deviation over 50 episodes on baselines and \ourppo trained for 2 million steps on Hopper, Walker2d, Halfcheetah and 7.5 million steps on Ant (less than the best settings). \textbf{Bold} numbers indicate the best results under each attack. The \colorbox{lightgray}{gray} rows are the most robust agents.}
\label{tab:mujoco_app_less}
\vspace{-0.5em}
\end{table*}

In Table~\ref{tab:mujoco_app_less}, we report the performance of \ourppo and all robust PPO baselines using the same training steps. 
We find that under limited training steps, ATLA-PPO, PA-ATLA-PPO and RADIAL-PPO obtain sub-optimal robustness, which suggests that these methods are more sample-hungry. In contrast, \ourppo converges under fewer steps and achieves best performance with a large advantage, which shows the higher efficiency of \ourppo.

\subsubsection{Additional Results of Time Efficiency}
\label{app:exp:eff}
\begin{table*}[!t]
\centering
\renewcommand{\arraystretch}{1.1}
\resizebox{0.7\textwidth}{!}{%
\setlength{\tabcolsep}{3pt}
\begin{tabular}{c|p{1.7cm}<{\centering} p{1.7cm}<{\centering} | p{1.7cm}<{\centering} p{1.7cm}<{\centering} }
\toprule
\multirow{2}{*}{Model} & \multicolumn{2}{c|}{Hopper} & \multicolumn{2}{c}{Ant} \\
\cline{2-5}
& Time (h) & Steps(m) & Time (h) & Steps (m)\\
\hline
SA-PPO & 3.0  & 2.0   & 8.9               & 10.0                   \\
ATLA-PPO   & 5.6    & 5.0     &     12.8           & 10.0                  \\
PA-ATLA-PPO   & 5.2               & 5.0      & 12.3                & 10.0                \\
RADIAL-PPO   & 3.2              & 4.0      & 10.2            & 10.0                \\
\cellcolor{lightgray}{\textbf{\ourppo (Ours)}} & \cellcolor{lightgray}{\textbf{2.3}}              & \cellcolor{lightgray}{\textbf{2.0}}      & \cellcolor{lightgray}{\textbf{8.7}}             & \cellcolor{lightgray}{\textbf{7.5}}    \\
\bottomrule
\end{tabular}}
\caption{Efficiency comparison of state-of-the-art robust training methods and \ourppo in Hopper and Ant. For Walker2d and Halfcheetah, the sampling steps are same as for Hopper and the training time is also extremely similar. We highlight the most efficient method as \colorbox{lightgray}{gray}.}
\label{efficiency}
\end{table*}

We show the training efficiency of \ourppo from three aspects including time, training iterations, and sampling in MuJoCo environments by comparing with SA-PPO and state-of-the-art methods ATLA-PPO, PA-ATLA-PPO, and RADIAL-PPO in Table \ref{efficiency}. For a fair comparison, we use the same {\it GeForce RTX 1080 Ti GPUs} to train all the robust agents. \\
It needs to mention that in continuous action spaces when estimating the worst-case value, we solve $\min_{\hat{a}\in\aadvaction} \worstcritic(s_{t+1},\hat{a})$ using 50-step gradient descent. The running time of this 50-step gradient descent is about \textbf{1.68 seconds} per batch with batch size 128. In total, this gradient descent computation takes 18\% of the total training time, thus it is not the computation bottleneck. 

Without training with an adversary, \textit{our algorithm requires much less (only 50\% or 75\%) steps to reliably converge.} \ourppo only takes less than half of time for low-dimensional environments to converge compared to ATLA methods and RADIAL-PPO. In high-dimensional environments like Ant, we only need 4 hours for training, while ATLA methods require at least 7 hours. When solving harder tasks, the efficiency advantage of \ourppo is more obvious.

\subsubsection{Effectiveness of Worst-attack Policy Optimization}
\label{app:exp:curve}
\begin{figure*}[!t]
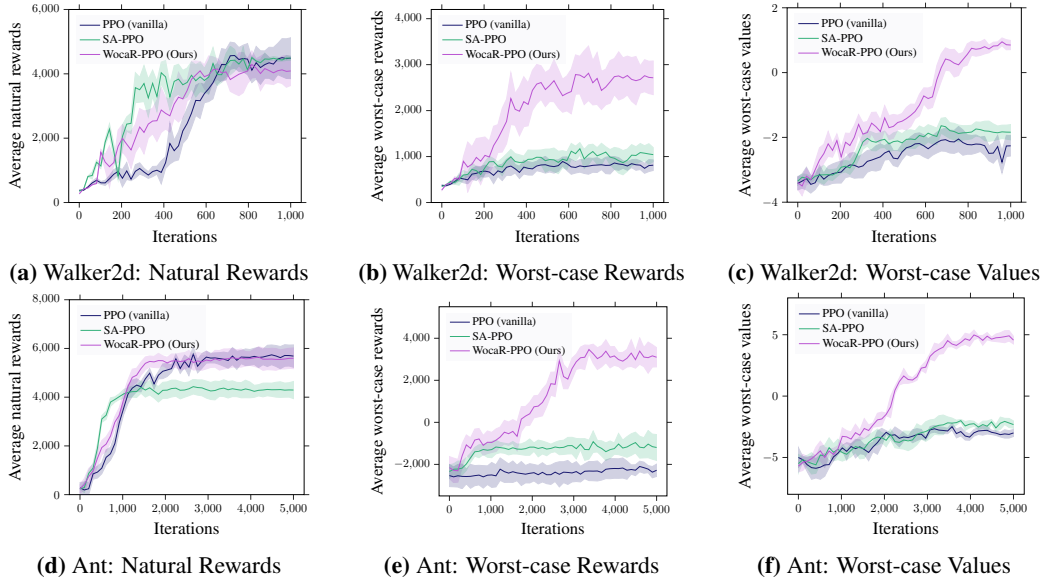

    \centering
    \begin{subfigure}[t]{0.31\textwidth}
        \centering
        \resizebox{0.95\textwidth}{!}{\input{\fighome/curve/walker-natural.tex}}
        \caption{\small{Walker2d: Natural Rewards}}
        \label{sfig:curves_walker_na}
    \end{subfigure}
    \hfill
    \begin{subfigure}[t]{0.31\textwidth}
        \centering
        \resizebox{0.95\textwidth}{!}{\input{\fighome/curve/walker-worst.tex}}
        \caption{\small{Walker2d: Worst-case Rewards}}
        \label{sfig:curves_walker_reward}
    \end{subfigure}
    \hfill
    \begin{subfigure}[t]{0.31\textwidth}
        \centering
        \resizebox{0.92\textwidth}{!}{\input{\fighome/curve/walker-value.tex}}
        \caption{\small{Walker2d: Worst-case Values}}
        \label{sfig:curves_walker_value}
    \end{subfigure}\\
    \begin{subfigure}[t]{0.31\textwidth}
        \centering
        \resizebox{0.95\textwidth}{!}{\input{\fighome/curve/ant-natural.tex}}
        \caption{\small{Ant: Natural Rewards}}
        \label{sfig:curves_ant_na}
    \end{subfigure}
    \hfill
    \begin{subfigure}[t]{0.31\textwidth}
        \centering
        \resizebox{0.95\textwidth}{!}{\input{\fighome/curve/ant-worst.tex}}
        \caption{\small{Ant: Worst-case Rewards}}
        \label{sfig:curves_ant_reward}
    \end{subfigure}
    \hfill
    \begin{subfigure}[t]{0.31\textwidth}
        \centering
        \resizebox{0.92\textwidth}{!}{\input{\fighome/curve/ant-value.tex}}
        \caption{\small{Ant: Worst-case Values}}
        \label{sfig:curves_ant_value}
    \end{subfigure}
    \caption{Learning curves (mean ± standard deviation) of natural rewards, worst-case rewards under attacks and estimated worst-case values during training on Walker2d and Ant for vanilla PPO (blue), SA-PPO (green) and \ourppo (purple).} 
    \label{fig:curve}
\end{figure*}

In addition to Figure~\ref{fig:curves_less}, we show the learning curves in Walker2d and Ant in Figure \ref{fig:curve} to verify the effectiveness of \estname and \worstname. Figure \ref{fig:curve}(a) and (d) show the natural rewards of agents during training without attacks. 
The actual worst-attack rewards in Figure \ref{fig:curve}(b) and (e) refer to the the reward obtained by the agents under PA-AD attack~\citep{sun2021strongest} which is the existing strongest attacking algorithm. To study the worst-case performance during training, We evaluate PPO, SA-PPO and \ourppo agents after every 20 iterations using all types of attacks and report the worst-case rewards for each checkpoint. 
We also present the trend of the estimated worst-case values during training in Figure \ref{fig:curve}(c) and (f), which are tested by the trained worst-attack value functions $\worstcritic$. \\
We observe from the curves that our \worstcriticname estimation matches the trend of actual worst-attack rewards. Also, the increases of estimated worst-attack values and actual worst-attack rewards of \ourppo show that our \ours significantly improves the robustness of agents by enhancing worst-attack values.

\subsubsection{Trade-off between Natural Performance and Robustness}
\label{app:exp:prefer}

As mentioned in Section~\ref{sec:result_effect}, the adjustable weight $\weightworst$ controls the trade-off between natural performance and robustness.
To discuss the effect of $\weightworst$, we train agents using \ourppo in Hopper, Walker2d, and Halfcheetah with uniformly sampled 40 different values of weight $\weightworst$ in range $(0,1]$. \\
Figure~\ref{app:fig:weight} plots the worst-case performance and natural performance of robust training baselines and 10 agents trained by \ourppo with various values of $\kappa_{wst}$. We can see that when reward under worst-case perturbations increases, it leads to a reduction of the natural reward. 

The choice of the worst-case value's weight $\kappa_{wst}$ is to control the trade-off between the final natural performance and robustness. It does not affect the convergence of the algorithm. When we increase the weight of worst-case values $\kappa_{wst}$, the reward under worst-case perturbations increases, but it leads to a reduction of the natural reward. Equally, when $\kappa_{wst}$ is set close to 0, the algorithm is similar to standard training, where the policy achieves high reward under no attack, but extremely low reward under attacks. Hence, $\kappa_{wst}$ is necessary for our algorithm to balance these two kinds of performance. In practice, one can adjust $\kappa_{wst}$ according to their preferences to robustness and natural performance.

We report the results in Table \ref{tab:mujoco_app} with significant better worst-case robustness and comparable natural performance compared with baselines. \ourppo can always find policies which dominate other robust agents.

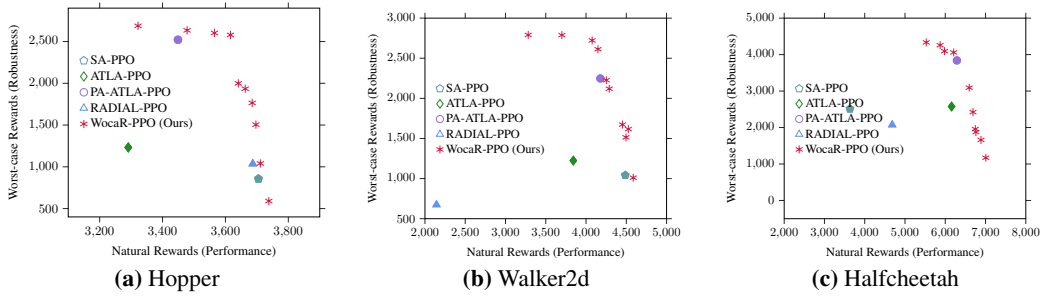
\begin{figure*}[!t]
    \centering
    \vspace{0.5em}
    \begin{subfigure}[t]{0.31\textwidth}
        \centering
        \resizebox{\textwidth}{!}{
\begin{tikzpicture}[scale=0.65]

\definecolor{color0}{rgb}{0.917647058823529,0.917647058823529,0.949019607843137}
\definecolor{color1}{rgb}{0.372549019607843,0.619607843137255,0.627450980392157}
\definecolor{color2}{rgb}{0.133333333333333,0.545098039215686,0.133333333333333}
\definecolor{color3}{rgb}{0.576470588235294,0.43921568627451,0.858823529411765}
\definecolor{color4}{rgb}{0.392156862745098,0.584313725490196,0.929411764705882}
\definecolor{color5}{rgb}{0.862745098039216,0.0784313725490196,0.235294117647059}

\begin{axis}[
axis background/.style={fill=white},
axis line style={black},
legend cell align={left},
legend style={
  fill opacity=0,
  draw opacity=1,
  text opacity=1,
  at={(0.03,0.8)},
  anchor=north west,
  draw=none,
  fill=white
},
tick align=outside,
x grid style={color=white},
xlabel={Natural Rewards (Performance)},
xmajorgrids,
xmin=3100, xmax=3900,
xtick style={color=white!15!black},
y grid style={color=white},
ylabel={Worst-case Rewards (Robustness)},
ymajorgrids,
ymin=400, ymax=2900,
ytick style={color=white!15!black}
]
\addplot [draw=color1, fill=color1, mark=pentagon*, mark size=3.2, mark options={solid}, only marks]
table{%
x  y
3705 856
3705 856
};
\addlegendentry{SA-PPO}
\addplot [draw=color2, fill=color2, mark=diamond*, mark size=3.5, mark options={solid}, only marks]
table{%
x  y
3291 1232
3291 1232
};
\addlegendentry{ATLA-PPO}
\addplot [draw=color3, fill=color3, mark=*, mark size=2.9, mark options={solid}, only marks]
table{%
x  y
3449 2521
3449 2521
};
\addlegendentry{PA-ATLA-PPO}
\addplot [draw=color4, fill=color4, mark=triangle*, mark size=3.3, only marks]
table{%
x  y
3687 1036
3687 1036
};
\addlegendentry{RADIAL-PPO}
\addplot [draw=color5, fill=color5, mark=asterisk, mark size=3.2, only marks]
table{%
x  y
3616 2579
3322 2688
3478 2634
3565 2603
3641 1999
3663 1935
3685 1765
3697 1505
3711 1040
3738 590
};
\addlegendentry{\ourppo (Ours)}
\end{axis}

\end{tikzpicture}}
        \vspace{-1.5em}
        \caption{\small{Hopper}}
    \end{subfigure}
    \hfill
    \begin{subfigure}[t]{0.31\textwidth}
        \centering
        \resizebox{\textwidth}{!}{
\begin{tikzpicture}[scale=0.65]

\definecolor{color0}{rgb}{0.917647058823529,0.917647058823529,0.949019607843137}
\definecolor{color1}{rgb}{0.372549019607843,0.619607843137255,0.627450980392157}
\definecolor{color2}{rgb}{0.133333333333333,0.545098039215686,0.133333333333333}
\definecolor{color3}{rgb}{0.576470588235294,0.43921568627451,0.858823529411765}
\definecolor{color4}{rgb}{0.392156862745098,0.584313725490196,0.929411764705882}
\definecolor{color5}{rgb}{0.862745098039216,0.0784313725490196,0.235294117647059}

\begin{axis}[
axis background/.style={fill=white},
axis line style={black},
legend cell align={left},
legend style={
  fill opacity=0,
  draw opacity=1,
  text opacity=1,
  at={(0.03,0.7)},
  anchor=north west,
  draw=none,
  fill=white
},
tick align=outside,
x grid style={white},
xlabel={Natural Rewards (Performance)},
xmajorgrids,
xmin=2000, xmax=5000,
xtick style={color=white!15!black},
y grid style={white},
ylabel={Worst-case Rewards (Robustness)},
ymajorgrids,
ymin=500, ymax=3000,
ytick style={color=white!15!black}
]
\addplot [draw=color1, fill=color1, mark=pentagon*, mark size=3.2, mark options={solid}, only marks]
table{%
x  y
4487 1042
4487 1042
};
\addlegendentry{SA-PPO}
\addplot [draw=color2, fill=color2, mark=diamond*, mark size=3.5, mark options={solid}, only marks]
table{%
x  y
3842 1224
3842 1224
};
\addlegendentry{ATLA-PPO}
\addplot [draw=color3, fill=color3, mark=*, mark size=2.9, mark options={solid}, only marks]
table{%
x  y
4178 2248
4178 2248
};
\addlegendentry{PA-ATLA-PPO}
\addplot [draw=color4, fill=color4, mark=triangle*, mark size=3.3, only marks]
table{%
x  y
2143 674
2143 674
};
\addlegendentry{RADIAL-PPO}
\addplot [draw=color5, fill=color5, mark=asterisk, mark size=3.2, only marks]
table{%
x  y
4076 2722
4252 2226
4452 1673
4148 2612
4525 1617
3282 2791
4586 1011
4288 2120
3700 2788
4496 1514
};
\addlegendentry{\ourppo (Ours)}
\end{axis}

\end{tikzpicture}}
        \vspace{-1.5em}
        \caption{\small{Walker2d}}
    \end{subfigure}
    \hfill
    \begin{subfigure}[t]{0.31\textwidth}
        \centering
        \resizebox{\textwidth}{!}{
\begin{tikzpicture}[scale=0.65]

\definecolor{color0}{rgb}{0.917647058823529,0.917647058823529,0.949019607843137}
\definecolor{color1}{rgb}{0.372549019607843,0.619607843137255,0.627450980392157}
\definecolor{color2}{rgb}{0.133333333333333,0.545098039215686,0.133333333333333}
\definecolor{color3}{rgb}{0.576470588235294,0.43921568627451,0.858823529411765}
\definecolor{color4}{rgb}{0.392156862745098,0.584313725490196,0.929411764705882}
\definecolor{color5}{rgb}{0.862745098039216,0.0784313725490196,0.235294117647059}
\begin{axis}[
axis background/.style={fill=white},
axis line style={black},
legend cell align={left},
legend style={
  fill opacity=0,
  draw opacity=1,
  text opacity=1,
  at={(0.03,0.7)},
  anchor=north west,
  draw=none,
  fill=white
},
tick align=outside,
x grid style={color=white},
xlabel={Natural Rewards (Performance)},
xmajorgrids,
xmin=2000, xmax=8000,
xtick style={color=white!15!black},
y grid style={color=white},
ylabel={Worst-case Rewards (Robustness)},
ymajorgrids,
ymin=-500, ymax=5000,
ytick style={color=white!15!black}
]
\addplot [draw=color1, fill=color1, mark=pentagon*, mark size=3.2, mark options={solid}, only marks]
table{%
x  y
3632 2512
3632 2512
};
\addlegendentry{SA-PPO}
\addplot [draw=color2, fill=color2, mark=diamond*, mark size=3.5, mark options={solid}, only marks]
table{%
x  y
6157 2576
6157 2576
};
\addlegendentry{ATLA-PPO}
\addplot [draw=color3, fill=color3, mark=*, mark size=2.9, mark options={solid}, only marks]
table{%
x  y
6289 3840
6289 3840
};
\addlegendentry{PA-ATLA-PPO}
\addplot [draw=color4, fill=color4, mark=triangle*, mark size=3.3, only marks]
table{%
x  y
4683 2074
4683 2074
};
\addlegendentry{RADIAL-PPO}
\addplot [draw=color5, fill=color5, mark=asterisk, mark size=3.2, only marks]
table{%
x  y
5872 4258
5982 4095
6595 3094
6688 2422 
5527 4334
6746 1958
6761 1875 
6209 4059
6887 1660
7008 1172
};
\addlegendentry{\ourppo (Ours)}
\end{axis}

\end{tikzpicture}}
        \vspace{-1.5em}
        \caption{\small{Halfcheetah}}
    \end{subfigure}
    \vspace{-0.5em}
    \caption{\small{Average natural rewards and worst-case rewards of \ourppo with different $\weightworst$ and other baselines on Hopper, Walker2d, and Halfcheetah.}}
    \label{app:fig:weight}
\end{figure*}

\subsubsection{Additional Ablation Studies}
\label{app:exp:reg}
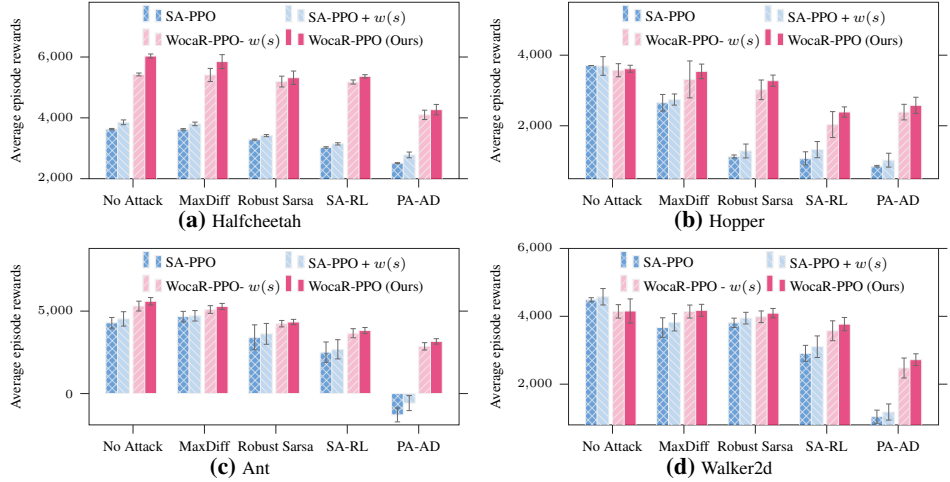
\begin{figure*}[!t]
    \centering
    \begin{subfigure}[t]{0.45\textwidth}
        \centering
        \resizebox{\textwidth}{!}{
\begin{tikzpicture}[scale=0.6]

\definecolor{color0}{rgb}{0.168966293897257,0.475610408705207,0.783522853775186}
\definecolor{color1}{rgb}{0.691639317416794,0.806706471648465,0.92224956246861}
\definecolor{color2}{rgb}{0.95715994546835,0.680184738372337,0.77629698219994}
\definecolor{color3}{rgb}{0.885434799735519,0.144732659736887,0.401761241571282}

\begin{axis}[
height=3cm,
width=\textwidth,
scale only axis,
axis line style={white!15!black},
legend cell align={left},
legend columns=2,
legend style={
  fill opacity=0,
  draw opacity=1,
  text opacity=1,
  at={(0.53,1.01)},
  anchor=north,
  draw=none,
  font=\scriptsize
},
xtick pos = lower,
tick align=outside,
x grid style={white!80!black},
xmin=-0.636, xmax=9.726,
xtick style={color=white!15!black},
xtick={0.595,2.595,4.595,6.595,8.595},
xticklabels={\scriptsize{No Attack},\scriptsize{MaxDiff},\scriptsize{Robust Sarsa},\scriptsize{SA-RL},\scriptsize{PA-AD}},
y grid style={white!80!black},
ylabel={\scriptsize{Average episode rewards}},
ymin=2000, ymax=7800,
ytick style={color=white!15!black},
y grid style={white},
tick label style={font=\tiny},
ymajorgrids,
]
\draw[draw=white!93.3333333333333!black,fill=color0,opacity=0.7,very thin,postaction={pattern=crosshatch, pattern color=white!93.3333333333333!black, fill opacity=0.7}] (axis cs:-0.165,0) rectangle (axis cs:0.165,3632);
\addlegendimage{ybar,ybar legend,draw=white!93.3333333333333!black,fill=color0,opacity=0.7,very thin,postaction={pattern=crosshatch, pattern color=white!93.3333333333333!black, fill opacity=0.7}}
\addlegendentry{SA-PPO}

\draw[draw=white!93.3333333333333!black,fill=color0,opacity=0.7,very thin,postaction={pattern=crosshatch, pattern color=white!93.3333333333333!black, fill opacity=0.7}] (axis cs:1.835,0) rectangle (axis cs:2.165,3624);
\draw[draw=white!93.3333333333333!black,fill=color0,opacity=0.7,very thin,postaction={pattern=crosshatch, pattern color=white!93.3333333333333!black, fill opacity=0.7}] (axis cs:3.835,0) rectangle (axis cs:4.165,3283);
\draw[draw=white!93.3333333333333!black,fill=color0,opacity=0.7,very thin,postaction={pattern=crosshatch, pattern color=white!93.3333333333333!black, fill opacity=0.7}] (axis cs:5.835,0) rectangle (axis cs:6.165,3028);
\draw[draw=white!93.3333333333333!black,fill=color0,opacity=0.7,very thin,postaction={pattern=crosshatch, pattern color=white!93.3333333333333!black, fill opacity=0.7}] (axis cs:7.835,0) rectangle (axis cs:8.165,2512);
\draw[draw=white!93.3333333333333!black,fill=color1,opacity=0.8,very thin,postaction={pattern=north west lines, pattern color=white!93.3333333333333!black, fill opacity=0.8}] (axis cs:0.165,0) rectangle (axis cs:0.495,3854);
\addlegendimage{ybar,ybar legend,draw=white!93.3333333333333!black,fill=color1,opacity=0.8,very thin,postaction={pattern=north west lines, pattern color=white!93.3333333333333!black, fill opacity=0.8}}
\addlegendentry{SA-PPO + $w(s)$}

\draw[draw=white!93.3333333333333!black,fill=color1,opacity=0.8,very thin,postaction={pattern=north west lines, pattern color=white!93.3333333333333!black, fill opacity=0.8}] (axis cs:2.165,0) rectangle (axis cs:2.495,3802);
\draw[draw=white!93.3333333333333!black,fill=color1,opacity=0.8,very thin,postaction={pattern=north west lines, pattern color=white!93.3333333333333!black, fill opacity=0.8}] (axis cs:4.165,0) rectangle (axis cs:4.495,3416);
\draw[draw=white!93.3333333333333!black,fill=color1,opacity=0.8,very thin,postaction={pattern=north west lines, pattern color=white!93.3333333333333!black, fill opacity=0.8}] (axis cs:6.165,0) rectangle (axis cs:6.495,3145);
\draw[draw=white!93.3333333333333!black,fill=color1,opacity=0.8,very thin,postaction={pattern=north west lines, pattern color=white!93.3333333333333!black, fill opacity=0.8}] (axis cs:8.165,0) rectangle (axis cs:8.495,2782);
\draw[draw=white!93.3333333333333!black,fill=color2,opacity=0.8,very thin,postaction={pattern=north east lines, pattern color=white!93.3333333333333!black, fill opacity=0.8}] (axis cs:0.595,0) rectangle (axis cs:0.925,5426);
\addlegendimage{ybar,ybar legend,draw=white!93.3333333333333!black,fill=color2,opacity=0.8,very thin,postaction={pattern=north east lines, pattern color=white!93.3333333333333!black, fill opacity=0.8}}
\addlegendentry{\ourppo - $w(s)$}

\draw[draw=white!93.3333333333333!black,fill=color2,opacity=0.8,very thin,postaction={pattern=north east lines, pattern color=white!93.3333333333333!black, fill opacity=0.8}] (axis cs:2.595,0) rectangle (axis cs:2.925,5409);
\draw[draw=white!93.3333333333333!black,fill=color2,opacity=0.8,very thin,postaction={pattern=north east lines, pattern color=white!93.3333333333333!black, fill opacity=0.8}] (axis cs:4.595,0) rectangle (axis cs:4.925,5193);
\draw[draw=white!93.3333333333333!black,fill=color2,opacity=0.8,very thin,postaction={pattern=north east lines, pattern color=white!93.3333333333333!black, fill opacity=0.8}] (axis cs:6.595,0) rectangle (axis cs:6.925,5178);
\draw[draw=white!93.3333333333333!black,fill=color2,opacity=0.8,very thin,postaction={pattern=north east lines, pattern color=white!93.3333333333333!black, fill opacity=0.8}] (axis cs:8.595,0) rectangle (axis cs:8.925,4098);
\draw[draw=white!93.3333333333333!black,fill=color3,opacity=0.8,very thin] (axis cs:0.925,0) rectangle (axis cs:1.255,6032);
\addlegendimage{ybar,ybar legend,draw=white!93.3333333333333!black,fill=color3,opacity=0.8,very thin}
\addlegendentry{\ourppo (Ours)}

\draw[draw=white!93.3333333333333!black,fill=color3,opacity=0.8,very thin] (axis cs:2.925,0) rectangle (axis cs:3.255,5850);
\draw[draw=white!93.3333333333333!black,fill=color3,opacity=0.8,very thin] (axis cs:4.925,0) rectangle (axis cs:5.255,5319);
\draw[draw=white!93.3333333333333!black,fill=color3,opacity=0.8,very thin] (axis cs:6.925,0) rectangle (axis cs:7.255,5365);
\draw[draw=white!93.3333333333333!black,fill=color3,opacity=0.8,very thin] (axis cs:8.925,0) rectangle (axis cs:9.255,4269);
\path [draw=white!41.1764705882353!black, line width=0.48pt]
(axis cs:0,3612)
--(axis cs:0,3652);

\path [draw=white!41.1764705882353!black, line width=0.48pt]
(axis cs:2,3601)
--(axis cs:2,3647);

\path [draw=white!41.1764705882353!black, line width=0.48pt]
(axis cs:4,3263)
--(axis cs:4,3303);

\path [draw=white!41.1764705882353!black, line width=0.48pt]
(axis cs:6,3005)
--(axis cs:6,3051);

\path [draw=white!41.1764705882353!black, line width=0.48pt]
(axis cs:8,2496)
--(axis cs:8,2528);

\path [draw=white!41.1764705882353!black, line width=0.48pt]
(axis cs:0.33,3776)
--(axis cs:0.33,3932);

\path [draw=white!41.1764705882353!black, line width=0.48pt]
(axis cs:2.33,3744)
--(axis cs:2.33,3860);

\path [draw=white!41.1764705882353!black, line width=0.48pt]
(axis cs:4.33,3382)
--(axis cs:4.33,3450);

\path [draw=white!41.1764705882353!black, line width=0.48pt]
(axis cs:6.33,3108)
--(axis cs:6.33,3182);

\path [draw=white!41.1764705882353!black, line width=0.48pt]
(axis cs:8.33,2688)
--(axis cs:8.33,2876);

\path [draw=white!41.1764705882353!black, line width=0.48pt]
(axis cs:0.76,5378)
--(axis cs:0.76,5474);

\path [draw=white!41.1764705882353!black, line width=0.48pt]
(axis cs:2.76,5193)
--(axis cs:2.76,5625);

\path [draw=white!41.1764705882353!black, line width=0.48pt]
(axis cs:4.76,5016)
--(axis cs:4.76,5370);

\path [draw=white!41.1764705882353!black, line width=0.48pt]
(axis cs:6.76,5110)
--(axis cs:6.76,5246);

\path [draw=white!41.1764705882353!black, line width=0.48pt]
(axis cs:8.76,3944)
--(axis cs:8.76,4252);

\path [draw=white!41.1764705882353!black, line width=0.48pt]
(axis cs:1.09,5964)
--(axis cs:1.09,6100);

\path [draw=white!41.1764705882353!black, line width=0.48pt]
(axis cs:3.09,5622)
--(axis cs:3.09,6078);

\path [draw=white!41.1764705882353!black, line width=0.48pt]
(axis cs:5.09,5099)
--(axis cs:5.09,5539);

\path [draw=white!41.1764705882353!black, line width=0.48pt]
(axis cs:7.09,5311)
--(axis cs:7.09,5419);

\path [draw=white!41.1764705882353!black, line width=0.48pt]
(axis cs:9.09,4097)
--(axis cs:9.09,4441);

\addplot [line width=0.48pt, white!41.1764705882353!black, opacity=1, mark=-, mark size=1.3, mark options={solid}, only marks]
table {%
0 3612
2 3601
4 3263
6 3005
8 2496
};
\addplot [line width=0.48pt, white!41.1764705882353!black, opacity=1, mark=-, mark size=1.3, mark options={solid}, only marks]
table {%
0 3652
2 3647
4 3303
6 3051
8 2528
};

\addplot [line width=0.48pt, white!41.1764705882353!black, opacity=1, mark=-, mark size=1.3, mark options={solid}, only marks]
table {%
0.33 3776
2.33 3744
4.33 3382
6.33 3108
8.33 2688
};

\addplot [line width=0.48pt, white!41.1764705882353!black, opacity=1, mark=-, mark size=1.3, mark options={solid}, only marks]
table {%
0.33 3932
2.33 3860
4.33 3450
6.33 3182
8.33 2876
};

\addplot [line width=0.48pt, white!41.1764705882353!black, opacity=1, mark=-, mark size=1.3, mark options={solid}, only marks]
table {%
0.76 5378
2.76 5193
4.76 5016
6.76 5110
8.76 3944
};

\addplot [line width=0.48pt, white!41.1764705882353!black, opacity=1, mark=-, mark size=1.3, mark options={solid}, only marks]
table {%
0.76 5474
2.76 5625
4.76 5370
6.76 5246
8.76 4252
};

\addplot [line width=0.48pt, white!41.1764705882353!black, opacity=1, mark=-, mark size=1.3, mark options={solid}, only marks]
table {%
1.09 5964
3.09 5622
5.09 5099
7.09 5311
9.09 4097
};

\addplot [line width=0.48pt, white!41.1764705882353!black, opacity=1, mark=-, mark size=1.3, mark options={solid}, only marks]
table {%
1.09 6100
3.09 6078
5.09 5539
7.09 5419
9.09 4441
};

\end{axis}

\end{tikzpicture}}
        \vspace{-1.8em}
        \caption{\scriptsize{Halfcheetah}}
        \label{sfig:abl_halfcheetah}
    \end{subfigure}
    \begin{subfigure}[t]{0.45\textwidth}
        \centering
        \resizebox{\textwidth}{!}{
\begin{tikzpicture}[scale=0.6]

\definecolor{color0}{rgb}{0.168966293897257,0.475610408705207,0.783522853775186}
\definecolor{color1}{rgb}{0.691639317416794,0.806706471648465,0.92224956246861}
\definecolor{color2}{rgb}{0.95715994546835,0.680184738372337,0.77629698219994}
\definecolor{color3}{rgb}{0.885434799735519,0.144732659736887,0.401761241571282}

\begin{axis}[
height=3cm,
width=\textwidth,
scale only axis,
axis line style={white!15!black},
legend cell align={left},
legend columns=2,
legend style={
  fill opacity=0,
  draw opacity=1,
  text opacity=1,
  at={(0.53,1.01)},
  anchor=north,
  draw=none,
  font=\scriptsize
},
xtick pos = lower,
tick align=outside,
x grid style={white!80!black},
xmin=-0.636, xmax=9.726,
xtick style={color=white!15!black},
xtick={0.595,2.595,4.595,6.595,8.595},
xticklabels={\scriptsize{No Attack},\scriptsize{MaxDiff},\scriptsize{Robust Sarsa},\scriptsize{SA-RL},\scriptsize{PA-AD}},
y grid style={white!80!black},
ylabel={\scriptsize{Average episode rewards}},
ymin=500, ymax=5500,
ytick style={color=white!15!black},
y grid style={white},
tick label style={font=\tiny},
ymajorgrids,
]
\draw[draw=white!93.3333333333333!black,fill=color0,opacity=0.7,very thin,postaction={pattern=crosshatch, pattern color=white!93.3333333333333!black, fill opacity=0.7}] (axis cs:-0.165,0) rectangle (axis cs:0.165,3705);
\addlegendimage{ybar,ybar legend,draw=white!93.3333333333333!black,fill=color0,opacity=0.7,very thin,postaction={pattern=crosshatch, pattern color=white!93.3333333333333!black, fill opacity=0.7}}
\addlegendentry{SA-PPO}

\draw[draw=white!93.3333333333333!black,fill=color0,opacity=0.7,very thin,postaction={pattern=crosshatch, pattern color=white!93.3333333333333!black, fill opacity=0.7}] (axis cs:1.835,0) rectangle (axis cs:2.165,2652);
\draw[draw=white!93.3333333333333!black,fill=color0,opacity=0.7,very thin,postaction={pattern=crosshatch, pattern color=white!93.3333333333333!black, fill opacity=0.7}] (axis cs:3.835,0) rectangle (axis cs:4.165,1130);
\draw[draw=white!93.3333333333333!black,fill=color0,opacity=0.7,very thin,postaction={pattern=crosshatch, pattern color=white!93.3333333333333!black, fill opacity=0.7}] (axis cs:5.835,0) rectangle (axis cs:6.165,1076);
\draw[draw=white!93.3333333333333!black,fill=color0,opacity=0.7,very thin,postaction={pattern=crosshatch, pattern color=white!93.3333333333333!black, fill opacity=0.7}] (axis cs:7.835,0) rectangle (axis cs:8.165,856);
\draw[draw=white!93.3333333333333!black,fill=color1,opacity=0.8,very thin,postaction={pattern=north west lines, pattern color=white!93.3333333333333!black, fill opacity=0.8}] (axis cs:0.165,0) rectangle (axis cs:0.495,3694);
\addlegendimage{ybar,ybar legend,draw=white!93.3333333333333!black,fill=color1,opacity=0.8,very thin,postaction={pattern=north west lines, pattern color=white!93.3333333333333!black, fill opacity=0.8}}
\addlegendentry{SA-PPO + $w(s)$}

\draw[draw=white!93.3333333333333!black,fill=color1,opacity=0.8,very thin,postaction={pattern=north west lines, pattern color=white!93.3333333333333!black, fill opacity=0.8}] (axis cs:2.165,0) rectangle (axis cs:2.495,2743);
\draw[draw=white!93.3333333333333!black,fill=color1,opacity=0.8,very thin,postaction={pattern=north west lines, pattern color=white!93.3333333333333!black, fill opacity=0.8}] (axis cs:4.165,0) rectangle (axis cs:4.495,1285);
\draw[draw=white!93.3333333333333!black,fill=color1,opacity=0.8,very thin,postaction={pattern=north west lines, pattern color=white!93.3333333333333!black, fill opacity=0.8}] (axis cs:6.165,0) rectangle (axis cs:6.495,1324);
\draw[draw=white!93.3333333333333!black,fill=color1,opacity=0.8,very thin,postaction={pattern=north west lines, pattern color=white!93.3333333333333!black, fill opacity=0.8}] (axis cs:8.165,0) rectangle (axis cs:8.495,1025);
\draw[draw=white!93.3333333333333!black,fill=color2,opacity=0.8,very thin,postaction={pattern=north east lines, pattern color=white!93.3333333333333!black, fill opacity=0.8}] (axis cs:0.595,0) rectangle (axis cs:0.925,3574);
\addlegendimage{ybar,ybar legend,draw=white!93.3333333333333!black,fill=color2,opacity=0.8,very thin,postaction={pattern=north east lines, pattern color=white!93.3333333333333!black, fill opacity=0.8}}
\addlegendentry{\ourppo - $w(s)$}

\draw[draw=white!93.3333333333333!black,fill=color2,opacity=0.8,very thin,postaction={pattern=north east lines, pattern color=white!93.3333333333333!black, fill opacity=0.8}] (axis cs:2.595,0) rectangle (axis cs:2.925,3314);
\draw[draw=white!93.3333333333333!black,fill=color2,opacity=0.8,very thin,postaction={pattern=north east lines, pattern color=white!93.3333333333333!black, fill opacity=0.8}] (axis cs:4.595,0) rectangle (axis cs:4.925,3019);
\draw[draw=white!93.3333333333333!black,fill=color2,opacity=0.8,very thin,postaction={pattern=north east lines, pattern color=white!93.3333333333333!black, fill opacity=0.8}] (axis cs:6.595,0) rectangle (axis cs:6.925,2034);
\draw[draw=white!93.3333333333333!black,fill=color2,opacity=0.8,very thin,postaction={pattern=north east lines, pattern color=white!93.3333333333333!black, fill opacity=0.8}] (axis cs:8.595,0) rectangle (axis cs:8.925,2387);
\draw[draw=white!93.3333333333333!black,fill=color3,opacity=0.8,very thin] (axis cs:0.925,0) rectangle (axis cs:1.255,3616);
\addlegendimage{ybar,ybar legend,draw=white!93.3333333333333!black,fill=color3,opacity=0.8,very thin}
\addlegendentry{\ourppo (Ours)}

\draw[draw=white!93.3333333333333!black,fill=color3,opacity=0.8,very thin] (axis cs:2.925,0) rectangle (axis cs:3.255,3541);
\draw[draw=white!93.3333333333333!black,fill=color3,opacity=0.8,very thin] (axis cs:4.925,0) rectangle (axis cs:5.255,3277);
\draw[draw=white!93.3333333333333!black,fill=color3,opacity=0.8,very thin] (axis cs:6.925,0) rectangle (axis cs:7.255,2390);
\draw[draw=white!93.3333333333333!black,fill=color3,opacity=0.8,very thin] (axis cs:8.925,0) rectangle (axis cs:9.255,2579);
\path [draw=white!41.1764705882353!black, line width=0.48pt]
(axis cs:0,3703)
--(axis cs:0,3707);

\path [draw=white!41.1764705882353!black, line width=0.48pt]
(axis cs:2,2417)
--(axis cs:2,2887);

\path [draw=white!41.1764705882353!black, line width=0.48pt]
(axis cs:4,1088)
--(axis cs:4,1172);

\path [draw=white!41.1764705882353!black, line width=0.48pt]
(axis cs:6,890)
--(axis cs:6,1262);

\path [draw=white!41.1764705882353!black, line width=0.48pt]
(axis cs:8,835)
--(axis cs:8,877);

\path [draw=white!41.1764705882353!black, line width=0.48pt]
(axis cs:0.33,3429)
--(axis cs:0.33,3959);

\path [draw=white!41.1764705882353!black, line width=0.48pt]
(axis cs:2.33,2585)
--(axis cs:2.33,2901);

\path [draw=white!41.1764705882353!black, line width=0.48pt]
(axis cs:4.33,1088)
--(axis cs:4.33,1482);

\path [draw=white!41.1764705882353!black, line width=0.48pt]
(axis cs:6.33,1095)
--(axis cs:6.33,1553);

\path [draw=white!41.1764705882353!black, line width=0.48pt]
(axis cs:8.33,828)
--(axis cs:8.33,1222);

\path [draw=white!41.1764705882353!black, line width=0.48pt]
(axis cs:0.76,3388)
--(axis cs:0.76,3760);

\path [draw=white!41.1764705882353!black, line width=0.48pt]
(axis cs:2.76,2791)
--(axis cs:2.76,3837);

\path [draw=white!41.1764705882353!black, line width=0.48pt]
(axis cs:4.76,2741)
--(axis cs:4.76,3297);

\path [draw=white!41.1764705882353!black, line width=0.48pt]
(axis cs:6.76,1667)
--(axis cs:6.76,2401);

\path [draw=white!41.1764705882353!black, line width=0.48pt]
(axis cs:8.76,2166)
--(axis cs:8.76,2608);

\path [draw=white!41.1764705882353!black, line width=0.48pt]
(axis cs:1.09,3517)
--(axis cs:1.09,3715);

\path [draw=white!41.1764705882353!black, line width=0.48pt]
(axis cs:3.09,3334)
--(axis cs:3.09,3748);

\path [draw=white!41.1764705882353!black, line width=0.48pt]
(axis cs:5.09,3118)
--(axis cs:5.09,3436);

\path [draw=white!41.1764705882353!black, line width=0.48pt]
(axis cs:7.09,2245)
--(axis cs:7.09,2535);

\path [draw=white!41.1764705882353!black, line width=0.48pt]
(axis cs:9.09,2350)
--(axis cs:9.09,2808);

\addplot [line width=0.48pt, white!41.1764705882353!black, opacity=1, mark=-, mark size=1.3, mark options={solid}, only marks]
table {%
0 3703
2 2417
4 1088
6 890
8 835
};

\addplot [line width=0.48pt, white!41.1764705882353!black, opacity=1, mark=-, mark size=1.3, mark options={solid}, only marks]
table {%
0 3707
2 2887
4 1172
6 1262
8 877
};

\addplot [line width=0.48pt, white!41.1764705882353!black, opacity=1, mark=-, mark size=1.3, mark options={solid}, only marks]
table {%
0.33 3429
2.33 2585
4.33 1088
6.33 1095
8.33 828
};

\addplot [line width=0.48pt, white!41.1764705882353!black, opacity=1, mark=-, mark size=1.3, mark options={solid}, only marks]
table {%
0.33 3959
2.33 2901
4.33 1482
6.33 1553
8.33 1222
};

\addplot [line width=0.48pt, white!41.1764705882353!black, opacity=1, mark=-, mark size=1.3, mark options={solid}, only marks]
table {%
0.76 3388
2.76 2791
4.76 2741
6.76 1667
8.76 2166
};

\addplot [line width=0.48pt, white!41.1764705882353!black, opacity=1, mark=-, mark size=1.3, mark options={solid}, only marks]
table {%
0.76 3760
2.76 3837
4.76 3297
6.76 2401
8.76 2608
};

\addplot [line width=0.48pt, white!41.1764705882353!black, opacity=1, mark=-, mark size=1.3, mark options={solid}, only marks]
table {%
1.09 3517
3.09 3334
5.09 3118
7.09 2245
9.09 2350
};

\addplot [line width=0.48pt, white!41.1764705882353!black, opacity=1, mark=-, mark size=1.3, mark options={solid}, only marks]
table {%
1.09 3715
3.09 3748
5.09 3436
7.09 2535
9.09 2808
};

\end{axis}

\end{tikzpicture}}
        \vspace{-1.8em}
        \caption{\scriptsize{Hopper}}
        \label{sfig:abl_hopper}
    \end{subfigure}\\
        \begin{subfigure}[t]{0.45\textwidth}
        \centering
        \resizebox{\textwidth}{!}{
\begin{tikzpicture}[scale=0.6]

\definecolor{color0}{rgb}{0.168966293897257,0.475610408705207,0.783522853775186}
\definecolor{color1}{rgb}{0.691639317416794,0.806706471648465,0.92224956246861}
\definecolor{color2}{rgb}{0.95715994546835,0.680184738372337,0.77629698219994}
\definecolor{color3}{rgb}{0.885434799735519,0.144732659736887,0.401761241571282}

\begin{axis}[
height=3cm,
width=\textwidth,
scale only axis,
axis line style={white!15!black},
legend cell align={left},
legend columns=2,
legend style={
  fill opacity=0,
  draw opacity=1,
  text opacity=1,
  at={(0.53,1.01)},
  anchor=north,
  draw=none,
  font=\scriptsize
},
xtick pos = lower,
tick align=outside,
x grid style={white!80!black},
xmin=-0.636, xmax=9.726,
xtick style={color=white!15!black},
xtick={0.595,2.595,4.595,6.595,8.595},
xticklabels={\scriptsize{No Attack},\scriptsize{MaxDiff},\scriptsize{Robust Sarsa},\scriptsize{SA-RL},\scriptsize{PA-AD}},
y grid style={white!80!black},
ylabel={\scriptsize{Average episode rewards}},
ymin=-1900, ymax=8800,
ytick style={color=white!15!black},
y grid style={white},
tick label style={font=\tiny},
ymajorgrids,
]
\draw[draw=white!93.3333333333333!black,fill=color0,opacity=0.7,very thin,postaction={pattern=crosshatch, pattern color=white!93.3333333333333!black, fill opacity=0.7}] (axis cs:-0.165,0) rectangle (axis cs:0.165,4292);
\addlegendimage{ybar,ybar legend,draw=white!93.3333333333333!black,fill=color0,opacity=0.7,very thin,postaction={pattern=crosshatch, pattern color=white!93.3333333333333!black, fill opacity=0.7}}
\addlegendentry{SA-PPO}

\draw[draw=white!93.3333333333333!black,fill=color0,opacity=0.7,very thin,postaction={pattern=crosshatch, pattern color=white!93.3333333333333!black, fill opacity=0.7}] (axis cs:1.835,0) rectangle (axis cs:2.165,4662);
\draw[draw=white!93.3333333333333!black,fill=color0,opacity=0.7,very thin,postaction={pattern=crosshatch, pattern color=white!93.3333333333333!black, fill opacity=0.7}] (axis cs:3.835,0) rectangle (axis cs:4.165,3412);
\draw[draw=white!93.3333333333333!black,fill=color0,opacity=0.7,very thin,postaction={pattern=crosshatch, pattern color=white!93.3333333333333!black, fill opacity=0.7}] (axis cs:5.835,0) rectangle (axis cs:6.165,2511);
\draw[draw=white!93.3333333333333!black,fill=color0,opacity=0.7,very thin,postaction={pattern=crosshatch, pattern color=white!93.3333333333333!black, fill opacity=0.7}] (axis cs:7.835,0) rectangle (axis cs:8.165,-1296);
\draw[draw=white!93.3333333333333!black,fill=color1,opacity=0.8,very thin,postaction={pattern=north west lines, pattern color=white!93.3333333333333!black, fill opacity=0.8}] (axis cs:0.165,0) rectangle (axis cs:0.495,4528);
\addlegendimage{ybar,ybar legend,draw=white!93.3333333333333!black,fill=color1,opacity=0.8,very thin,postaction={pattern=north west lines, pattern color=white!93.3333333333333!black, fill opacity=0.8}}
\addlegendentry{SA-PPO + $w(s)$}

\draw[draw=white!93.3333333333333!black,fill=color1,opacity=0.8,very thin,postaction={pattern=north west lines, pattern color=white!93.3333333333333!black, fill opacity=0.8}] (axis cs:2.165,0) rectangle (axis cs:2.495,4716);
\draw[draw=white!93.3333333333333!black,fill=color1,opacity=0.8,very thin,postaction={pattern=north west lines, pattern color=white!93.3333333333333!black, fill opacity=0.8}] (axis cs:4.165,0) rectangle (axis cs:4.495,3621);
\draw[draw=white!93.3333333333333!black,fill=color1,opacity=0.8,very thin,postaction={pattern=north west lines, pattern color=white!93.3333333333333!black, fill opacity=0.8}] (axis cs:6.165,0) rectangle (axis cs:6.495,2684);
\draw[draw=white!93.3333333333333!black,fill=color1,opacity=0.8,very thin,postaction={pattern=north west lines, pattern color=white!93.3333333333333!black, fill opacity=0.8}] (axis cs:8.165,0) rectangle (axis cs:8.495,-578);
\draw[draw=white!93.3333333333333!black,fill=color2,opacity=0.8,very thin,postaction={pattern=north east lines, pattern color=white!93.3333333333333!black, fill opacity=0.8}] (axis cs:0.595,0) rectangle (axis cs:0.925,5304);
\addlegendimage{ybar,ybar legend,draw=white!93.3333333333333!black,fill=color2,opacity=0.8,very thin,postaction={pattern=north east lines, pattern color=white!93.3333333333333!black, fill opacity=0.8}}
\addlegendentry{\ourppo - $w(s)$}

\draw[draw=white!93.3333333333333!black,fill=color2,opacity=0.8,very thin,postaction={pattern=north east lines, pattern color=white!93.3333333333333!black, fill opacity=0.8}] (axis cs:2.595,0) rectangle (axis cs:2.925,5097);
\draw[draw=white!93.3333333333333!black,fill=color2,opacity=0.8,very thin,postaction={pattern=north east lines, pattern color=white!93.3333333333333!black, fill opacity=0.8}] (axis cs:4.595,0) rectangle (axis cs:4.925,4238);
\draw[draw=white!93.3333333333333!black,fill=color2,opacity=0.8,very thin,postaction={pattern=north east lines, pattern color=white!93.3333333333333!black, fill opacity=0.8}] (axis cs:6.595,0) rectangle (axis cs:6.925,3659);
\draw[draw=white!93.3333333333333!black,fill=color2,opacity=0.8,very thin,postaction={pattern=north east lines, pattern color=white!93.3333333333333!black, fill opacity=0.8}] (axis cs:8.595,0) rectangle (axis cs:8.925,2864);
\draw[draw=white!93.3333333333333!black,fill=color3,opacity=0.8,very thin] (axis cs:0.925,0) rectangle (axis cs:1.255,5596);
\addlegendimage{ybar,ybar legend,draw=white!93.3333333333333!black,fill=color3,opacity=0.8,very thin}
\addlegendentry{\ourppo (Ours)}

\draw[draw=white!93.3333333333333!black,fill=color3,opacity=0.8,very thin] (axis cs:2.925,0) rectangle (axis cs:3.255,5284);
\draw[draw=white!93.3333333333333!black,fill=color3,opacity=0.8,very thin] (axis cs:4.925,0) rectangle (axis cs:5.255,4339);
\draw[draw=white!93.3333333333333!black,fill=color3,opacity=0.8,very thin] (axis cs:6.925,0) rectangle (axis cs:7.255,3822);
\draw[draw=white!93.3333333333333!black,fill=color3,opacity=0.8,very thin] (axis cs:8.925,0) rectangle (axis cs:9.255,3164);
\path [draw=white!41.1764705882353!black, line width=0.48pt]
(axis cs:0,3968)
--(axis cs:0,4616);

\path [draw=white!41.1764705882353!black, line width=0.48pt]
(axis cs:2,4340)
--(axis cs:2,4984);

\path [draw=white!41.1764705882353!black, line width=0.48pt]
(axis cs:4,2657)
--(axis cs:4,4167);

\path [draw=white!41.1764705882353!black, line width=0.48pt]
(axis cs:6,1894)
--(axis cs:6,3128);

\path [draw=white!41.1764705882353!black, line width=0.48pt]
(axis cs:8,-1719)
--(axis cs:8,-873);

\path [draw=white!41.1764705882353!black, line width=0.48pt]
(axis cs:0.33,4092)
--(axis cs:0.33,4964);

\path [draw=white!41.1764705882353!black, line width=0.48pt]
(axis cs:2.33,4398)
--(axis cs:2.33,5034);

\path [draw=white!41.1764705882353!black, line width=0.48pt]
(axis cs:4.33,2989)
--(axis cs:4.33,4253);

\path [draw=white!41.1764705882353!black, line width=0.48pt]
(axis cs:6.33,2101)
--(axis cs:6.33,3267);

\path [draw=white!41.1764705882353!black, line width=0.48pt]
(axis cs:8.33,-1034)
--(axis cs:8.33,-122);

\path [draw=white!41.1764705882353!black, line width=0.48pt]
(axis cs:0.76,5006)
--(axis cs:0.76,5602);

\path [draw=white!41.1764705882353!black, line width=0.48pt]
(axis cs:2.76,4863)
--(axis cs:2.76,5331);

\path [draw=white!41.1764705882353!black, line width=0.48pt]
(axis cs:4.76,4042)
--(axis cs:4.76,4434);

\path [draw=white!41.1764705882353!black, line width=0.48pt]
(axis cs:6.76,3383)
--(axis cs:6.76,3935);

\path [draw=white!41.1764705882353!black, line width=0.48pt]
(axis cs:8.76,2641)
--(axis cs:8.76,3087);

\path [draw=white!41.1764705882353!black, line width=0.48pt]
(axis cs:1.09,5371)
--(axis cs:1.09,5821);

\path [draw=white!41.1764705882353!black, line width=0.48pt]
(axis cs:3.09,5102)
--(axis cs:3.09,5466);

\path [draw=white!41.1764705882353!black, line width=0.48pt]
(axis cs:5.09,4179)
--(axis cs:5.09,4499);

\path [draw=white!41.1764705882353!black, line width=0.48pt]
(axis cs:7.09,3637)
--(axis cs:7.09,4007);

\path [draw=white!41.1764705882353!black, line width=0.48pt]
(axis cs:9.09,3001)
--(axis cs:9.09,3327);

\addplot [line width=0.48pt, white!41.1764705882353!black, opacity=1, mark=-, mark size=1.3, mark options={solid}, only marks]
table {%
0 3968
2 4340
4 2657
6 1894
8 -1719
};

\addplot [line width=0.48pt, white!41.1764705882353!black, opacity=1, mark=-, mark size=1.3, mark options={solid}, only marks]
table {%
0 4616
2 4984
4 4167
6 3128
8 -873
};

\addplot [line width=0.48pt, white!41.1764705882353!black, opacity=1, mark=-, mark size=1.3, mark options={solid}, only marks]
table {%
0.33 4092
2.33 4398
4.33 2989
6.33 2101
8.33 -1034
};

\addplot [line width=0.48pt, white!41.1764705882353!black, opacity=1, mark=-, mark size=1.3, mark options={solid}, only marks]
table {%
0.33 4964
2.33 5034
4.33 4253
6.33 3267
8.33 -122
};

\addplot [line width=0.48pt, white!41.1764705882353!black, opacity=1, mark=-, mark size=1.3, mark options={solid}, only marks]
table {%
0.76 5006
2.76 4863
4.76 4042
6.76 3383
8.76 2641
};

\addplot [line width=0.48pt, white!41.1764705882353!black, opacity=1, mark=-, mark size=1.3, mark options={solid}, only marks]
table {%
0.76 5602
2.76 5331
4.76 4434
6.76 3935
8.76 3087
};

\addplot [line width=0.48pt, white!41.1764705882353!black, opacity=1, mark=-, mark size=1.3, mark options={solid}, only marks]
table {%
1.09 5371
3.09 5102
5.09 4179
7.09 3637
9.09 3001
};

\addplot [line width=0.48pt, white!41.1764705882353!black, opacity=1, mark=-, mark size=1.3, mark options={solid}, only marks]
table {%
1.09 5821
3.09 5466
5.09 4499
7.09 4007
9.09 3327
};

\end{axis}

\end{tikzpicture}}
        \vspace{-1.8em}
        \caption{\scriptsize{Ant}}
        \label{sfig:abl_ant}
    \end{subfigure}
    \begin{subfigure}[t]{0.45\textwidth}
        \centering
        \resizebox{\textwidth}{!}{
\begin{tikzpicture}[scale=0.6]

\definecolor{color0}{rgb}{0.168966293897257,0.475610408705207,0.783522853775186}
\definecolor{color1}{rgb}{0.691639317416794,0.806706471648465,0.92224956246861}
\definecolor{color2}{rgb}{0.95715994546835,0.680184738372337,0.77629698219994}
\definecolor{color3}{rgb}{0.885434799735519,0.144732659736887,0.401761241571282}

\begin{axis}[
height=3cm,
width=\textwidth,
scale only axis,
axis line style={white!15!black},
legend cell align={left},
legend columns=2,
legend style={
  fill opacity=0,
  draw opacity=1,
  text opacity=1,
  at={(0.53,1.01)},
  anchor=north,
  draw=none,
  font=\scriptsize
},
xtick pos = lower,
tick align=outside,
x grid style={white!80!black},
xmin=-0.636, xmax=9.726,
xtick style={color=white!15!black},
xtick={0.595,2.595,4.595,6.595,8.595},
xticklabels={\scriptsize{No Attack},\scriptsize{MaxDiff},\scriptsize{Robust Sarsa},\scriptsize{SA-RL},\scriptsize{PA-AD}},
y grid style={white!80!black},
ylabel={\scriptsize{Average episode rewards}},
ymin=800, ymax=6000,
ytick style={color=white!15!black},
y grid style={white},
tick label style={font=\tiny},
ymajorgrids,
]
\draw[draw=white!93.3333333333333!black,fill=color0,opacity=0.7,very thin,postaction={pattern=crosshatch, pattern color=white!93.3333333333333!black, fill opacity=0.7}] (axis cs:-0.165,0) rectangle (axis cs:0.165,4487);
\addlegendimage{ybar,ybar legend,draw=white!93.3333333333333!black,fill=color0,opacity=0.7,very thin,postaction={pattern=crosshatch, pattern color=white!93.3333333333333!black, fill opacity=0.7}}
\addlegendentry{SA-PPO}

\draw[draw=white!93.3333333333333!black,fill=color0,opacity=0.7,very thin,postaction={pattern=crosshatch, pattern color=white!93.3333333333333!black, fill opacity=0.7}] (axis cs:1.835,0) rectangle (axis cs:2.165,3668);
\draw[draw=white!93.3333333333333!black,fill=color0,opacity=0.7,very thin,postaction={pattern=crosshatch, pattern color=white!93.3333333333333!black, fill opacity=0.7}] (axis cs:3.835,0) rectangle (axis cs:4.165,3808);
\draw[draw=white!93.3333333333333!black,fill=color0,opacity=0.7,very thin,postaction={pattern=crosshatch, pattern color=white!93.3333333333333!black, fill opacity=0.7}] (axis cs:5.835,0) rectangle (axis cs:6.165,2908);
\draw[draw=white!93.3333333333333!black,fill=color0,opacity=0.7,very thin,postaction={pattern=crosshatch, pattern color=white!93.3333333333333!black, fill opacity=0.7}] (axis cs:7.835,0) rectangle (axis cs:8.165,1042);
\draw[draw=white!93.3333333333333!black,fill=color1,opacity=0.8,very thin,postaction={pattern=north west lines, pattern color=white!93.3333333333333!black, fill opacity=0.8}] (axis cs:0.165,0) rectangle (axis cs:0.495,4576);
\addlegendimage{ybar,ybar legend,draw=white!93.3333333333333!black,fill=color1,opacity=0.8,very thin,postaction={pattern=north west lines, pattern color=white!93.3333333333333!black, fill opacity=0.8}}
\addlegendentry{SA-PPO + $w(s)$}

\draw[draw=white!93.3333333333333!black,fill=color1,opacity=0.8,very thin,postaction={pattern=north west lines, pattern color=white!93.3333333333333!black, fill opacity=0.8}] (axis cs:2.165,0) rectangle (axis cs:2.495,3823);
\draw[draw=white!93.3333333333333!black,fill=color1,opacity=0.8,very thin,postaction={pattern=north west lines, pattern color=white!93.3333333333333!black, fill opacity=0.8}] (axis cs:4.165,0) rectangle (axis cs:4.495,3946);
\draw[draw=white!93.3333333333333!black,fill=color1,opacity=0.8,very thin,postaction={pattern=north west lines, pattern color=white!93.3333333333333!black, fill opacity=0.8}] (axis cs:6.165,0) rectangle (axis cs:6.495,3104);
\draw[draw=white!93.3333333333333!black,fill=color1,opacity=0.8,very thin,postaction={pattern=north west lines, pattern color=white!93.3333333333333!black, fill opacity=0.8}] (axis cs:8.165,0) rectangle (axis cs:8.495,1179);
\draw[draw=white!93.3333333333333!black,fill=color2,opacity=0.8,very thin,postaction={pattern=north east lines, pattern color=white!93.3333333333333!black, fill opacity=0.8}] (axis cs:0.595,0) rectangle (axis cs:0.925,4146);
\addlegendimage{ybar,ybar legend,draw=white!93.3333333333333!black,fill=color2,opacity=0.8,very thin,postaction={pattern=north east lines, pattern color=white!93.3333333333333!black, fill opacity=0.8}}
\addlegendentry{WocaR-PPO - $w(s)$}

\draw[draw=white!93.3333333333333!black,fill=color2,opacity=0.8,very thin,postaction={pattern=north east lines, pattern color=white!93.3333333333333!black, fill opacity=0.8}] (axis cs:2.595,0) rectangle (axis cs:2.925,4141);
\draw[draw=white!93.3333333333333!black,fill=color2,opacity=0.8,very thin,postaction={pattern=north east lines, pattern color=white!93.3333333333333!black, fill opacity=0.8}] (axis cs:4.595,0) rectangle (axis cs:4.925,3986);
\draw[draw=white!93.3333333333333!black,fill=color2,opacity=0.8,very thin,postaction={pattern=north east lines, pattern color=white!93.3333333333333!black, fill opacity=0.8}] (axis cs:6.595,0) rectangle (axis cs:6.925,3574);
\draw[draw=white!93.3333333333333!black,fill=color2,opacity=0.8,very thin,postaction={pattern=north east lines, pattern color=white!93.3333333333333!black, fill opacity=0.8}] (axis cs:8.595,0) rectangle (axis cs:8.925,2476);
\draw[draw=white!93.3333333333333!black,fill=color3,opacity=0.8,very thin] (axis cs:0.925,0) rectangle (axis cs:1.255,4156);
\addlegendimage{ybar,ybar legend,draw=white!93.3333333333333!black,fill=color3,opacity=0.8,very thin}
\addlegendentry{WocaR-PPO (Ours)}

\draw[draw=white!93.3333333333333!black,fill=color3,opacity=0.8,very thin] (axis cs:2.925,0) rectangle (axis cs:3.255,4177);
\draw[draw=white!93.3333333333333!black,fill=color3,opacity=0.8,very thin] (axis cs:4.925,0) rectangle (axis cs:5.255,4093);
\draw[draw=white!93.3333333333333!black,fill=color3,opacity=0.8,very thin] (axis cs:6.925,0) rectangle (axis cs:7.255,3770);
\draw[draw=white!93.3333333333333!black,fill=color3,opacity=0.8,very thin] (axis cs:8.925,0) rectangle (axis cs:9.255,2722);
\path [draw=white!41.1764705882353!black, line width=0.48pt]
(axis cs:0,4426)
--(axis cs:0,4548);

\path [draw=white!41.1764705882353!black, line width=0.48pt]
(axis cs:2,3379)
--(axis cs:2,3957);

\path [draw=white!41.1764705882353!black, line width=0.48pt]
(axis cs:4,3670)
--(axis cs:4,3946);

\path [draw=white!41.1764705882353!black, line width=0.48pt]
(axis cs:6,2672)
--(axis cs:6,3144);

\path [draw=white!41.1764705882353!black, line width=0.48pt]
(axis cs:8,849)
--(axis cs:8,1235);

\path [draw=white!41.1764705882353!black, line width=0.48pt]
(axis cs:0.33,4334)
--(axis cs:0.33,4818);

\path [draw=white!41.1764705882353!black, line width=0.48pt]
(axis cs:2.33,3568)
--(axis cs:2.33,4078);

\path [draw=white!41.1764705882353!black, line width=0.48pt]
(axis cs:4.33,3774)
--(axis cs:4.33,4118);

\path [draw=white!41.1764705882353!black, line width=0.48pt]
(axis cs:6.33,2786)
--(axis cs:6.33,3422);

\path [draw=white!41.1764705882353!black, line width=0.48pt]
(axis cs:8.33,941)
--(axis cs:8.33,1417);

\path [draw=white!41.1764705882353!black, line width=0.48pt]
(axis cs:0.76,3948)
--(axis cs:0.76,4344);

\path [draw=white!41.1764705882353!black, line width=0.48pt]
(axis cs:2.76,3950)
--(axis cs:2.76,4332);

\path [draw=white!41.1764705882353!black, line width=0.48pt]
(axis cs:4.76,3814)
--(axis cs:4.76,4158);

\path [draw=white!41.1764705882353!black, line width=0.48pt]
(axis cs:6.76,3281)
--(axis cs:6.76,3867);

\path [draw=white!41.1764705882353!black, line width=0.48pt]
(axis cs:8.76,2182)
--(axis cs:8.76,2770);

\path [draw=white!41.1764705882353!black, line width=0.48pt]
(axis cs:1.09,3799)
--(axis cs:1.09,4513);

\path [draw=white!41.1764705882353!black, line width=0.48pt]
(axis cs:3.09,4001)
--(axis cs:3.09,4353);

\path [draw=white!41.1764705882353!black, line width=0.48pt]
(axis cs:5.09,3955)
--(axis cs:5.09,4231);

\path [draw=white!41.1764705882353!black, line width=0.48pt]
(axis cs:7.09,3574)
--(axis cs:7.09,3966);

\path [draw=white!41.1764705882353!black, line width=0.48pt]
(axis cs:9.09,2549)
--(axis cs:9.09,2895);

\addplot [line width=0.48pt, white!41.1764705882353!black, opacity=1, mark=-, mark size=1.3, mark options={solid}, only marks]
table {%
0 4426
2 3379
4 3670
6 2672
8 849
};
\addplot [line width=0.48pt, white!41.1764705882353!black, opacity=1, mark=-, mark size=1.3, mark options={solid}, only marks]
table {%
0 4548
2 3957
4 3946
6 3144
8 1235
};
\addplot [line width=0.48pt, white!41.1764705882353!black, opacity=1, mark=-, mark size=1.3, mark options={solid}, only marks]
table {%
0.33 4334
2.33 3568
4.33 3774
6.33 2786
8.33 941
};
\addplot [line width=0.48pt, white!41.1764705882353!black, opacity=1, mark=-, mark size=1.3, mark options={solid}, only marks]
table {%
0.33 4818
2.33 4078
4.33 4118
6.33 3422
8.33 1417
};
\addplot [line width=0.48pt, white!41.1764705882353!black, opacity=1, mark=-, mark size=1.3, mark options={solid}, only marks]
table {%
0.76 3948
2.76 3950
4.76 3814
6.76 3281
8.76 2182
};
\addplot [line width=0.48pt, white!41.1764705882353!black, opacity=1, mark=-, mark size=1.3, mark options={solid}, only marks]
table {%
0.76 4344
2.76 4332
4.76 4158
6.76 3867
8.76 2770
};
\addplot [line width=0.48pt, white!41.1764705882353!black, opacity=1, mark=-, mark size=1.3, mark options={solid}, only marks]
table {%
1.09 3799
3.09 4001
5.09 3955
7.09 3574
9.09 2549
};
\addplot [line width=0.48pt, white!41.1764705882353!black, opacity=1, mark=-, mark size=1.3, mark options={solid}, only marks]
table {%
1.09 4513
3.09 4353
5.09 4231
7.09 3966
9.09 2895
};
\end{axis}

\end{tikzpicture}}
        \vspace{-1.8em}
        \caption{\scriptsize{Walker2d}}
        \label{sfig:abl_w}
    \end{subfigure}
    \caption{\small{Ablation performance for the state importance weight $w(s)$ under no attack and different attacks on Hopper, Walker2d, Halfcheetah, and Ant.}}
\label{app:fig:abl}
\end{figure*}
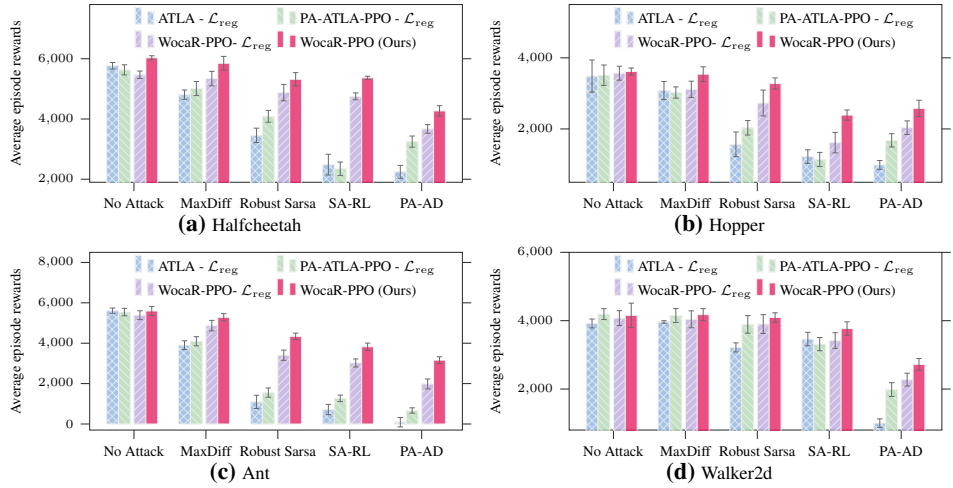
\begin{figure*}[!t]
    \centering
    \begin{subfigure}[t]{0.45\textwidth}
        \centering
        \resizebox{\textwidth}{!}{
\begin{tikzpicture}

\definecolor{color0}{rgb}{0.552941176470588,0.707450980392157,0.896078431372549}
\definecolor{color1}{rgb}{0.703057937269964,0.856470725837709,0.708208636049633}
\definecolor{color2}{rgb}{0.744102947491485,0.669256898711933,0.879184450886271}
\definecolor{color3}{rgb}{0.905434799735519,0.144732659736887,0.401761241571282}

\begin{axis}[
height=3cm,
width=\textwidth,
scale only axis,
axis line style={white!15!black},
legend cell align={left},
legend columns=2,
legend style={
  fill opacity=0,
  draw opacity=1,
  text opacity=1,
  at={(0.53,1.01)},
  anchor=north,
  draw=none,
  font=\scriptsize
},
xtick pos = lower,
tick align=outside,
x grid style={white!80!black},
xmin=-0.636, xmax=9.726,
xtick style={color=white!15!black},
xtick={0.595,2.595,4.595,6.595,8.595},
xticklabels={\scriptsize{No Attack},\scriptsize{MaxDiff},\scriptsize{Robust Sarsa},\scriptsize{SA-RL},\scriptsize{PA-AD}},
y grid style={white!80!black},
ylabel={\scriptsize{Average episode rewards}},
ymin=1900, ymax=7800,
ytick style={color=white!15!black},
y grid style={white},
tick label style={font=\tiny},
ymajorgrids,
]
\draw[draw=white!93.3333333333333!black,fill=color0,opacity=0.8,very thin,postaction={pattern=crosshatch, pattern color=white!93.3333333333333!black, fill opacity=0.8}] (axis cs:-0.165,0) rectangle (axis cs:0.165,5766);
\addlegendimage{ybar,ybar legend,draw=white!93.3333333333333!black,fill=color0,opacity=0.8,very thin,postaction={pattern=crosshatch, pattern color=white!93.3333333333333!black, fill opacity=0.8}}
\addlegendentry{ATLA - $\lossreg$}

\draw[draw=white!93.3333333333333!black,fill=color0,opacity=0.8,very thin,postaction={pattern=crosshatch, pattern color=white!93.3333333333333!black, fill opacity=0.8}] (axis cs:1.835,0) rectangle (axis cs:2.165,4807);
\draw[draw=white!93.3333333333333!black,fill=color0,opacity=0.8,very thin,postaction={pattern=crosshatch, pattern color=white!93.3333333333333!black, fill opacity=0.8}] (axis cs:3.835,0) rectangle (axis cs:4.165,3458);
\draw[draw=white!93.3333333333333!black,fill=color0,opacity=0.8,very thin,postaction={pattern=crosshatch, pattern color=white!93.3333333333333!black, fill opacity=0.8}] (axis cs:5.835,0) rectangle (axis cs:6.165,2485);
\draw[draw=white!93.3333333333333!black,fill=color0,opacity=0.8,very thin,postaction={pattern=crosshatch, pattern color=white!93.3333333333333!black, fill opacity=0.8}] (axis cs:7.835,0) rectangle (axis cs:8.165,2246);
\draw[draw=white!93.3333333333333!black,fill=color1,opacity=0.8,very thin,postaction={pattern=north west lines, pattern color=white!93.3333333333333!black, fill opacity=0.8}] (axis cs:0.165,0) rectangle (axis cs:0.495,5632);
\addlegendimage{ybar,ybar legend,draw=white!93.3333333333333!black,fill=color1,opacity=0.8,very thin,postaction={pattern=north west lines, pattern color=white!93.3333333333333!black, fill opacity=0.8}}
\addlegendentry{PA-ATLA-PPO - $\lossreg$}

\draw[draw=white!93.3333333333333!black,fill=color1,opacity=0.8,very thin,postaction={pattern=north west lines, pattern color=white!93.3333333333333!black, fill opacity=0.8}] (axis cs:2.165,0) rectangle (axis cs:2.495,5012);
\draw[draw=white!93.3333333333333!black,fill=color1,opacity=0.8,very thin,postaction={pattern=north west lines, pattern color=white!93.3333333333333!black, fill opacity=0.8}] (axis cs:4.165,0) rectangle (axis cs:4.495,4089);
\draw[draw=white!93.3333333333333!black,fill=color1,opacity=0.8,very thin,postaction={pattern=north west lines, pattern color=white!93.3333333333333!black, fill opacity=0.8}] (axis cs:6.165,0) rectangle (axis cs:6.495,2346);
\draw[draw=white!93.3333333333333!black,fill=color1,opacity=0.8,very thin,postaction={pattern=north west lines, pattern color=white!93.3333333333333!black, fill opacity=0.8}] (axis cs:8.165,0) rectangle (axis cs:8.495,3249);
\draw[draw=white!93.3333333333333!black,fill=color2,opacity=0.8,very thin,postaction={pattern=north east lines, pattern color=white!93.3333333333333!black, fill opacity=0.8}] (axis cs:0.595,0) rectangle (axis cs:0.925,5469);
\addlegendimage{ybar,ybar legend,draw=white!93.3333333333333!black,fill=color2,opacity=0.8,very thin,postaction={pattern=north east lines, pattern color=white!93.3333333333333!black, fill opacity=0.8}}
\addlegendentry{\ourppo - $\lossreg$}

\draw[draw=white!93.3333333333333!black,fill=color2,opacity=0.8,very thin,postaction={pattern=north east lines, pattern color=white!93.3333333333333!black, fill opacity=0.8}] (axis cs:2.595,0) rectangle (axis cs:2.925,5342);
\draw[draw=white!93.3333333333333!black,fill=color2,opacity=0.8,very thin,postaction={pattern=north east lines, pattern color=white!93.3333333333333!black, fill opacity=0.8}] (axis cs:4.595,0) rectangle (axis cs:4.925,4876);
\draw[draw=white!93.3333333333333!black,fill=color2,opacity=0.8,very thin,postaction={pattern=north east lines, pattern color=white!93.3333333333333!black, fill opacity=0.8}] (axis cs:6.595,0) rectangle (axis cs:6.925,4752);
\draw[draw=white!93.3333333333333!black,fill=color2,opacity=0.8,very thin,postaction={pattern=north east lines, pattern color=white!93.3333333333333!black, fill opacity=0.8}] (axis cs:8.595,0) rectangle (axis cs:8.925,3671);
\draw[draw=white!93.3333333333333!black,fill=color3,opacity=0.8,very thin] (axis cs:0.925,0) rectangle (axis cs:1.255,6032);
\addlegendimage{ybar,ybar legend,draw=white!93.3333333333333!black,fill=color3,opacity=0.8,very thin}
\addlegendentry{\ourppo (Ours)}

\draw[draw=white!93.3333333333333!black,fill=color3,opacity=0.8,very thin] (axis cs:2.925,0) rectangle (axis cs:3.255,5850);
\draw[draw=white!93.3333333333333!black,fill=color3,opacity=0.8,very thin] (axis cs:4.925,0) rectangle (axis cs:5.255,5319);
\draw[draw=white!93.3333333333333!black,fill=color3,opacity=0.8,very thin] (axis cs:6.925,0) rectangle (axis cs:7.255,5365);
\draw[draw=white!93.3333333333333!black,fill=color3,opacity=0.8,very thin] (axis cs:8.925,0) rectangle (axis cs:9.255,4269);
\path [draw=white!41.1764705882353!black, line width=0.48pt]
(axis cs:0,5657)
--(axis cs:0,5875);

\path [draw=white!41.1764705882353!black, line width=0.48pt]
(axis cs:2,4653)
--(axis cs:2,4961);

\path [draw=white!41.1764705882353!black, line width=0.48pt]
(axis cs:4,3220)
--(axis cs:4,3696);

\path [draw=white!41.1764705882353!black, line width=0.48pt]
(axis cs:6,2140)
--(axis cs:6,2830);

\path [draw=white!41.1764705882353!black, line width=0.48pt]
(axis cs:8,2033)
--(axis cs:8,2459);

\path [draw=white!41.1764705882353!black, line width=0.48pt]
(axis cs:0.33,5464)
--(axis cs:0.33,5800);

\path [draw=white!41.1764705882353!black, line width=0.48pt]
(axis cs:2.33,4778)
--(axis cs:2.33,5246);

\path [draw=white!41.1764705882353!black, line width=0.48pt]
(axis cs:4.33,3892)
--(axis cs:4.33,4286);

\path [draw=white!41.1764705882353!black, line width=0.48pt]
(axis cs:6.33,2121)
--(axis cs:6.33,2571);

\path [draw=white!41.1764705882353!black, line width=0.48pt]
(axis cs:8.33,3063)
--(axis cs:8.33,3435);

\path [draw=white!41.1764705882353!black, line width=0.48pt]
(axis cs:0.76,5345)
--(axis cs:0.76,5593);

\path [draw=white!41.1764705882353!black, line width=0.48pt]
(axis cs:2.76,5102)
--(axis cs:2.76,5582);

\path [draw=white!41.1764705882353!black, line width=0.48pt]
(axis cs:4.76,4604)
--(axis cs:4.76,5148);

\path [draw=white!41.1764705882353!black, line width=0.48pt]
(axis cs:6.76,4638)
--(axis cs:6.76,4866);

\path [draw=white!41.1764705882353!black, line width=0.48pt]
(axis cs:8.76,3528)
--(axis cs:8.76,3814);

\path [draw=white!41.1764705882353!black, line width=0.48pt]
(axis cs:1.09,5964)
--(axis cs:1.09,6100);

\path [draw=white!41.1764705882353!black, line width=0.48pt]
(axis cs:3.09,5622)
--(axis cs:3.09,6078);

\path [draw=white!41.1764705882353!black, line width=0.48pt]
(axis cs:5.09,5099)
--(axis cs:5.09,5539);

\path [draw=white!41.1764705882353!black, line width=0.48pt]
(axis cs:7.09,5311)
--(axis cs:7.09,5419);

\path [draw=white!41.1764705882353!black, line width=0.48pt]
(axis cs:9.09,4097)
--(axis cs:9.09,4441);

\addplot [line width=0.48pt, white!41.1764705882353!black, opacity=1, mark=-, mark size=1.3, mark options={solid}, only marks]
table {%
0 5657
2 4653
4 3220
6 2140
8 2033
};

\addplot [line width=0.48pt, white!41.1764705882353!black, opacity=1, mark=-, mark size=1.3, mark options={solid}, only marks]
table {%
0 5875
2 4961
4 3696
6 2830
8 2459
};

\addplot [line width=0.48pt, white!41.1764705882353!black, opacity=1, mark=-, mark size=1.3, mark options={solid}, only marks]
table {%
0.33 5464
2.33 4778
4.33 3892
6.33 2121
8.33 3063
};

\addplot [line width=0.48pt, white!41.1764705882353!black, opacity=1, mark=-, mark size=1.3, mark options={solid}, only marks]
table {%
0.33 5800
2.33 5246
4.33 4286
6.33 2571
8.33 3435
};

\addplot [line width=0.48pt, white!41.1764705882353!black, opacity=1, mark=-, mark size=1.3, mark options={solid}, only marks]
table {%
0.76 5345
2.76 5102
4.76 4604
6.76 4638
8.76 3528
};

\addplot [line width=0.48pt, white!41.1764705882353!black, opacity=1, mark=-, mark size=1.3, mark options={solid}, only marks]
table {%
0.76 5593
2.76 5582
4.76 5148
6.76 4866
8.76 3814
};

\addplot [line width=0.48pt, white!41.1764705882353!black, opacity=1, mark=-, mark size=1.3, mark options={solid}, only marks]
table {%
1.09 5964
3.09 5622
5.09 5099
7.09 5311
9.09 4097
};

\addplot [line width=0.48pt, white!41.1764705882353!black, opacity=1, mark=-, mark size=1.3, mark options={solid}, only marks]
table {%
1.09 6100
3.09 6078
5.09 5539
7.09 5419
9.09 4441
};

\end{axis}

\end{tikzpicture}}
        \vspace{-1.8em}
        \caption{\scriptsize{Halfcheetah}}
        \label{sfig:reg_halfcheetah}
    \end{subfigure}
    \begin{subfigure}[t]{0.45\textwidth}
        \centering
        \resizebox{\textwidth}{!}{
\begin{tikzpicture}

\definecolor{color0}{rgb}{0.552941176470588,0.707450980392157,0.896078431372549}
\definecolor{color1}{rgb}{0.703057937269964,0.856470725837709,0.708208636049633}
\definecolor{color2}{rgb}{0.744102947491485,0.669256898711933,0.879184450886271}
\definecolor{color3}{rgb}{0.905434799735519,0.144732659736887,0.401761241571282}

\begin{axis}[
height=3cm,
width=\textwidth,
scale only axis,
axis line style={white!15!black},
legend cell align={left},
legend columns=2,
legend style={
  fill opacity=0,
  draw opacity=1,
  text opacity=1,
  at={(0.53,1.01)},
  anchor=north,
  draw=none,
  font=\scriptsize
},
xtick pos = lower,
tick align=outside,
x grid style={white!80!black},
xmin=-0.636, xmax=9.726,
xtick style={color=white!15!black},
xtick={0.595,2.595,4.595,6.595,8.595},
xticklabels={\scriptsize{No Attack},\scriptsize{MaxDiff},\scriptsize{Robust Sarsa},\scriptsize{SA-RL},\scriptsize{PA-AD}},
y grid style={white!80!black},
ylabel={\scriptsize{Average episode rewards}},
ymin=500, ymax=5500,
ytick style={color=white!15!black},
y grid style={white},
tick label style={font=\tiny},
ymajorgrids,
]
\draw[draw=white!93.3333333333333!black,fill=color0,opacity=0.8,very thin,postaction={pattern=crosshatch, pattern color=white!93.3333333333333!black, fill opacity=0.8}] (axis cs:-0.165,0) rectangle (axis cs:0.165,3487);
\addlegendimage{ybar,ybar legend,draw=white!93.3333333333333!black,fill=color0,opacity=0.8,very thin,postaction={pattern=crosshatch, pattern color=white!93.3333333333333!black, fill opacity=0.8}}
\addlegendentry{ATLA - $\lossreg$}

\draw[draw=white!93.3333333333333!black,fill=color0,opacity=0.8,very thin,postaction={pattern=crosshatch, pattern color=white!93.3333333333333!black, fill opacity=0.8}] (axis cs:1.835,0) rectangle (axis cs:2.165,3081);
\draw[draw=white!93.3333333333333!black,fill=color0,opacity=0.8,very thin,postaction={pattern=crosshatch, pattern color=white!93.3333333333333!black, fill opacity=0.8}] (axis cs:3.835,0) rectangle (axis cs:4.165,1567);
\draw[draw=white!93.3333333333333!black,fill=color0,opacity=0.8,very thin,postaction={pattern=crosshatch, pattern color=white!93.3333333333333!black, fill opacity=0.8}] (axis cs:5.835,0) rectangle (axis cs:6.165,1224);
\draw[draw=white!93.3333333333333!black,fill=color0,opacity=0.8,very thin,postaction={pattern=crosshatch, pattern color=white!93.3333333333333!black, fill opacity=0.8}] (axis cs:7.835,0) rectangle (axis cs:8.165,987);
\draw[draw=white!93.3333333333333!black,fill=color1,opacity=0.8,very thin,postaction={pattern=north west lines, pattern color=white!93.3333333333333!black, fill opacity=0.8}] (axis cs:0.165,0) rectangle (axis cs:0.495,3512);
\addlegendimage{ybar,ybar legend,draw=white!93.3333333333333!black,fill=color1,opacity=0.8,very thin,postaction={pattern=north west lines, pattern color=white!93.3333333333333!black, fill opacity=0.8}}
\addlegendentry{PA-ATLA-PPO - $\lossreg$}

\draw[draw=white!93.3333333333333!black,fill=color1,opacity=0.8,very thin,postaction={pattern=north west lines, pattern color=white!93.3333333333333!black, fill opacity=0.8}] (axis cs:2.165,0) rectangle (axis cs:2.495,3024);
\draw[draw=white!93.3333333333333!black,fill=color1,opacity=0.8,very thin,postaction={pattern=north west lines, pattern color=white!93.3333333333333!black, fill opacity=0.8}] (axis cs:4.165,0) rectangle (axis cs:4.495,2032);
\draw[draw=white!93.3333333333333!black,fill=color1,opacity=0.8,very thin,postaction={pattern=north west lines, pattern color=white!93.3333333333333!black, fill opacity=0.8}] (axis cs:6.165,0) rectangle (axis cs:6.495,1142);
\draw[draw=white!93.3333333333333!black,fill=color1,opacity=0.8,very thin,postaction={pattern=north west lines, pattern color=white!93.3333333333333!black, fill opacity=0.8}] (axis cs:8.165,0) rectangle (axis cs:8.495,1679);
\draw[draw=white!93.3333333333333!black,fill=color2,opacity=0.8,very thin,postaction={pattern=north east lines, pattern color=white!93.3333333333333!black, fill opacity=0.8}] (axis cs:0.595,0) rectangle (axis cs:0.925,3568);
\addlegendimage{ybar,ybar legend,draw=white!93.3333333333333!black,fill=color2,opacity=0.8,very thin,postaction={pattern=north east lines, pattern color=white!93.3333333333333!black, fill opacity=0.8}}
\addlegendentry{\ourppo - $\lossreg$}

\draw[draw=white!93.3333333333333!black,fill=color2,opacity=0.8,very thin,postaction={pattern=north east lines, pattern color=white!93.3333333333333!black, fill opacity=0.8}] (axis cs:2.595,0) rectangle (axis cs:2.925,3114);
\draw[draw=white!93.3333333333333!black,fill=color2,opacity=0.8,very thin,postaction={pattern=north east lines, pattern color=white!93.3333333333333!black, fill opacity=0.8}] (axis cs:4.595,0) rectangle (axis cs:4.925,2729);
\draw[draw=white!93.3333333333333!black,fill=color2,opacity=0.8,very thin,postaction={pattern=north east lines, pattern color=white!93.3333333333333!black, fill opacity=0.8}] (axis cs:6.595,0) rectangle (axis cs:6.925,1615);
\draw[draw=white!93.3333333333333!black,fill=color2,opacity=0.8,very thin,postaction={pattern=north east lines, pattern color=white!93.3333333333333!black, fill opacity=0.8}] (axis cs:8.595,0) rectangle (axis cs:8.925,2035);
\draw[draw=white!93.3333333333333!black,fill=color3,opacity=0.8,very thin] (axis cs:0.925,0) rectangle (axis cs:1.255,3616);
\addlegendimage{ybar,ybar legend,draw=white!93.3333333333333!black,fill=color3,opacity=0.8,very thin}
\addlegendentry{\ourppo (Ours)}

\draw[draw=white!93.3333333333333!black,fill=color3,opacity=0.8,very thin] (axis cs:2.925,0) rectangle (axis cs:3.255,3541);
\draw[draw=white!93.3333333333333!black,fill=color3,opacity=0.8,very thin] (axis cs:4.925,0) rectangle (axis cs:5.255,3277);
\draw[draw=white!93.3333333333333!black,fill=color3,opacity=0.8,very thin] (axis cs:6.925,0) rectangle (axis cs:7.255,2390);
\draw[draw=white!93.3333333333333!black,fill=color3,opacity=0.8,very thin] (axis cs:8.925,0) rectangle (axis cs:9.255,2579);
\path [draw=white!41.1764705882353!black, line width=0.48pt]
(axis cs:0,3035)
--(axis cs:0,3939);

\path [draw=white!41.1764705882353!black, line width=0.48pt]
(axis cs:2,2827)
--(axis cs:2,3335);

\path [draw=white!41.1764705882353!black, line width=0.48pt]
(axis cs:4,1220)
--(axis cs:4,1914);

\path [draw=white!41.1764705882353!black, line width=0.48pt]
(axis cs:6,1033)
--(axis cs:6,1415);

\path [draw=white!41.1764705882353!black, line width=0.48pt]
(axis cs:8,863)
--(axis cs:8,1111);

\path [draw=white!41.1764705882353!black, line width=0.48pt]
(axis cs:0.33,3225)
--(axis cs:0.33,3799);

\path [draw=white!41.1764705882353!black, line width=0.48pt]
(axis cs:2.33,2865)
--(axis cs:2.33,3183);

\path [draw=white!41.1764705882353!black, line width=0.48pt]
(axis cs:4.33,1829)
--(axis cs:4.33,2235);

\path [draw=white!41.1764705882353!black, line width=0.48pt]
(axis cs:6.33,944)
--(axis cs:6.33,1340);

\path [draw=white!41.1764705882353!black, line width=0.48pt]
(axis cs:8.33,1492)
--(axis cs:8.33,1866);

\path [draw=white!41.1764705882353!black, line width=0.48pt]
(axis cs:0.76,3372)
--(axis cs:0.76,3764);

\path [draw=white!41.1764705882353!black, line width=0.48pt]
(axis cs:2.76,2885)
--(axis cs:2.76,3343);

\path [draw=white!41.1764705882353!black, line width=0.48pt]
(axis cs:4.76,2365)
--(axis cs:4.76,3093);

\path [draw=white!41.1764705882353!black, line width=0.48pt]
(axis cs:6.76,1326)
--(axis cs:6.76,1904);

\path [draw=white!41.1764705882353!black, line width=0.48pt]
(axis cs:8.76,1844)
--(axis cs:8.76,2226);

\path [draw=white!41.1764705882353!black, line width=0.48pt]
(axis cs:1.09,3517)
--(axis cs:1.09,3715);

\path [draw=white!41.1764705882353!black, line width=0.48pt]
(axis cs:3.09,3334)
--(axis cs:3.09,3748);

\path [draw=white!41.1764705882353!black, line width=0.48pt]
(axis cs:5.09,3118)
--(axis cs:5.09,3436);

\path [draw=white!41.1764705882353!black, line width=0.48pt]
(axis cs:7.09,2245)
--(axis cs:7.09,2535);

\path [draw=white!41.1764705882353!black, line width=0.48pt]
(axis cs:9.09,2350)
--(axis cs:9.09,2808);

\addplot [line width=0.48pt, white!41.1764705882353!black, opacity=1, mark=-, mark size=1.3, mark options={solid}, only marks]
table {%
0 3035
2 2827
4 1220
6 1033
8 863
};

\addplot [line width=0.48pt, white!41.1764705882353!black, opacity=1, mark=-, mark size=1.3, mark options={solid}, only marks]
table {%
0 3939
2 3335
4 1914
6 1415
8 1111
};

\addplot [line width=0.48pt, white!41.1764705882353!black, opacity=1, mark=-, mark size=1.3, mark options={solid}, only marks]
table {%
0.33 3225
2.33 2865
4.33 1829
6.33 944
8.33 1492
};

\addplot [line width=0.48pt, white!41.1764705882353!black, opacity=1, mark=-, mark size=1.3, mark options={solid}, only marks]
table {%
0.33 3799
2.33 3183
4.33 2235
6.33 1340
8.33 1866
};

\addplot [line width=0.48pt, white!41.1764705882353!black, opacity=1, mark=-, mark size=1.3, mark options={solid}, only marks]
table {%
0.76 3372
2.76 2885
4.76 2365
6.76 1326
8.76 1844
};

\addplot [line width=0.48pt, white!41.1764705882353!black, opacity=1, mark=-, mark size=1.3, mark options={solid}, only marks]
table {%
0.76 3764
2.76 3343
4.76 3093
6.76 1904
8.76 2226
};

\addplot [line width=0.48pt, white!41.1764705882353!black, opacity=1, mark=-, mark size=1.3, mark options={solid}, only marks]
table {%
1.09 3517
3.09 3334
5.09 3118
7.09 2245
9.09 2350
};

\addplot [line width=0.48pt, white!41.1764705882353!black, opacity=1, mark=-, mark size=1.3, mark options={solid}, only marks]
table {%
1.09 3715
3.09 3748
5.09 3436
7.09 2535
9.09 2808
};

\end{axis}
\end{tikzpicture}}
        \vspace{-1.8em}
        \caption{\scriptsize{Hopper}}
        \label{sfig:reg_hopper}
    \end{subfigure}\\
        \begin{subfigure}[t]{0.45\textwidth}
        \centering
        \resizebox{\textwidth}{!}{
\begin{tikzpicture}

\definecolor{color0}{rgb}{0.552941176470588,0.707450980392157,0.896078431372549}
\definecolor{color1}{rgb}{0.703057937269964,0.856470725837709,0.708208636049633}
\definecolor{color2}{rgb}{0.744102947491485,0.669256898711933,0.879184450886271}
\definecolor{color3}{rgb}{0.905434799735519,0.144732659736887,0.401761241571282}

\begin{axis}[
height=3cm,
width=\textwidth,
scale only axis,
axis line style={white!15!black},
legend cell align={left},
legend columns=2,
legend style={
  fill opacity=0,
  draw opacity=1,
  text opacity=1,
  at={(0.53,1.01)},
  anchor=north,
  draw=none,
  font=\scriptsize
},
xtick pos = lower,
tick align=outside,
x grid style={white!80!black},
xmin=-0.636, xmax=9.726,
xtick style={color=white!15!black},
xtick={0.595,2.595,4.595,6.595,8.595},
xticklabels={\scriptsize{No Attack},\scriptsize{MaxDiff},\scriptsize{Robust Sarsa},\scriptsize{SA-RL},\scriptsize{PA-AD}},
y grid style={white!80!black},
ylabel={\scriptsize{Average episode rewards}},
ymin=-300, ymax=8500,
ytick style={color=white!15!black},
y grid style={white},
tick label style={font=\tiny},
ymajorgrids,
]
\draw[draw=white!93.3333333333333!black,fill=color0,opacity=0.8,very thin,postaction={pattern=crosshatch, pattern color=white!93.3333333333333!black, fill opacity=0.8}] (axis cs:-0.165,0) rectangle (axis cs:0.165,5612);
\addlegendimage{ybar,ybar legend,draw=white!93.3333333333333!black,fill=color0,opacity=0.8,very thin,postaction={pattern=crosshatch, pattern color=white!93.3333333333333!black, fill opacity=0.8}}
\addlegendentry{ ATLA - $\lossreg$}

\draw[draw=white!93.3333333333333!black,fill=color0,opacity=0.8,very thin,postaction={pattern=crosshatch, pattern color=white!93.3333333333333!black, fill opacity=0.8}] (axis cs:1.835,0) rectangle (axis cs:2.165,3903);
\draw[draw=white!93.3333333333333!black,fill=color0,opacity=0.8,very thin,postaction={pattern=crosshatch, pattern color=white!93.3333333333333!black, fill opacity=0.8}] (axis cs:3.835,0) rectangle (axis cs:4.165,1096);
\draw[draw=white!93.3333333333333!black,fill=color0,opacity=0.8,very thin,postaction={pattern=crosshatch, pattern color=white!93.3333333333333!black, fill opacity=0.8}] (axis cs:5.835,0) rectangle (axis cs:6.165,716);
\draw[draw=white!93.3333333333333!black,fill=color0,opacity=0.8,very thin,postaction={pattern=crosshatch, pattern color=white!93.3333333333333!black, fill opacity=0.8}] (axis cs:7.835,0) rectangle (axis cs:8.165,89);
\draw[draw=white!93.3333333333333!black,fill=color1,opacity=0.8,very thin,postaction={pattern=north west lines, pattern color=white!93.3333333333333!black, fill opacity=0.8}] (axis cs:0.165,0) rectangle (axis cs:0.495,5543);
\addlegendimage{ybar,ybar legend,draw=white!93.3333333333333!black,fill=color1,opacity=0.8,very thin,postaction={pattern=north west lines, pattern color=white!93.3333333333333!black, fill opacity=0.8}}
\addlegendentry{PA-ATLA-PPO - $\lossreg$}

\draw[draw=white!93.3333333333333!black,fill=color1,opacity=0.8,very thin,postaction={pattern=north west lines, pattern color=white!93.3333333333333!black, fill opacity=0.8}] (axis cs:2.165,0) rectangle (axis cs:2.495,4098);
\draw[draw=white!93.3333333333333!black,fill=color1,opacity=0.8,very thin,postaction={pattern=north west lines, pattern color=white!93.3333333333333!black, fill opacity=0.8}] (axis cs:4.165,0) rectangle (axis cs:4.495,1554);
\draw[draw=white!93.3333333333333!black,fill=color1,opacity=0.8,very thin,postaction={pattern=north west lines, pattern color=white!93.3333333333333!black, fill opacity=0.8}] (axis cs:6.165,0) rectangle (axis cs:6.495,1272);
\draw[draw=white!93.3333333333333!black,fill=color1,opacity=0.8,very thin,postaction={pattern=north west lines, pattern color=white!93.3333333333333!black, fill opacity=0.8}] (axis cs:8.165,0) rectangle (axis cs:8.495,669);
\draw[draw=white!93.3333333333333!black,fill=color2,opacity=0.8,very thin,postaction={pattern=north east lines, pattern color=white!93.3333333333333!black, fill opacity=0.8}] (axis cs:0.595,0) rectangle (axis cs:0.925,5388);
\addlegendimage{ybar,ybar legend,draw=white!93.3333333333333!black,fill=color2,opacity=0.8,very thin,postaction={pattern=north east lines, pattern color=white!93.3333333333333!black, fill opacity=0.8}}
\addlegendentry{\ourppo - $\lossreg$}

\draw[draw=white!93.3333333333333!black,fill=color2,opacity=0.8,very thin,postaction={pattern=north east lines, pattern color=white!93.3333333333333!black, fill opacity=0.8}] (axis cs:2.595,0) rectangle (axis cs:2.925,4872);
\draw[draw=white!93.3333333333333!black,fill=color2,opacity=0.8,very thin,postaction={pattern=north east lines, pattern color=white!93.3333333333333!black, fill opacity=0.8}] (axis cs:4.595,0) rectangle (axis cs:4.925,3406);
\draw[draw=white!93.3333333333333!black,fill=color2,opacity=0.8,very thin,postaction={pattern=north east lines, pattern color=white!93.3333333333333!black, fill opacity=0.8}] (axis cs:6.595,0) rectangle (axis cs:6.925,3023);
\draw[draw=white!93.3333333333333!black,fill=color2,opacity=0.8,very thin,postaction={pattern=north east lines, pattern color=white!93.3333333333333!black, fill opacity=0.8}] (axis cs:8.595,0) rectangle (axis cs:8.925,1982);
\draw[draw=white!93.3333333333333!black,fill=color3,opacity=0.8,very thin] (axis cs:0.925,0) rectangle (axis cs:1.255,5596);
\addlegendimage{ybar,ybar legend,draw=white!93.3333333333333!black,fill=color3,opacity=0.8,very thin}
\addlegendentry{\ourppo (Ours)}

\draw[draw=white!93.3333333333333!black,fill=color3,opacity=0.8,very thin] (axis cs:2.925,0) rectangle (axis cs:3.255,5284);
\draw[draw=white!93.3333333333333!black,fill=color3,opacity=0.8,very thin] (axis cs:4.925,0) rectangle (axis cs:5.255,4339);
\draw[draw=white!93.3333333333333!black,fill=color3,opacity=0.8,very thin] (axis cs:6.925,0) rectangle (axis cs:7.255,3822);
\draw[draw=white!93.3333333333333!black,fill=color3,opacity=0.8,very thin] (axis cs:8.925,0) rectangle (axis cs:9.255,3164);
\path [draw=white!41.1764705882353!black, line width=0.48pt]
(axis cs:0,5482)
--(axis cs:0,5742);

\path [draw=white!41.1764705882353!black, line width=0.48pt]
(axis cs:2,3686)
--(axis cs:2,4120);

\path [draw=white!41.1764705882353!black, line width=0.48pt]
(axis cs:4,767)
--(axis cs:4,1425);

\path [draw=white!41.1764705882353!black, line width=0.48pt]
(axis cs:6,460)
--(axis cs:6,972);

\path [draw=white!41.1764705882353!black, line width=0.48pt]
(axis cs:8,-143)
--(axis cs:8,321);

\path [draw=white!41.1764705882353!black, line width=0.48pt]
(axis cs:0.33,5359)
--(axis cs:0.33,5727);

\path [draw=white!41.1764705882353!black, line width=0.48pt]
(axis cs:2.33,3873)
--(axis cs:2.33,4323);

\path [draw=white!41.1764705882353!black, line width=0.48pt]
(axis cs:4.33,1324)
--(axis cs:4.33,1784);

\path [draw=white!41.1764705882353!black, line width=0.48pt]
(axis cs:6.33,1113)
--(axis cs:6.33,1431);

\path [draw=white!41.1764705882353!black, line width=0.48pt]
(axis cs:8.33,537)
--(axis cs:8.33,801);

\path [draw=white!41.1764705882353!black, line width=0.48pt]
(axis cs:0.76,5170)
--(axis cs:0.76,5606);

\path [draw=white!41.1764705882353!black, line width=0.48pt]
(axis cs:2.76,4612)
--(axis cs:2.76,5132);

\path [draw=white!41.1764705882353!black, line width=0.48pt]
(axis cs:4.76,3153)
--(axis cs:4.76,3659);

\path [draw=white!41.1764705882353!black, line width=0.48pt]
(axis cs:6.76,2824)
--(axis cs:6.76,3222);

\path [draw=white!41.1764705882353!black, line width=0.48pt]
(axis cs:8.76,1736)
--(axis cs:8.76,2228);

\path [draw=white!41.1764705882353!black, line width=0.48pt]
(axis cs:1.09,5371)
--(axis cs:1.09,5821);

\path [draw=white!41.1764705882353!black, line width=0.48pt]
(axis cs:3.09,5102)
--(axis cs:3.09,5466);

\path [draw=white!41.1764705882353!black, line width=0.48pt]
(axis cs:5.09,4179)
--(axis cs:5.09,4499);

\path [draw=white!41.1764705882353!black, line width=0.48pt]
(axis cs:7.09,3637)
--(axis cs:7.09,4007);

\path [draw=white!41.1764705882353!black, line width=0.48pt]
(axis cs:9.09,3001)
--(axis cs:9.09,3327);

\addplot [line width=0.48pt, white!41.1764705882353!black, opacity=1, mark=-, mark size=1.3, mark options={solid}, only marks]
table {%
0 5482
2 3686
4 767
6 460
8 -143
};

\addplot [line width=0.48pt, white!41.1764705882353!black, opacity=1, mark=-, mark size=1.3, mark options={solid}, only marks]
table {%
0 5742
2 4120
4 1425
6 972
8 321
};

\addplot [line width=0.48pt, white!41.1764705882353!black, opacity=1, mark=-, mark size=1.3, mark options={solid}, only marks]
table {%
0.33 5359
2.33 3873
4.33 1324
6.33 1113
8.33 537
};

\addplot [line width=0.48pt, white!41.1764705882353!black, opacity=1, mark=-, mark size=1.3, mark options={solid}, only marks]
table {%
0.33 5727
2.33 4323
4.33 1784
6.33 1431
8.33 801
};

\addplot [line width=0.48pt, white!41.1764705882353!black, opacity=1, mark=-, mark size=1.3, mark options={solid}, only marks]
table {%
0.76 5170
2.76 4612
4.76 3153
6.76 2824
8.76 1736
};

\addplot [line width=0.48pt, white!41.1764705882353!black, opacity=1, mark=-, mark size=1.3, mark options={solid}, only marks]
table {%
0.76 5606
2.76 5132
4.76 3659
6.76 3222
8.76 2228
};

\addplot [line width=0.48pt, white!41.1764705882353!black, opacity=1, mark=-, mark size=1.3, mark options={solid}, only marks]
table {%
1.09 5371
3.09 5102
5.09 4179
7.09 3637
9.09 3001
};

\addplot [line width=0.48pt, white!41.1764705882353!black, opacity=1, mark=-, mark size=1.3, mark options={solid}, only marks]
table {%
1.09 5821
3.09 5466
5.09 4499
7.09 4007
9.09 3327
};

\end{axis}

\end{tikzpicture}}
        \vspace{-1.8em}
        \caption{\scriptsize{Ant}}
        \label{sfig:reg_ant}
    \end{subfigure}
    \begin{subfigure}[t]{0.45\textwidth}
        \centering
        \resizebox{\textwidth}{!}{
\begin{tikzpicture}

\definecolor{color0}{rgb}{0.552941176470588,0.707450980392157,0.896078431372549}
\definecolor{color1}{rgb}{0.703057937269964,0.856470725837709,0.708208636049633}
\definecolor{color2}{rgb}{0.744102947491485,0.669256898711933,0.879184450886271}
\definecolor{color3}{rgb}{0.905434799735519,0.144732659736887,0.401761241571282}

\begin{axis}[
height=3cm,
width=\textwidth,
scale only axis,
axis line style={white!15!black},
legend cell align={left},
legend columns=2,
legend style={
  fill opacity=0,
  draw opacity=1,
  text opacity=1,
  at={(0.53,1.01)},
  anchor=north,
  draw=none,
  font=\scriptsize
},
xtick pos = lower,
tick align=outside,
x grid style={white!80!black},
xmin=-0.636, xmax=9.726,
xtick style={color=white!15!black},
xtick={0.595,2.595,4.595,6.595,8.595},
xticklabels={\scriptsize{No Attack},\scriptsize{MaxDiff},\scriptsize{Robust Sarsa},\scriptsize{SA-RL},\scriptsize{PA-AD}},
y grid style={white!80!black},
ylabel={\scriptsize{Average episode rewards}},
ymin=800, ymax=6000,
ytick style={color=white!15!black},
y grid style={white},
tick label style={font=\tiny},
ymajorgrids,
]
\draw[draw=white!93.3333333333333!black,fill=color0,opacity=0.8,very thin,postaction={pattern=crosshatch, pattern color=white!93.3333333333333!black, fill opacity=0.8}] (axis cs:-0.165,0) rectangle (axis cs:0.165,3920);
\addlegendimage{ybar,ybar legend,draw=white!93.3333333333333!black,fill=color0,opacity=0.8,very thin,postaction={pattern=crosshatch, pattern color=white!93.3333333333333!black, fill opacity=0.8}}
\addlegendentry{ATLA - $\lossreg$}

\draw[draw=white!93.3333333333333!black,fill=color0,opacity=0.8,very thin,postaction={pattern=crosshatch, pattern color=white!93.3333333333333!black, fill opacity=0.8}] (axis cs:1.835,0) rectangle (axis cs:2.165,3963);
\draw[draw=white!93.3333333333333!black,fill=color0,opacity=0.8,very thin,postaction={pattern=crosshatch, pattern color=white!93.3333333333333!black, fill opacity=0.8}] (axis cs:3.835,0) rectangle (axis cs:4.165,3219);
\draw[draw=white!93.3333333333333!black,fill=color0,opacity=0.8,very thin,postaction={pattern=crosshatch, pattern color=white!93.3333333333333!black, fill opacity=0.8}] (axis cs:5.835,0) rectangle (axis cs:6.165,3463);
\draw[draw=white!93.3333333333333!black,fill=color0,opacity=0.8,very thin,postaction={pattern=crosshatch, pattern color=white!93.3333333333333!black, fill opacity=0.8}] (axis cs:7.835,0) rectangle (axis cs:8.165,1004);
\draw[draw=white!93.3333333333333!black,fill=color1,opacity=0.8,very thin,postaction={pattern=north west lines, pattern color=white!93.3333333333333!black, fill opacity=0.8}] (axis cs:0.165,0) rectangle (axis cs:0.495,4192);
\addlegendimage{ybar,ybar legend,draw=white!93.3333333333333!black,fill=color1,opacity=0.8,very thin,postaction={pattern=north west lines, pattern color=white!93.3333333333333!black, fill opacity=0.8}}
\addlegendentry{PA-ATLA-PPO - $\lossreg$}

\draw[draw=white!93.3333333333333!black,fill=color1,opacity=0.8,very thin,postaction={pattern=north west lines, pattern color=white!93.3333333333333!black, fill opacity=0.8}] (axis cs:2.165,0) rectangle (axis cs:2.495,4152);
\draw[draw=white!93.3333333333333!black,fill=color1,opacity=0.8,very thin,postaction={pattern=north west lines, pattern color=white!93.3333333333333!black, fill opacity=0.8}] (axis cs:4.165,0) rectangle (axis cs:4.495,3893);
\draw[draw=white!93.3333333333333!black,fill=color1,opacity=0.8,very thin,postaction={pattern=north west lines, pattern color=white!93.3333333333333!black, fill opacity=0.8}] (axis cs:6.165,0) rectangle (axis cs:6.495,3315);
\draw[draw=white!93.3333333333333!black,fill=color1,opacity=0.8,very thin,postaction={pattern=north west lines, pattern color=white!93.3333333333333!black, fill opacity=0.8}] (axis cs:8.165,0) rectangle (axis cs:8.495,1984);
\draw[draw=white!93.3333333333333!black,fill=color2,opacity=0.8,very thin,postaction={pattern=north east lines, pattern color=white!93.3333333333333!black, fill opacity=0.8}] (axis cs:0.595,0) rectangle (axis cs:0.925,4076);
\addlegendimage{ybar,ybar legend,draw=white!93.3333333333333!black,fill=color2,opacity=0.8,very thin,postaction={pattern=north east lines, pattern color=white!93.3333333333333!black, fill opacity=0.8}}
\addlegendentry{\ourppo - $\lossreg$}

\draw[draw=white!93.3333333333333!black,fill=color2,opacity=0.8,very thin,postaction={pattern=north east lines, pattern color=white!93.3333333333333!black, fill opacity=0.8}] (axis cs:2.595,0) rectangle (axis cs:2.925,4041);
\draw[draw=white!93.3333333333333!black,fill=color2,opacity=0.8,very thin,postaction={pattern=north east lines, pattern color=white!93.3333333333333!black, fill opacity=0.8}] (axis cs:4.595,0) rectangle (axis cs:4.925,3902);
\draw[draw=white!93.3333333333333!black,fill=color2,opacity=0.8,very thin,postaction={pattern=north east lines, pattern color=white!93.3333333333333!black, fill opacity=0.8}] (axis cs:6.595,0) rectangle (axis cs:6.925,3419);
\draw[draw=white!93.3333333333333!black,fill=color2,opacity=0.8,very thin,postaction={pattern=north east lines, pattern color=white!93.3333333333333!black, fill opacity=0.8}] (axis cs:8.595,0) rectangle (axis cs:8.925,2276);
\draw[draw=white!93.3333333333333!black,fill=color3,opacity=0.8,very thin] (axis cs:0.925,0) rectangle (axis cs:1.255,4156);
\addlegendimage{ybar,ybar legend,draw=white!93.3333333333333!black,fill=color3,opacity=0.8,very thin}
\addlegendentry{\ourppo (Ours)}

\draw[draw=white!93.3333333333333!black,fill=color3,opacity=0.8,very thin] (axis cs:2.925,0) rectangle (axis cs:3.255,4177);
\draw[draw=white!93.3333333333333!black,fill=color3,opacity=0.8,very thin] (axis cs:4.925,0) rectangle (axis cs:5.255,4093);
\draw[draw=white!93.3333333333333!black,fill=color3,opacity=0.8,very thin] (axis cs:6.925,0) rectangle (axis cs:7.255,3770);
\draw[draw=white!93.3333333333333!black,fill=color3,opacity=0.8,very thin] (axis cs:8.925,0) rectangle (axis cs:9.255,2722);
\path [draw=white!41.1764705882353!black, line width=0.48pt]
(axis cs:0,3791)
--(axis cs:0,4049);

\path [draw=white!41.1764705882353!black, line width=0.48pt]
(axis cs:2,3927)
--(axis cs:2,3999);

\path [draw=white!41.1764705882353!black, line width=0.48pt]
(axis cs:4,3087)
--(axis cs:4,3351);

\path [draw=white!41.1764705882353!black, line width=0.48pt]
(axis cs:6,3267)
--(axis cs:6,3659);

\path [draw=white!41.1764705882353!black, line width=0.48pt]
(axis cs:8,879)
--(axis cs:8,1129);

\path [draw=white!41.1764705882353!black, line width=0.48pt]
(axis cs:0.33,4033)
--(axis cs:0.33,4351);

\path [draw=white!41.1764705882353!black, line width=0.48pt]
(axis cs:2.33,3948)
--(axis cs:2.33,4356);

\path [draw=white!41.1764705882353!black, line width=0.48pt]
(axis cs:4.33,3635)
--(axis cs:4.33,4151);

\path [draw=white!41.1764705882353!black, line width=0.48pt]
(axis cs:6.33,3122)
--(axis cs:6.33,3508);

\path [draw=white!41.1764705882353!black, line width=0.48pt]
(axis cs:8.33,1782)
--(axis cs:8.33,2186);

\path [draw=white!41.1764705882353!black, line width=0.48pt]
(axis cs:0.76,3858)
--(axis cs:0.76,4294);

\path [draw=white!41.1764705882353!black, line width=0.48pt]
(axis cs:2.76,3793)
--(axis cs:2.76,4289);

\path [draw=white!41.1764705882353!black, line width=0.48pt]
(axis cs:4.76,3628)
--(axis cs:4.76,4176);

\path [draw=white!41.1764705882353!black, line width=0.48pt]
(axis cs:6.76,3189)
--(axis cs:6.76,3649);

\path [draw=white!41.1764705882353!black, line width=0.48pt]
(axis cs:8.76,2089)
--(axis cs:8.76,2463);

\path [draw=white!41.1764705882353!black, line width=0.48pt]
(axis cs:1.09,3799)
--(axis cs:1.09,4513);

\path [draw=white!41.1764705882353!black, line width=0.48pt]
(axis cs:3.09,4001)
--(axis cs:3.09,4353);

\path [draw=white!41.1764705882353!black, line width=0.48pt]
(axis cs:5.09,3955)
--(axis cs:5.09,4231);

\path [draw=white!41.1764705882353!black, line width=0.48pt]
(axis cs:7.09,3574)
--(axis cs:7.09,3966);

\path [draw=white!41.1764705882353!black, line width=0.48pt]
(axis cs:9.09,2549)
--(axis cs:9.09,2895);

\addplot [line width=0.48pt, white!41.1764705882353!black, opacity=1, mark=-, mark size=1.3, mark options={solid}, only marks]
table {%
0 3791
2 3927
4 3087
6 3267
8 879
};

\addplot [line width=0.48pt, white!41.1764705882353!black, opacity=1, mark=-, mark size=1.3, mark options={solid}, only marks]
table {%
0 4049
2 3999
4 3351
6 3659
8 1129
};

\addplot [line width=0.48pt, white!41.1764705882353!black, opacity=1, mark=-, mark size=1.3, mark options={solid}, only marks]
table {%
0.33 4033
2.33 3948
4.33 3635
6.33 3122
8.33 1782
};

\addplot [line width=0.48pt, white!41.1764705882353!black, opacity=1, mark=-, mark size=1.3, mark options={solid}, only marks]
table {%
0.33 4351
2.33 4356
4.33 4151
6.33 3508
8.33 2186
};

\addplot [line width=0.48pt, white!41.1764705882353!black, opacity=1, mark=-, mark size=1.3, mark options={solid}, only marks]
table {%
0.76 3858
2.76 3793
4.76 3628
6.76 3189
8.76 2089
};

\addplot [line width=0.48pt, white!41.1764705882353!black, opacity=1, mark=-, mark size=1.3, mark options={solid}, only marks]
table {%
0.76 4294
2.76 4289
4.76 4176
6.76 3649
8.76 2463
};

\addplot [line width=0.48pt, white!41.1764705882353!black, opacity=1, mark=-, mark size=1.3, mark options={solid}, only marks]
table {%
1.09 3799
3.09 4001
5.09 3955
7.09 3574
9.09 2549
};
\addplot [line width=0.48pt, white!41.1764705882353!black, opacity=1, mark=-, mark size=1.3, mark options={solid}, only marks]
table {%
1.09 4513
3.09 4353
5.09 4231
7.09 3966
9.09 2895
};
\end{axis}

\end{tikzpicture}}
        \vspace{-1.8em}
        \caption{\scriptsize{Walker2d}}
        \label{sfig:reg_w}
    \end{subfigure}
    \caption{\small{Ablation performance for the state regularization loss $\lossreg$ under no attack and different attacks on Hopper, Walker2d, Halfcheetah, and Ant.}}
\label{app:fig:reg}
\end{figure*}
We provide full ablation experimental results for the state importance weight $w(s)$ and the regularization loss $\lossreg$ \citep{zhang2020robust} on four MuJoCo environments.

For the state importance weight $w(s)$, we compare the performance between the original \ourppo and \ourppo without $w(s)$ in Figure~\ref{app:fig:abl}. Additionally, we also equip \sappo with $w(s)$ to show the universal applicability of this design.
In all four MuJoCo environments, we can see that with $w(s)$, both \ourppo and SA-PPO get boosted robustness, verifying the effectiveness of the state importance weight.

For the state regularization loss $\lossreg$, Figure~\ref{app:fig:reg} verifies that $\lossreg$ enhances the robustness of \ourppo, since the performance of \ourppo drops without $\lossreg$. 
On the other hand, Figure~\ref{app:fig:reg} also compares the performance of ATLA methods and our algorithm without $\lossreg$ (note that ATLA methods also regularizes the PPO policies during training). The results indicate that \textit{the decisive contribution of \ourppo to robustness improving comes from the worst-attack-aware policy optimization. }

These ablation studies demonstrate that all the techniques are beneficial for robustness improvement and further show that our worst-case-aware training performs better than training with attackers.

\section{Potential Societal Impacts}
\label{app:limit}


This work focuses on improving the robustness of deep RL agents, which can make RL models more reliable in high-stakes applications. Although it is generally positive for the community to build more robust agents, such robust agents may also bring some potentially negative impacts, including the possibility of robust robots replacing some occupations and causing mass unemployment.

\end{document}